\documentclass[english]{article}
\usepackage[T1]{fontenc}
\usepackage[latin9]{inputenc}
\usepackage{array}
\usepackage{float}
\usepackage{multirow}
\usepackage{amsmath}
\usepackage{amsthm}
\usepackage{amssymb}

\makeatletter

\newcommand{\lyxmathsym}[1]{\ifmmode\begingroup\def\b@ld{bold}
  \text{\ifx\math@version\b@ld\bfseries\fi#1}\endgroup\else#1\fi}

\providecommand{\tabularnewline}{\\}
\floatstyle{ruled}
\newfloat{algorithm}{tbp}{loa}
\providecommand{\algorithmname}{Algorithm}
\floatname{algorithm}{\protect\algorithmname}

\theoremstyle{plain}
\newtheorem{thm}{\protect\theoremname}[section]
\theoremstyle{plain}
\newtheorem{assumption}[thm]{\protect\assumptionname}
\theoremstyle{plain}
\newtheorem{lem}[thm]{\protect\lemmaname}
\theoremstyle{definition}
\newtheorem{example}[thm]{\protect\examplename}
\theoremstyle{remark}
\newtheorem{rem}[thm]{\protect\remarkname}

\usepackage{pifont}
\usepackage{savefnmark}
\usepackage{microtype}
\usepackage{graphicx}
\usepackage{subfigure}
\usepackage{booktabs} 

\usepackage{hyperref}

\usepackage[accepted]{icml2024}

\usepackage[capitalize,noabbrev]{cleveref}

\usepackage[textsize=tiny]{todonotes}

\makeatother

\usepackage{babel}
\providecommand{\assumptionname}{Assumption}
\providecommand{\examplename}{Example}
\providecommand{\lemmaname}{Lemma}
\providecommand{\remarkname}{Remark}
\providecommand{\theoremname}{Theorem}

\begin{document}
\global\long\def\E{\mathbb{\mathbb{E}}}%
\global\long\def\F{\mathcal{F}}%
\global\long\def\N{\mathbb{\mathbb{N}}}%
\global\long\def\R{\mathbb{R}}%
\global\long\def\dom{\mathbf{X}}%
\global\long\def\argmin{\mathrm{argmin}}%
\global\long\def\bx{\mathbf{x}}%
\global\long\def\by{\mathbf{y}}%
\global\long\def\bz{\mathbf{z}}%
\global\long\def\bg{\mathbf{g}}%
\global\long\def\rmI{\mathrm{I}}%
\global\long\def\O{\mathcal{O}}%
\global\long\def\bG{\bar{G}}%
\global\long\def\bL{\bar{L}}%
\global\long\def\asigma{\sigma_{\mathrm{any}}^{2}}%
\global\long\def\rsigma{\sigma_{\mathrm{rand}}^{2}}%

\twocolumn[
\icmltitle{On the Last-Iterate Convergence of Shuffling Gradient Methods}


\icmlsetsymbol{equal}{*}




\begin{icmlauthorlist}
\icmlauthor{Zijian Liu}{Stern} 
\icmlauthor{Zhengyuan Zhou}{Stern,Arena} 
\end{icmlauthorlist}

\icmlaffiliation{Stern}{Stern School of Business, New York University}
\icmlaffiliation{Arena}{Arena Technologies}

\icmlcorrespondingauthor{Zijian Liu}{zl3067@stern.nyu.edu}
\icmlcorrespondingauthor{Zhengyuan Zhou}{zhengyuanzhou24@gmail.com}

\icmlkeywords{Machine Learning, ICML}

\vskip 0.3in 
]


\printAffiliationsAndNotice{}  
\begin{abstract}
Shuffling gradient methods are widely used in modern machine learning
tasks and include three popular implementations: Random Reshuffle
(RR), Shuffle Once (SO), and Incremental Gradient (IG). Compared to
the empirical success, the theoretical guarantee of shuffling gradient
methods was not well-understood for a long time. Until recently, the
convergence rates had just been established for the average iterate
for convex functions and the last iterate for strongly convex problems
(using squared distance as the metric). However, when using the function
value gap as the convergence criterion, existing theories cannot interpret
the good performance of the last iterate in different settings (e.g.,
constrained optimization). To bridge this gap between practice and
theory, we prove the first last-iterate convergence rates for shuffling
gradient methods with respect to the objective value even without
strong convexity. Our new results either (nearly) match the existing
last-iterate lower bounds or are as fast as the previous best upper
bounds for the average iterate.
\end{abstract}

\section{Introduction\label{sec:intro}}

\begin{figure*}[ht]

\noindent\begin{minipage}[t]{1\textwidth}%
\vspace{-0.25in}
\begin{table}[H]
\caption{\label{tab:smooth}Summary of our new upper bounds and the existing
lower bounds for $L$-smooth $f_{i}(\protect\bx)$ for large $K$.
If no lower bound was established before in the case, we instead state
the previous best-known rate. Here, $\protect\asigma\triangleq\frac{1}{n}\sum_{i=1}^{n}\left\Vert \nabla f_{i}(\protect\bx_{*})\right\Vert ^{2}$,
$\protect\rsigma\triangleq\protect\asigma+n\left\Vert \nabla f(\protect\bx_{*})\right\Vert ^{2}$
and $D\triangleq\left\Vert \protect\bx_{*}-\protect\bx_{1}\right\Vert $.
All rates use the function value gap as the convergence criterion.
In the column of \textquotedbl Type\textquotedbl , \textquotedbl Any\textquotedbl{}
means the rate holds for whatever permutation not limited to RR/SO/IG.
\textquotedbl Random\textquotedbl{} refers to the uniformly sampled
permutation but is not restricted to RR/SO (see Remark \ref{rem:sample}
for a detailed explanation). \textquotedbl Avg\textquotedbl{} and
\textquotedbl Last\textquotedbl{} in the \textquotedbl Output\textquotedbl{}
column stand for the average iterate and the last iterate, respectively.
In the last column, \textquotedbl\ding{51}\textquotedbl{} means
$\psi(\protect\bx)$ can be taken arbitrarily and \textquotedbl\ding{55}\textquotedbl{}
implies $\psi(\protect\bx)=0$.}
\vspace{0.1in}

\centering{}%
\begin{tabular}{|>{\centering}m{0.135\textwidth}|c|c|c|c|c|}
\hline 
\multicolumn{6}{|c|}{$F(\bx)=f(\bx)+\psi(\bx)$ where $f(\bx)=\frac{1}{n}\sum_{i=1}^{n}f_{i}(\bx)$
and $f_{i}(\bx)$ and $\psi(\bx)$ are convex}\tabularnewline
\hline 
Settings & References & Rate & Type & Output & $\psi(\bx)$\tabularnewline
\hline 
 & \cite{mishchenko2020random} & $\O\left(\frac{L^{1/3}\sigma_{\mathrm{any}}^{2/3}D^{4/3}}{K^{2/3}}\right)$ & IG & Avg & \ding{55}\tabularnewline
\multirow{2}{0.135\textwidth}{$L$-smooth $f_{i}(\bx)$} & \textbf{Ours} (Theorem \ref{thm:smooth-cvx-main}) & $\widetilde{\O}\left(\frac{L^{1/3}\sigma_{\mathrm{any}}^{2/3}D^{4/3}}{K^{2/3}}\right)$ & Any & Last & \ding{51}\tabularnewline
\cline{2-6} \cline{3-6} \cline{4-6} \cline{5-6} \cline{6-6} 
 & \textbf{Ours} (Theorem \ref{thm:smooth-cvx-main}) & $\widetilde{\O}\left(\frac{L^{1/3}\sigma_{\mathrm{rand}}^{2/3}D^{4/3}}{n^{1/3}K^{2/3}}\right)$\footnote{Note that when $\psi(\bx)=0$, there is $\rsigma=\asigma+n\left\Vert \nabla f(\bx_{*})\right\Vert ^{2}=\asigma$
due to $\nabla f(\bx_{*})=\mathbf{0}$.}\saveFN\sfna\ & Random & Last & \ding{51}\tabularnewline
 & \cite{pmlr-v202-cha23a} & $\Omega\left(\frac{L^{1/3}\sigma_{\mathrm{any}}^{2/3}D^{4/3}}{n^{1/3}K^{2/3}}\right)$ & RR & Last & \ding{55}\tabularnewline
\hline 
 & \textbf{Ours} (Theorem \ref{thm:smooth-str-main}) & $\widetilde{\O}\left(\frac{L\asigma}{\mu^{2}K^{2}}\right)$ & Any & Last & \ding{51}\tabularnewline
$L$-smooth $f_{i}(\bx)$, & \cite{safran2020good} & $\Omega\left(\frac{G^{2}}{\mu K^{2}}\right)$\footnote{This lower bound is established under $L=\mu$ and additionally requires
$\left\Vert \nabla f_{i}(\bx)\right\Vert \leq G$. Under the same
condition, our above upper bound $\widetilde{\O}\left(\frac{L\asigma}{\mu^{2}K^{2}}\right)$
will be $\widetilde{\O}\left(\frac{G^{2}}{\mu K^{2}}\right)$ and
hence almost matches this lower bound.} & IG & Last & \ding{55}\tabularnewline
\cline{2-6} \cline{3-6} \cline{4-6} \cline{5-6} \cline{6-6} 
$\mu$-strongly convex $f(\bx)$ & \textbf{Ours} (Theorem \ref{thm:smooth-str-main}) & $\widetilde{\O}\left(\frac{L\rsigma}{\mu^{2}nK^{2}}\right)$\useFN\sfna\ & Random & Last & \ding{51}\tabularnewline
 & \cite{safran2021random,pmlr-v202-cha23a} & $\Omega\left(\frac{L\asigma}{\mu^{2}nK^{2}}\right)$ & RR/SO & Last & \ding{55}\tabularnewline
\hline 
\end{tabular}
\end{table}
\end{minipage}

\end{figure*}

\begin{figure*}[ht]

\noindent\begin{minipage}[t]{1\textwidth}%
\vspace{-0.25in}
\begin{table}[H]
\caption{\label{tab:lip}Summary of our new upper bounds and the previous fastest
rates for $G$-Lipschitz $f_{i}(\protect\bx)$ for large $K$. The
lower bound in this case has not been proved as far as we know. Here,
$D\triangleq\left\Vert \protect\bx_{*}-\protect\bx_{1}\right\Vert $.
All rates use the function value gap as the convergence criterion.
In the column of \textquotedbl Type\textquotedbl , \textquotedbl Any\textquotedbl{}
means the rate holds for whatever permutation not limited to RR/SO/IG.
\textquotedbl Avg\textquotedbl{} and \textquotedbl Last\textquotedbl{}
in the \textquotedbl Output\textquotedbl{} column stand for the average
iterate and the last iterate, respectively. In the last column, \textquotedbl\ding{51}\textquotedbl{}
means $\psi(\protect\bx)$ can be taken arbitrarily and \textquotedbl\ding{55}\textquotedbl{}
implies $\psi(\protect\bx)=0$.}
\vspace{0.1in}

\centering{}%
\begin{tabular}{|c|c|c|c|c|c|}
\hline 
\multicolumn{6}{|c|}{$F(\bx)=f(\bx)+\psi(\bx)$ where $f(\bx)=\frac{1}{n}\sum_{i=1}^{n}f_{i}(\bx)$
and $f_{i}(\bx)$ and $\psi(\bx)$ are convex}\tabularnewline
\hline 
Settings & References & Rate & Type & Output & $\psi(\bx)$\tabularnewline
\hline 
\multirow{2}{*}{$G$-Lipschitz $f_{i}(\bx)$} & \cite{bertsekas2011incremental} & $\O\left(\frac{GD}{\sqrt{K}}\right)$ & IG & Avg & \ding{51}\footnote{Only for $\psi(\bx)=\varphi(\bx)+\rmI_{\dom}(\bx)$ where $\varphi(\bx)$
is Lipschitz on $\dom$.}\tabularnewline
 & \textbf{Ours} (Theorem \ref{thm:lip-cvx-main}) & $\O\left(\frac{GD}{\sqrt{K}}\right)$ & Any & Last & \ding{51}\tabularnewline
\hline 
$G$-Lipschitz $f_{i}(\bx)$, $f(\bx)-f(\bx_{*})\geq\mu\left\Vert \bx-\bx_{*}\right\Vert ^{2}$\footnote{The original rate $\O\left(\frac{G^{2}}{\mu^{2}K}\right)$ is only
for $\psi(\bx)=\rmI_{\dom}(\bx)$ and proved w.r.t. $\left\Vert \bx_{K+1}-\bx_{*}\right\Vert ^{2}$.
We first remark that this setting implies that $\dom$ has to be compact
due to $\mu\left\Vert \bx-\bx_{*}\right\Vert ^{2}\leq f(\bx)-f(\bx_{*})\leq G\left\Vert \bx-\bx_{*}\right\Vert $.
Next, the original rate cannot give the $\O(1/K)$ bound on the function
value gap since $F(\bx_{K+1})-F(\bx_{*})\geq\mu\left\Vert \bx_{K+1}-\bx_{*}\right\Vert ^{2}$.
The result reported here is only for the convenience of comparison.}\saveFN\sfnb\  & \cite{nedic2001convergence} & \multicolumn{1}{c|}{$\O\left(\frac{G^{2}}{\mu K}\right)$\useFN\sfnb\ } & IG & Last\useFN\sfnb\  & \ding{51}\useFN\sfnb\ \tabularnewline
$G$-Lipschitz $f_{i}(\bx)$, $\mu$-strongly convex $\psi(\bx)$ & \textbf{Ours} (Theorem \ref{thm:lip-str-main}) & $\widetilde{\O}\left(\frac{G^{2}}{\mu K}\right)$ & Any & Last & \ding{51}\tabularnewline
\hline 
\end{tabular}
\end{table}
\end{minipage}

\end{figure*}

Solving the convex optimization problem in the form of
\[
\min_{\bx\in\R^{d}}F(\bx)\triangleq f(\bx)+\psi(\bx)\text{ where }f(\bx)\triangleq\frac{1}{n}\sum_{i=1}^{n}f_{i}(\bx)
\]
is arguably one of the most common tasks in machine learning. Many
well-known problems such as the LAD regression \cite{pollard1991asymptotics},
SVMs \cite{cortes1995support}, the Lasso \cite{tibshirani1996regression},
and any general problem from empirical risk minimization (ERM) \cite{shalev2014understanding}
fit the framework perfectly. A famous method for the problem with
a large $n$ (the standard case nowadays) is the stochastic gradient
descent (SGD) algorithm, which in every step only computes one gradient
differing from calculating $n$ gradients in the classic gradient
descent (GD) method and hence is more computationally affordable.

However, in practice, the shuffling gradient method is more widely
implemented than SGD. Unlike SGD uniformly sampling the index in every
step, shuffling gradient methods split the optimization procedure
into $K$ epochs each of which contains $n$ updates where the index
used in every step is determined by a permutation $\sigma$ of $\left\{ 1,2,\cdots,n\right\} $.
Especially, three ways to generate $\sigma$ are mainly used: (I)
Random Reshuffle (RR) where $\sigma$ is uniformly sampled without
replacement in every epoch. (II) Shuffle Once (SO) also known as Single
Shuffle where $\sigma$ in every epoch is identical to a pre-generated
permutation by uniform sampling without replacement. (III) Incremental
Gradient (IG) where $\sigma$ in every epoch is always the same as
some deterministic order.

Though shuffling gradient methods are empirically popular and successful
\cite{bottou2009curiously,bottou2012stochastic,bengio2012practical,recht2013parallel},
the theoretical understanding has not been well-developed for a long
time due to the biased gradient used in every step making the analysis
notoriously harder than SGD. \cite{gurbuzbalaban2021random} is the
first to demonstrate that RR provably converges faster than SGD under
certain conditions. Later on, different works give extensive studies
for the theoretical convergence of shuffling gradient methods (see
Section \ref{sec:related work} for a detailed discussion).

Despite the existing substantial progress, the convergence of shuffling
gradient methods is still not fully understood. Specifically, prior
works heavily rely on the following problem form that significantly
restricts the applicability:

\emph{The regularizer $\psi(\bx)$ is always omitted. }Almost all
previous studies narrow to the case of $\psi(\bx)=0$, which even
does not include the common constrained optimization problems (i.e.,
taking $\psi(\bx)=\rmI_{\dom}(\bx)$ where $\rmI_{\dom}(\bx)=0$ if
$\bx\in\dom$\footnote{This notation $\dom$ always denotes a nonempty closed convex set
in $\R^{d}$ throughout the paper.}, otherwise, $+\infty$). The only exception we know is \cite{mishchenko2022proximal},
which allows the existence of a general $\psi(\bx)$. However, \cite{mishchenko2022proximal}
requires either $f_{i}(\bx)$ or $\psi(\bx)$ to be strongly convex,
which can be easily violated in practice like the Lasso.

\emph{The components $f_{i}(\bx)$ are smooth.} This assumption also
reduces the generality of previous studies. Though, theoretically
speaking, smoothness is known to be necessary for shuffling gradient
methods to obtain better convergence rates than SGD \cite{pmlr-v97-nagaraj19a},
lots of problems are actually non-smooth in reality, for example,
the LAD regression and SVMs (with the hinge loss) mentioned before.
However, the convergence of shuffling-based algorithms without smoothness
is not yet fully established.

In addition to the above problems, there exists another critical gap
between practice and theory, i.e., the empirical success of the last
iterate versus the lack of theoretical interpretation. Returning the
last iterate is cheaper and hence commonly used in practice, but proving
the last-iterate rate is instead difficult. Even for GD and SGD, the
optimal last-iterate bounds had only been established recently (see
the discussion in Section \ref{sec:related work}). However, as for
shuffling gradient methods (regardless of which method in RR/SO/IG
is used), no existing bounds can be directly applied to the last-iterate
objective value (i.e., $F(\bx_{K+1})-F(\bx_{*})$ where $\bx_{*}$
is the optimal solution) and still match the lower bounds at the same
time. Even for the strongly convex problem, the previous rates are
studied for the metric $\left\Vert \bx_{K+1}-\bx_{*}\right\Vert ^{2}$
and thus cannot be converted to a bound on $F(\bx_{K+1})-F(\bx_{*})$\footnote{Conversely, $F(\bx_{K+1})-F(\bx_{*})=\Omega(\left\Vert \bx_{K+1}-\bx_{*}\right\Vert ^{2})$
always holds under strong convexity due to the first-order optimality.
Thus, the function value gap is always a stronger convergence criterion
than the squared distance from the optimum.} whenever $F(\bx)$ is non-smooth (the case like the Lasso) or $\nabla f(\bx_{*})\neq\mathbf{0}$
happens in constrained optimization (i.e., $\psi(\bx)=\rmI_{\dom}(\bx)$).
Motivated by this, we would like to ask the following core question
considered in this paper:
\begin{center}
\emph{For smooth/non-smooth $f_{i}(\bx)$ and general $\psi(\bx)$,
does the last iterate of shuffling gradient methods provably converge
in terms of the function value gap? If so, how fast is it?}
\par\end{center}

\subsection{Our Contributions}

We provide an affirmative answer to the question by proving the first
last-iterate convergence rates for shuffling gradient methods. We
particularly focus on two cases, smooth $f_{i}(\bx)$ and Lipschitz
$f_{i}(\bx)$, but without special requirements on $\psi(\bx)$.

Our new last-iterate rates for smooth components are summarized in
Table \ref{tab:smooth} (see Table \ref{tab:smooth-detailed} in Appendix
\ref{sec:detailed-table} for a more detailed summary due to limited
space). They almost match the existing lower bounds or the fastest
upper bounds of the average iterate.

Our new last-iterate rates for Lipschitz components are summarized
in Table \ref{tab:lip}. They are as fast as the previous best-known
upper bounds.

Limitations in our work are also discussed in Section \ref{sec:conclusion}.

\section{Related Work\label{sec:related work}}

When stating the convergence rate, we only provide the dependence
on $n$ and $K$ (for large $K$) without problem-dependent parameters
to save space unless the discussion will be beneficial from specifying
them in some cases. In the last-iterate bound of SGD, $nK$ serves
as time $T$.

\textbf{Convergence of shuffling gradient methods.} We first review
the prior works for general smooth components since smoothness is
known to be necessary for obtaining better rates compared to SGD \cite{pmlr-v97-nagaraj19a}.
The existing results in the non-smooth case will be discussed afterwards.
As for the special quadratic functions (note that they can be covered
by smooth components), we do not include a detailed discussion here
due to limited space. The reader could refer to \cite{gurbuzbalaban2019convergence,gurbuzbalaban2021random,pmlr-v97-haochen19a,pmlr-v119-rajput20a,safran2020good,ahn2020sgd,safran2021random,rajput2022permutationbased}
for recent progress.

For general $L$-smooth components that are not necessarily quadratic,
without strong convexity, \cite{pmlr-v97-nagaraj19a} first demonstrates
the rate $\O(1/\sqrt{nK})$ for the averaged output of RR when $\psi(\bx)=\rmI_{\dom}(\bx)$.
However, it does not only extra require bounded gradients but also
needs a compact $\dom$. Later, in the unconstrained setting (i.e.,
$\psi(\bx)=0$) without the Lipschitz condition, \cite{mishchenko2020random,nguyen2021unified}
improve the rate to $\O(1/n^{1/3}K^{2/3})$ for RR/SO and $\O(1/K^{2/3})$
for IG but still only analyze the average iterate. Their rate for
RR/SO has the optimal order on both $n,K$ and other problem-dependent
parameters as illustrated by the lower bound in \cite{pmlr-v202-cha23a}.

When $f(\bx)$ is assumed to be $\mu$-strongly convex, \cite{gurbuzbalaban2021random,pmlr-v97-haochen19a}
justify that RR provably converges under the unusual condition of
Lipschitz Hessian matrices. If assuming bounded gradients (which implicitly
means a bounded domain) instead of Lipschitz Hessian matrices, \cite{shamir2016without}
proves the rate $\widetilde{\O}(1/n)$ for RR/SO when $K=1$. For
RR, \cite{pmlr-v97-nagaraj19a} and \cite{ahn2020sgd}\footnote{The rates in \cite{gurbuzbalaban2019convergence,ying2018stochastic,ahn2020sgd,mishchenko2020random,nguyen2021unified}
are all obtained under $\psi(\bx)=0$ and use $\left\Vert \bx_{K+1}-\bx_{*}\right\Vert ^{2}$
as the metric (except \cite{nguyen2021unified}). Thus, they cannot
measure the convergence of the objective value in general. Nevertheless,
we here write them as the last-iterate convergence for convenience.}\saveFN\sfnc\  respectively obtain the rates $\widetilde{\O}\left(\frac{L^{2}}{\mu^{3}nK^{2}}\right)$
for the tail average iterate and $\O\left(\frac{L^{3}}{\mu^{4}nK^{2}}\right)$
for the last iterate. \cite{mishchenko2020random,nguyen2021unified}\useFN\sfnc\ 
further shave off the Lipschitz condition and still show the last-iterate
rate $\widetilde{\O}\left(\frac{L^{2}}{\mu^{3}nK^{2}}\right)$ for
RR/SO. However, none of them matches the lower bound $\Omega\left(\frac{L}{\mu^{2}nK^{2}}\right)$
for RR/SO \cite{pmlr-v119-rajput20a,safran2021random,pmlr-v202-cha23a}.
Until recently, \cite{pmlr-v202-cha23a} demonstrates that the tail
average iterate of RR indeed guarantees the rate $\widetilde{\O}\left(\frac{L}{\mu^{2}nK^{2}}\right)$.

As for IG, the first asymptotic last-iterate bound $\O(1/K^{2})$
is shown in \cite{gurbuzbalaban2019convergence}\useFN\sfnc\  while
it extra assumes bounded gradients. \cite{ying2018stochastic}\useFN\sfnc\ 
provides the last-iterate rate $\widetilde{\O}\left(\frac{L^{3}}{\mu^{4}K^{2}}\right)$
without the Lipschitz condition. \cite{mishchenko2020random,nguyen2021unified}\useFN\sfnc\ 
then improve it to $\widetilde{\O}\left(\frac{L^{2}}{\mu^{3}K^{2}}\right)$.
Though it seems like their rate has already matched the lower bound
$\Omega\left(\frac{1}{\mu K^{2}}\right)$ established under $L=\mu$
\cite{safran2020good}, however, this is not true as indicated by
our new last-iterate bound $\widetilde{\O}\left(\frac{L}{\mu^{2}K^{2}}\right)$
holding for any kind of permutation.

We would also like to mention some other relevant results. The lower
bound $\Omega(1/n^{2}K^{2})$ in \cite{pmlr-v202-cha23a} considers
the best possible permutation in the algorithm. A near-optimal rate
in this scenario is achieved by GraB \cite{lu2022grab}, in which
the permutation used in every epoch is chosen manually instead of
RR/SO/IG. In addition, \cite{mishchenko2022proximal} is the only
work we know that can apply to a general $\psi(\bx)$. However, their
results are all presented for the squared distance from the optimum,
which cannot be used in our settings as explained in Section \ref{sec:intro}. 

For non-smooth components, fewer results are known. The first convergence
rate $\O(1/\sqrt{K})$ established for the averaged output of IG dates
back to \cite{nedic2001incremental} where the authors consider constrained
optimization (i.e., $\psi(\bx)=\rmI_{\dom}(\bx)$). \cite{bertsekas2011incremental}
obtains the same $\O(1/\sqrt{K})$ bound for IG with a general $\psi(\bx)=\varphi(\bx)+\rmI_{\dom}(\bx)$
but where $\varphi(\bx)$ needs to be Lipschitz on $\dom$. The reader
could refer to \cite{bertsekas2011incrementalsurvey} for a detailed
survey. For the special case $K=1$, \cite{shamir2016without} presents
a rate $\O(1/\sqrt{n})$ for the averaged output of RR and SO, whose
proof is based on online optimization. If strong convexity holds,
the only general result that we are aware of is \cite{nedic2001convergence}
showing the rate $\O(1/K)$ for $\left\Vert \bx_{K+1}-\bx_{*}\right\Vert ^{2}$
output by IG, which cannot be converted to a bound on the objective
value due to $F(\bx_{K+1})-F(\bx_{*})\geq\Omega(\mu\left\Vert \bx_{K+1}-\bx_{*}\right\Vert ^{2})$. 

Apart from the basic shuffling gradient method considered in this
work, \cite{cai2023empirical} studies the primal-dual interpretation
of shuffled SGD and obtains fine-grained dependence on problem-dependent
parameters. \cite{pmlr-v162-tran22a} incorporates Nesterov's momentum
to improve the convergence rate. Besides the convex problem, shuffling-based
methods also provably converge in non-convex optimization, for example,
see \cite{solodov1998incremental,li2019incremental,pmlr-v139-tran21b,pauwels2021incremental,mohtashami2022characterizing,li2023convergence}.
Particularly, \cite{yu2023high} provides the last-iterate analysis
in the non-convex regime. Further generalization of shuffling gradient
methods has also been done in different works, e.g., distributed and
federated learning \cite{meng2019convergence,yun2022minibatch,malinovsky2022federated,sadiev2023federated,huang2023distributed}
and minimax optimization \cite{das2022sampling,cho2023sgda}.

\textbf{Last-iterate convergence of SGD. }\cite{zhang2004solving}
obtains the first finite-time bound, however, which is limited to
the linear prediction problem. \cite{rakhlin2011making} proves an
$\O(1/nK)$ rate for Lipschitz strongly convex functions but additionally
requires smoothness. The extra smooth condition is then removed by
\cite{shamir2013stochastic}, showing $\widetilde{\O}(1/\sqrt{nK})$
and $\widetilde{\O}(1/nK)$ rates for Lipschitz convex and strongly
convex problems, respectively. The first optimal high-probability
bounds for Lipschitz (strongly) convex functions are later established
in \cite{harvey2019tight,doi:10.1137/19M128908X}. Whereas, all works
mentioned so far are subject to compact domains. \cite{orabona2020blog}
is the first to remove this restriction and still gets the $\widetilde{\O}(1/\sqrt{nK})$
rate for Lipschitz convex optimization. Recently, a new proof by \cite{zamani2023exact}
gives the tight rate $\O(1/\sqrt{nK})$ for GD without compact domains
either. \cite{liu2024revisiting} then extends the idea in \cite{zamani2023exact}
to stochastic optimization and shows a unified proof working for various
settings. Results in \cite{liu2024revisiting} can apply to smooth
optimization without the Lipschitz condition and give the $\widetilde{\O}(1/\sqrt{nK})$
rate for convex objectives (improving upon the previous best rate
$\O(1/\sqrt[3]{nK})$ for smooth stochastic optimization \cite{moulines2011non})
and the $\widetilde{\O}(1/nK)$ bound for strongly convex functions.

\textbf{Independent work.} A manuscript \cite{cai2024last} appearing
on arXiv after the submission deadline of ICML 2024 also studied the
last-iterate convergence of shuffling methods based on the proof technique
initialized in \cite{zamani2023exact} and later developed by our
previous work \cite{liu2024revisiting}. Compared to \cite{cai2024last},
there are several differences we would like to discuss. We first remark
that the problem itself and the assumptions are different. Precisely,
the objective in \cite{cai2024last} does not contain the regularizer
(i.e., $\psi(\bx)=0$ in Assumption \ref{assu:basic}). In addition,
every component $f_{i}(\bx)$ is assumed to have the same smooth/Lipschitz
parameter (i.e., $L_{i}\equiv L$ in Assumption \ref{assu:smooth}
or $G_{i}\equiv G$ in Assumption \ref{assu:lip}). Moreover, \cite{cai2024last}
does not consider the strongly convex case (i.e., $\mu_{f}=0$ in
Assumption \ref{assu:str}). Consequently, when each $f_{i}(\bx)$
is smooth, their last-iterate convergence results for RR/SO/IG of
the plain shuffling gradient method (i.e., taking $\psi(\bx)=0$ in
Algorithm \ref{alg:shuffling-SGD}) can be covered by our Theorem
\ref{thm:smooth-cvx-main}; when each $f_{i}(\bx)$ is Lipschitz,
\cite{cai2024last} does not provide a last-iterate convergence result
for the plain shuffling gradient method. Hence, our Theorems \ref{thm:smooth-str-main},
\ref{thm:lip-cvx-main}, \ref{thm:lip-ada-main}, and \ref{thm:lip-str-main}
are still independently interesting.

It is noteworthy that \cite{cai2024last} obtains the convergence
rate for the weighted average iterate of the plain shuffling gradient
method for smooth $f_{i}(\bx)$, which is not considered in our paper.
They also develop the last-iterate convergence for the incremental
proximal method, which is a different algorithm from our proximal
shuffling gradient method (Algorithm \ref{alg:shuffling-SGD}).

\section{Preliminaries}

\textbf{Notation. }$\N$ denotes the set of natural numbers (excluding
$0$). $\left[K\right]\triangleq\left\{ 1,\cdots,K\right\} $ for
any $K\in\N$. $a\lor b$ and $a\land b$ represent $\max\left\{ a,b\right\} $
and $\min\left\{ a,b\right\} $. $\left\lfloor \cdot\right\rfloor $
and $\left\lceil \cdot\right\rceil $ are floor and ceiling functions.
$\langle\cdot,\cdot\rangle$ stands for the Euclidean inner product
on $\R^{d}$. $\left\Vert \cdot\right\Vert \triangleq\sqrt{\langle\cdot,\cdot\rangle}$
is the standard $\ell_{2}$ norm. Given an extended real-valued convex
function $g(\bx):\R^{d}\to\R\cup\left\{ +\infty\right\} $, the domain
of $g(\bx)$ is defined as $\mathrm{dom}g\triangleq\left\{ \bx\in\R^{d}:g(\bx)<+\infty\right\} $.
$\partial g(\bx)$ is the set of subgradients at $\bx$. $\nabla g(\bx)$
denotes an element in $\partial g(\bx)$ if $\partial g(\bx)\neq\varnothing$.
The Bregman divergence induced by $g(\bx)$ is $B_{g}(\bx,\by)\triangleq g(\bx)-g(\by)-\langle\nabla g(\by),\bx-\by\rangle$,$\forall\bx,\by\in\R^{d}$
satisfying $\partial g(\by)\neq\varnothing$. $\rmI_{\dom}(\bx)$
is the indicator function of the set $\dom$, i.e., $\rmI_{\dom}(\bx)=0$
if $\bx\in\dom$, otherwise, $\rmI_{\dom}(\bx)=+\infty$.

We study the following optimization problem in this paper
\[
\min_{\bx\in\R^{d}}F(\bx)\triangleq f(\bx)+\psi(\bx)\text{ where }f(\bx)\triangleq\frac{1}{n}\sum_{i=1}^{n}f_{i}(\bx).
\]
Next, we list the assumptions used in the analysis.
\begin{assumption}
\textup{\label{assu:minimizer}$\exists\bx_{*}\in\R^{d}$ such that
$F(\bx_{*})=\inf_{\bx\in\R^{d}}F(\bx)\in\R$.}
\end{assumption}

\begin{assumption}
\textup{\label{assu:basic}Each $f_{i}(\bx):\R^{d}\to\R$ is convex.
$\psi(\bx):\R^{d}\to\R\cup\left\{ +\infty\right\} $ is proper, closed
and convex}\footnote{Recall that this means the epigraph of $\psi(\bx)$, i.e., $\mathrm{epi}\psi\triangleq\left\{ (\bx,t)\in\R^{d}\times\R:\psi(\bx)\leq t\right\} $,
is a nonempty closed convex set.}\textup{. Moreover, $\exists\mu_{\psi}\geq0$ such that $B_{\psi}(\bx,\by)\geq\frac{\mu_{\psi}}{2}\left\Vert \bx-\by\right\Vert ^{2}$,
$\forall\bx,\by\in\R^{d}$ satisfying $\partial\psi(\by)\neq\varnothing$.}
\end{assumption}

Assumptions \ref{assu:minimizer} and \ref{assu:basic} are standard
in the optimization literature. Note that $\partial f_{i}(\bx)\neq\varnothing$
for all $\bx\in\R^{d}$ and $i\in\left[n\right]$ since $f_{i}(\bx)$
is convex on $\R^{d}$, so we use $\nabla f_{i}(\bx)$ to denote an
element in $\partial f_{i}(\bx)$ throughout the remaining paper.
We allow $\psi(\bx)$ to take the value of $+\infty$ and hence include
constrained optimization, e.g., setting $\psi(\bx)=\varphi(\bx)+\rmI_{\dom}(\bx)$
where $\varphi(\bx)$ can be some other regularizer. Moreover, the
parameter $\mu_{\psi}$ is possibly to be zero to better fit different
tasks, e.g., $\psi(\bx)=\left\Vert \bx\right\Vert _{1}$ in the Lasso
 implies $\mu_{\psi}=0$.
\begin{assumption}
\textup{\label{assu:smooth}$\exists L_{i}>0$ such that $\left\Vert \nabla f_{i}(\bx)-\nabla f_{i}(\by)\right\Vert $$\leq L_{i}\left\Vert \bx-\by\right\Vert $,
$\forall\bx,\by\in\R^{d},i\in\left[n\right]$.}
\end{assumption}

\begin{assumption}
\textup{\label{assu:str}$\exists\mu_{f}\geq0$ such that $B_{f}(\bx,\by)\geq\frac{\mu_{f}}{2}\left\Vert \bx-\by\right\Vert ^{2}$,
$\forall\bx,\by\in\R^{d}$.}
\end{assumption}

\begin{assumption}
\textup{\label{assu:lip}$\exists G_{i}>0$ such that $\left\Vert \nabla f_{i}(\bx)\right\Vert \leq G_{i}$,
$\forall\bx\in\R^{d},i\in\left[n\right]$.}
\end{assumption}

Assumptions \ref{assu:smooth}, \ref{assu:str} and \ref{assu:lip}
are commonly used in the related field. Note that Assumptions \ref{assu:str}
(when $\mu_{f}>0$) and \ref{assu:lip} cannot hold together on $\R^{d}$
(e.g., see \cite{nguyen2018sgd}). Hence, they will not be assumed
to be true simultaneously. In addition, our analysis relies on the
following well-known fact for smooth convex functions, whose proof
is omitted and can refer to, for example, Theorem 2.1.5 in \cite{nesterov2018lectures}.
\begin{lem}
\label{lem:cocore}Given a convex and differentiable function $g(\bx):\R^{d}\to\R$
satisfying $\left\Vert \nabla g(\bx)-\nabla g(\by)\right\Vert \leq L\left\Vert \bx-\by\right\Vert $,
$\forall\bx,\by\in\R^{d}$ for some $L>0$, then $\forall\bx,\by\in\R^{d}$,
\[
\frac{\left\Vert \nabla g(\bx)-\nabla g(\by)\right\Vert ^{2}}{2L}\leq B_{g}(\bx,\by)\leq\frac{L}{2}\left\Vert \bx-\by\right\Vert ^{2}.
\]
\end{lem}

\section{Proximal Shuffling Gradient Method\label{sec:convex}}

\begin{algorithm}[h]
\caption{\label{alg:shuffling-SGD}Proximal Shuffling Gradient Method}

\textbf{Input}: initial point $\bx_{1}\in\mathrm{dom}\psi$, stepsize
$\eta_{k}>0$.

\textbf{for} $k=1$ \textbf{to} $K$ \textbf{do}

$\quad$Generate a permutation $\left\{ \sigma_{k}^{i}:i\in\left[n\right]\right\} $
of $\left[n\right]$

$\quad$$\bx_{k}^{1}=\bx_{k}$

$\quad$\textbf{for} $i=1$ \textbf{to} $n$ \textbf{do}

$\qquad$$\bx_{k}^{i+1}=\bx_{k}^{i}-\eta_{k}\nabla f_{\sigma_{k}^{i}}(\bx_{k}^{i})$

$\quad$$\bx_{k+1}=\argmin_{\bx\in\R^{d}}n\psi(\bx)+\frac{\|\bx-\bx_{k}^{n+1}\|^{2}}{2\eta_{k}}$

\textbf{Return $\bx_{K+1}$}
\end{algorithm}

The algorithmic framework of the proximal shuffling gradient method
is shown in Algorithm \ref{alg:shuffling-SGD} where the number of
epochs $K$ is assumed to satisfy $K\geq2$ to avoid algebraic issues
in the proof. Unlike \cite{mishchenko2022proximal}, we do not specify
how the permutation $\sigma_{k}$ is generated in the $k$-th epoch.
Hence, Algorithm \ref{alg:shuffling-SGD} actually includes several
different algorithms depending on how $\sigma_{k}$ is defined.
\begin{example}
\label{exa:prox-RR}Algorithm \ref{alg:shuffling-SGD} recovers the
ProxRR algorithm studied by \cite{mishchenko2022proximal} if $\sigma_{k}$
is uniformly sampled without replacement in every epoch.
\end{example}

\begin{example}
\label{exa:prox-SO}Algorithm \ref{alg:shuffling-SGD} reduces to
the ProxSO algorithm introduced by \cite{mishchenko2022proximal}
when $\sigma_{k}=\sigma_{0}$ for all epochs where $\sigma_{0}$ is
a pre-generated permutation by uniform sampling without replacement.
\end{example}

\begin{example}
\label{exa:prox-IG}Algorithm \ref{alg:shuffling-SGD} includes the
proximal IG method developed in \cite{kibardin1979decomposition}
as a subcase by considering $\sigma_{k}=\sigma_{0}$ for all epochs
where $\sigma_{0}$ is a deterministic pre-generated permutation (e.g.,
$\sigma_{0}^{i}=i,\forall i\in\left[n\right]$).
\end{example}

It is worth noting that $\bx_{k+1}\in\mathrm{dom}\psi$ holds for
all $k\in\left[K\right]$ since we assume $\psi(\bx)$ is proper.
Therefore, all $F(\bx_{k})$ take values on $\R$ and hence are well-defined. 

\subsection{Last-Iterate Convergence Rates}

In the convergence rate, we denote the initial distance from the optimum
by $D\triangleq\left\Vert \bx_{*}-\bx_{1}\right\Vert $ and the optimal
function value by $F_{*}\triangleq F(\bx_{*})$. The full statements
of all theorems presented in this subsection along with their proofs
are deferred into Appendix \ref{sec:full-thms}. The hidden $\log$
factor in the tilde $\widetilde{\O}(\cdot)$ will also be explicitly
given in the full theorem.

\subsubsection{Smooth Functions}

This subsection focuses on smooth components $f_{i}(\bx)$. To simplify
expressions, we define the following symbols. $\bL\triangleq\frac{1}{n}\sum_{i=1}^{n}L_{i}$
and $L^{*}\triangleq\max_{i\in\left[n\right]}L_{i}$ respectively
stand for the averaged and the largest smooth parameter. In addition,
similar to \cite{mishchenko2020random,nguyen2021unified}, we use
the quantity $\asigma\triangleq\frac{1}{n}\sum_{i=1}^{n}\left\Vert \nabla f_{i}(\bx_{*})\right\Vert ^{2}$
to measure the uncertainty induced by shuffling. Besides, we introduce
another term $\rsigma\triangleq\asigma+n\left\Vert \nabla f(\bx_{*})\right\Vert ^{2}$,
where the appearance of the extra $n\left\Vert \nabla f(\bx_{*})\right\Vert ^{2}$
is due to the proximal step. We remark that the term $n\left\Vert \nabla f(\bx_{*})\right\Vert ^{2}$
also showed up when studying ProxRR and ProxSO before in \cite{mishchenko2022proximal}.
\begin{thm}
\label{thm:smooth-cvx-main}Under Assumptions \ref{assu:minimizer},
\ref{assu:basic} (with $\mu_{\psi}=0$) and \ref{assu:smooth}:

Regardless of how the permutation $\sigma_{k}$ is generated in every
epoch, taking the stepsize $\eta_{k}=\frac{1}{4n\sqrt{2\bL L^{*}(1+\log K)}}\land\frac{D^{2/3}}{n\sqrt[3]{\bL\asigma K(1+\log K)}}$,
$\forall k\in\left[K\right]$, Algorithm \ref{alg:shuffling-SGD}
guarantees
\[
F(\bx_{K+1})-F_{*}\leq\widetilde{\O}\left(\frac{\sqrt{\bL L^{*}}D^{2}}{K}+\frac{\bL^{1/3}\sigma_{\mathrm{any}}^{2/3}D^{4/3}}{K^{2/3}}\right).
\]

Suppose the permutation $\sigma_{k}$ is uniformly sampled without
replacement, taking the stepsize $\eta_{k}=\frac{1}{4n\sqrt{2\bL L^{*}(1+\log K)}}\land\frac{D^{2/3}}{\sqrt[3]{n^{2}\bL\rsigma K(1+\log K)}}$,
$\forall k\in\left[K\right]$, Algorithm \ref{alg:shuffling-SGD}
guarantees
\[
\E\left[F(\bx_{K+1})-F_{*}\right]\leq\widetilde{\O}\left(\frac{\sqrt{\bL L^{*}}D^{2}}{K}+\frac{\bL^{1/3}\sigma_{\mathrm{rand}}^{2/3}D^{4/3}}{n^{1/3}K^{2/3}}\right).
\]
\end{thm}

\begin{rem}
\label{rem:sample}When we say the permutation $\sigma_{k}$ is uniformly
sampled without replacement hereinafter, it is not necessary to think
that $\sigma_{k}$ is re-sampled in every epoch as in Example \ref{exa:prox-RR}.
Instead, it can also be sampled in advance like Example \ref{exa:prox-SO}.
Or even more generally, one can sample a permutation in advance and
use it for an arbitrary number of epochs then re-sample a new permutation
and repeat this procedure. Hence, our results are not only true for
ProxRR and ProxSO but also hold for the general form of Algorithm
\ref{alg:shuffling-SGD}.
\end{rem}

As far as we know, Theorem \ref{thm:smooth-cvx-main} is the first
to provide the last-iterate convergence rate of Algorithm \ref{alg:shuffling-SGD}
for smooth components but without strong convexity. Even in the simplest
case of $\psi(\bx)=0$, the previous best-known result only works
for the averaged output $\bx_{K+1}^{\mathrm{avg}}\triangleq\frac{1}{K}\sum_{k=1}^{K}\bx_{k+1}$
\cite{mishchenko2020random,nguyen2021unified}. 

Next, we would like to discuss the convergence rate in more detail.
For the first case (i.e., regardless of how the permutation $\sigma_{k}$
is generated in every epoch), up to a logarithmic factor, our last-iterate
result matches the fastest existing rate of $\bx_{K+1}^{\mathrm{avg}}$
output by the IG method when $\psi(\bx)=0$ \cite{mishchenko2020random}.
But as one can see, our theorem is more general and can apply to various
situations (e.g., $\psi(\bx)=\left\Vert \bx\right\Vert _{1}$) for
whatever the permutation $\sigma_{k}$ is used.

When the permutation $\sigma_{k}$ is uniformly sampled without replacement
(including RR and SO as special cases), let us first consider the
unconstrained setting without any regularizer, i.e., $\psi(\bx)=0$.
Note that now $\nabla f(\bx_{*})=\mathbf{0}$ and it implies $\rsigma=\asigma+n\left\Vert \nabla f(\bx_{*})\right\Vert ^{2}=\asigma$.
Hence, our rate matches the last-iterate lower bound $\Omega\left(\frac{L^{1/3}\sigma_{\mathrm{any}}^{2/3}D^{4/3}}{n^{1/3}K^{2/3}}\right)$
established under the condition $L_{i}\equiv L$ and $\psi(\bx)=0$
for large $K$ (which is only proved for RR) in \cite{pmlr-v202-cha23a}
up to a logarithmic term. When a general $\psi(\bx)$ exists, the
only difference is to replace $\asigma$ with the larger quantity
$\rsigma$. A similar penalty also appeared in \cite{mishchenko2022proximal}
when studying ProxRR and ProxSO. Whether $\rsigma$ can be improved
to $\asigma$ remains unclear to us.

Moreover, we would like to talk about the stepsize. Suppose $L_{i}\equiv L$
for simplicity, then our stepsizes for both cases are almost the same
as the choices in \cite{mishchenko2020random} and only different
by $1+\log K$ in the denominator. However, this distinction plays
a key role in our analysis of proving the last-iterate bound. The
reader could refer to Theorem \ref{thm:smooth-cvx-full} and its proof
for why we need it.
\begin{thm}
\label{thm:smooth-str-main}Under Assumptions \ref{assu:minimizer},
\ref{assu:basic} (with $\mu_{\psi}=0$), \ref{assu:smooth} and \ref{assu:str}
(with $\mu_{f}>0$), let $\bar{\kappa}_{f}\triangleq\frac{\bL}{\mu_{f}}$
and $\kappa_{f}^{*}\triangleq\frac{L^{*}}{\mu_{f}}$: 

Regardless of how the permutation $\sigma_{k}$ is generated in every
epoch, taking the stepsize $\eta_{k}=\frac{1}{4n\sqrt{2\bL L^{*}(1+\log K)}}\land\frac{u}{n\mu_{f}K}$
where $u=1\lor\log\frac{\mu_{f}^{3}D^{2}K^{2}}{\bL\asigma(1+\log K)}$,
$\forall k\in\left[K\right]$, Algorithm \ref{alg:shuffling-SGD}
guarantees
\[
F(\bx_{K+1})-F_{*}\leq\widetilde{\O}\left(\frac{\sqrt{\bL L^{*}}D^{2}}{Ke^{\frac{K}{\sqrt{\bar{\kappa}_{f}\kappa_{f}^{*}}}}}+\frac{\bL\asigma}{\mu_{f}^{2}K^{2}}\right).
\]

Suppose the permutation $\sigma_{k}$ is uniformly sampled without
replacement, taking the stepsize $\eta_{k}=\frac{1}{4n\sqrt{2\bL L^{*}(1+\log K)}}\land\frac{u}{n\mu_{f}K}$
where $u=1\lor\log\frac{n\mu_{f}^{3}D^{2}K^{2}}{\bL\rsigma(1+\log K)}$,
$\forall k\in\left[K\right]$, Algorithm \ref{alg:shuffling-SGD}
guarantees
\[
\E\left[F(\bx_{K+1})-F_{*}\right]\leq\widetilde{\O}\left(\frac{\sqrt{\bL L^{*}}D^{2}}{Ke^{\frac{K}{\sqrt{\bar{\kappa}_{f}\kappa_{f}^{*}}}}}+\frac{\bL\rsigma}{\mu_{f}^{2}nK^{2}}\right).
\]
\end{thm}

Another result in this subsection is Theorem \ref{thm:smooth-str-main}
showing the last-iterate convergence guarantee when strong convexity
is additionally assumed. For simplicity, we only present the case
of $\mu_{\psi}=0$ and $\mu_{f}>0$ here. We refer the reader to Theorem
\ref{thm:smooth-str-full} in the appendix for the full statement.
To our best knowledge, this is also the first last-iterate convergence
result w.r.t. the function value gap for shuffling gradient methods
under strong convexity when a general $\psi(\bx)$ is allowed. In
fact, even for the special case $\psi(\bx)=\rmI_{\dom}(\bx)$, i.e.,
constrained optimization, the previous bounds could fail as explained
in Section \ref{sec:intro}.

As before, let us first take a look at the case of whatever the permutation
$\sigma_{k}$ is used. If we specify to the IG method, \cite{safran2020good}
provided a last-iterate lower bound $\Omega\left(\frac{G^{2}}{\mu_{f}K^{2}}\right)$
for large $K$ with $\psi(\bx)=0$, $L_{i}\equiv\mu_{f}$ and an additional
requirement $\left\Vert \nabla f_{i}(\bx)\right\Vert \leq G$. Under
the same condition, our rate degenerates to $\widetilde{\O}\left(\frac{G^{2}}{\mu_{f}K^{2}}\right)$
almost matching the lower bound in \cite{safran2020good} when $K$
is large. However, for the arbitrary permutation case, the last-iterate
lower bound under the condition $\psi(\bx)=0$ and $L_{i}\equiv L$
for large $K$ is $\Omega\left(\frac{L\asigma}{\mu_{f}^{2}n^{2}K^{2}}\right)$
\cite{pmlr-v202-cha23a}, which is better than our rate by a factor
$\O\left(\frac{1}{n^{2}}\right)$. We remark that this is to be expected
because, roughly speaking, our upper bound is equivalent to say (for
$L_{i}\equiv L$ and large $K$)
\[
\inf_{\mathrm{stepsizes}\,\eta_{k}}\sup_{\substack{\mathrm{functions}\,F(\bx)\\
\mathrm{permutations}\,\sigma_{k}
}
}F(\bx_{K+1})-F_{*}\leq\widetilde{\O}\left(\frac{L\asigma}{\mu_{f}^{2}K^{2}}\right),
\]
whereras the lower bound in \cite{pmlr-v202-cha23a} is proved for
\[
\sup_{\mathrm{functions}\,F(\bx)}\inf_{\substack{\mathrm{stepsizes}\,\eta_{k}\\
\mathrm{permutations}\,\sigma_{k}
}
}F(\bx_{K+1})-F_{*}\geq\Omega\left(\frac{L\asigma}{\mu_{f}^{2}n^{2}K^{2}}\right).
\]
Hence, to achieve the lower bound in \cite{pmlr-v202-cha23a}, one
does not only need to specify stepsizes $\eta_{k}$ but also has to
choose the permutation $\sigma_{k}$ according to the problem. Recently,
\cite{lu2022grab} proposed GraB, in which the permutation for the
current epoch is selected based on the information from previous epochs.
GraB is provably faster than RR/SO/IG and achieves the rate $\widetilde{\O}\left(\frac{H^{2}L^{2}\asigma}{\mu_{f}^{3}n^{2}K^{2}}\right)$
differing from the lower bound in \cite{pmlr-v202-cha23a} by a factor
of $\O\left(\frac{H^{2}L}{\mu_{f}}\right)$ (the constant $H$ here
is named herding bound, see \cite{lu2022grab} for details).

Next, for the case of uniform sampling without replacement (including
RR and SO as special cases), we again first check the rate when $\psi(\bx)=0$.
Recall now $\rsigma=\asigma$ due to $\left\Vert \nabla f(\bx_{*})\right\Vert =0$.
Thus, our bound for large $K$ is $\widetilde{\O}\left(\frac{\bL\asigma}{\mu_{f}^{2}nK^{2}}\right)$,
which is nearly optimal compared with the lower bound $\Omega\left(\frac{L\asigma}{\mu_{f}^{2}nK^{2}}\right)$
established for RR/SO under $L_{i}\equiv L$ \cite{safran2021random,pmlr-v202-cha23a}.
As far as we are aware, this is the first last-iterate result not
only being nearly tight in $n$ and $K$ but also optimal in the parameter
$\Omega\left(\frac{L}{\mu_{f}^{2}}\right)$. In contrast, the previous
best bound with the correct dependence only applies to the tail average
iterate $\bx_{K+1}^{\mathrm{tail}}\triangleq\frac{1}{\left\lfloor \frac{K}{2}\right\rfloor +1}\sum_{k=\left\lceil \frac{K}{2}\right\rceil }^{K}\bx_{k+1}$
\cite{pmlr-v202-cha23a}. Therefore, Theorem \ref{thm:smooth-str-main}
fulfills the existing gap when $\psi(\bx)=0$ for large $K$. More
importantly, our result can also be used for any general $\psi(\bx)$
with the only tradeoff of changing $\asigma$ to $\rsigma$ but not
hurting any other term.

Lastly, we would like to talk about the stepsize used in Theorem \ref{thm:smooth-str-main}.
As one can check, when $L_{i}\equiv L$, it is similar to the stepsize
used in \cite{pmlr-v202-cha23a} to obtain the rate for $\bx_{K+1}^{\mathrm{tail}}$.
The only significant difference is still the extra term $1+\log K$.
However, as mentioned in the discussion about the stepsize used in
Theorem \ref{thm:smooth-cvx-main} (i.e., the general convex case),
this change is crucial in our analysis.

\subsubsection{Lipschitz Functions}

In this subsection, we consider the case of $f_{i}(\bx)$ all being
Lipschitz. We employ the notation $\bG\triangleq\frac{1}{n}\sum_{i=1}^{n}G_{i}$
to denote the averaged Lipschitz parameter.
\begin{thm}
\label{thm:lip-cvx-main}Under Assumptions \ref{assu:minimizer},
\ref{assu:basic} (with $\mu_{\psi}=0$) and \ref{assu:lip}, regardless
of how the permutation $\sigma_{k}$ is generated in every epoch:

Taking the stepsize $\eta_{k}=\frac{\eta}{\sqrt{k}}$, $\forall k\in\left[K\right]$
with $\eta=\frac{D}{n\bG}$ or $\eta_{k}=\frac{\eta}{\sqrt{K}}$,
$\forall k\in\left[K\right]$ with $\eta=\frac{D}{n\bG\sqrt{1+\log K}}$,
Algorithm \ref{alg:shuffling-SGD} guarantees
\[
F(\bx_{K+1})-F_{*}\leq\widetilde{\O}\left(\frac{\bG D}{\sqrt{K}}\right).
\]

Taking the stepsize $\eta_{k}=\eta\frac{K-k+1}{K^{3/2}}$, $\forall k\in\left[K\right]$
with $\eta=\frac{D}{n\bG}$, Algorithm \ref{alg:shuffling-SGD} guarantees
\[
F(\bx_{K+1})-F_{*}\leq\O\left(\frac{\bG D}{\sqrt{K}}\right).
\]
\end{thm}

\begin{rem}
The values of $\eta$ in Theorem \ref{thm:lip-cvx-main} are set to
obtain the best dependence on parameters like $\bG$ and $D$. In
fact, Algorithm \ref{alg:shuffling-SGD} converges for all three stepsizes
with arbitrary $\eta>0$ but suffers a worse dependence on parameters.
We refer the reader to Theorem \ref{thm:lip-cvx-full} for the full
statement.
\end{rem}

We first focus on the general convex case. As shown in Theorem \ref{thm:lip-cvx-main},
regardless of how the permutation $\sigma_{k}$ is generated, the
last-iterate convergence rate is $\widetilde{\O}(\bG D/\sqrt{K})$
for stepsizes $\eta_{k}=\frac{\eta}{\sqrt{k}}$ or $\eta_{k}=\frac{\eta}{\sqrt{K}}$
with a theoretically fine-tuned $\eta$. The extra $\log$ factor
can be removed by taking the stepsize $\eta_{k}=\eta\frac{K-k+1}{K^{3/2}}$
that has a linear decay rate. This kind of stepsize was originally
introduced by \cite{zamani2023exact} to show the $\O(\bG D/\sqrt{K})$
last-iterate rate can be achieved for the projected subgradient descent
method, i.e., $n=1$ and $\psi(\bx)=\rmI_{\dom}(\bx)$. Here, we prove
it can also be applied to the proximal shuffling gradient method.

Theorem \ref{thm:lip-cvx-main} provides the first concrete theoretical
evidence that the last iterate of Algorithm \ref{alg:shuffling-SGD}
for any permutation is comparable with the averaged output by the
proximal IG method, which was known to be able to achieve the rate
$\O(GD/\sqrt{K})$ when $G_{i}\equiv G$ \cite{bertsekas2011incremental}.

Moreover, we compare the number of individual gradient evaluations
required to make the function value gap of the last iterate be at
most $\epsilon$ between Algorithm \ref{alg:shuffling-SGD}, GD and
SGD. 
\begin{itemize}
\item Algorithm \ref{alg:shuffling-SGD}: $\O\left(\frac{(\sum_{i=1}^{n}G_{i})^{2}D^{2}}{n\epsilon^{2}}\right)$
by Theorem \ref{thm:lip-cvx-main}.
\item GD: $\O\left(\frac{(\sum_{i=1}^{n}G_{i})^{2}D^{2}}{n\epsilon^{2}}\right)$
by prior works in Section \ref{sec:related work}.
\item SGD: $\O\left(\frac{(\sum_{i=1}^{n}G_{i}^{2})D^{2}}{n\epsilon^{2}}\right)$
by prior works in Section \ref{sec:related work}.
\end{itemize}
This indicates that the sample complexity of Algorithm \ref{alg:shuffling-SGD}
matches GD's but is worse than SGD's by at most a factor of $\O(n)$
due to $\sum_{i=1}^{n}G_{i}^{2}\leq(\sum_{i=1}^{n}G_{i})^{2}\leq n\sum_{i=1}^{n}G_{i}^{2}$. 
\begin{thm}
\label{thm:lip-ada-main}Under Assumptions \ref{assu:minimizer},
\ref{assu:basic} (with $\mu_{\psi}=0$) and \ref{assu:lip}, regardless
of how the permutation $\sigma_{k}$ is generated in every epoch,
taking the stepsize $\eta_{k}=\frac{r_{k}}{4\sqrt{3}n\bG\sqrt{k}(1+\log k)},\forall k\in\N$
where $r_{k}=r\lor\max_{\ell\in\left[k\right]}\left\Vert \bx_{\ell}-\bx_{1}\right\Vert $
for some $r>0$, then Algorithm \ref{alg:shuffling-SGD} guarantees
for large enough $K$
\[
F(\bx_{K+1})-F_{*}\leq\widetilde{\O}\left(\frac{\bG\left(D\lor r\right)}{\sqrt{K}}\right).
\]
\end{thm}

In Theorem \ref{thm:lip-cvx-main}, one may have noticed that it needs
the prior knowledge of $\bG$ and $D$ to obtain the optimal linear
dependence on them. In practice, estimating $\bG$ could be relatively
easy. In contrast, the requirement of knowing $D$ is however hard
to satisfy. Fortunately, several methods have been proposed to overcome
this difficulty. Here, we borrow the idea from \cite{pmlr-v202-ivgi23a}
that is to introduce the term $r_{k}=r\lor\max_{\ell\in\left[k\right]}\left\Vert \bx_{\ell}-\bx_{1}\right\Vert $
in the stepsize as shown in Theorem \ref{thm:lip-ada-main}. As a
result, we can provide an asymptotic last-iterate convergence rate
with a linear dependence on $D$ without knowing it. However, we are
not able to establish a finite-time rate that also has a linear dependence
on $D$, which is left as a future research direction.
\begin{thm}
\label{thm:lip-str-main}Under Assumptions \ref{assu:minimizer},
\ref{assu:basic} (with $\mu_{\psi}>0$) and \ref{assu:lip}, regardless
of how the permutation $\sigma_{k}$ is generated in every epoch,
taking the stepsize $\eta_{k}=\frac{2}{n\mu_{\psi}k},\forall k\in\left[K\right]$,
Algorithm \ref{alg:shuffling-SGD} guarantees
\[
F(\bx_{K+1})-F_{*}\leq\widetilde{\O}\left(\frac{\mu_{\psi}D^{2}}{K^{2}}+\frac{\bG^{2}}{\mu_{\psi}K}\right).
\]
\end{thm}

Lastly, we consider the case of strongly convex $\psi(\bx)$. As stated
in Theorem \ref{thm:lip-str-main}, Algorithm \ref{alg:shuffling-SGD}
guarantees the rate of $\widetilde{\O}(1/K)$ w.r.t. the function
value gap for the stepsize $\eta_{k}=\frac{2}{n\mu_{\psi}k}$. In
contrast, the same rate was previously only known to hold for $\left\Vert \bx_{K+1}-\bx^{*}\right\Vert ^{2}$
in IG \cite{nedic2001convergence}. Theorem \ref{thm:lip-str-full}
in the appendix will generalize to $\eta_{k}=\frac{m}{n\mu_{\psi}k}$
for any $m\in\N$ and prove the rate $\widetilde{O}(m/K)$.

By a similar comparison (after Theorem \ref{thm:lip-cvx-main}), the
sample complexity of Algorithm \ref{alg:shuffling-SGD} in the case
$\mu_{\psi}>0$ is as good as GD's but still worse than SGD's by at
most an $\O(n)$ term.

\section{Proof Idea\label{sec:idea}}

In this section, we outline some key steps in the analysis of Theorem
\ref{thm:smooth-cvx-main} and highlight our novel techniques.

From a high-level view, our proof is inspired by the recent progress
on the last-iterate convergence of GD \cite{zamani2023exact}. \cite{zamani2023exact}
designs an auxiliary sequence $\bz_{k}$ for $k\in\left[K\right]$,
each of which is a convex combination of $\bx_{*},\bx_{1},\cdots,\bx_{k}$
(say $\bz_{k}=w_{k,0}\bx_{*}+\sum_{\ell=1}^{k}w_{k,\ell}\bx_{\ell}$
where $w_{k,\ell}\geq0$ and $\sum_{\ell=0}^{k}w_{k,\ell}=1$), and
then bounds $F(\bx_{k+1})-F(\bz_{k})$ instead of $F(\bx_{k+1})-F(\bx_{*})$.
By using $F(\bz_{k})\leq w_{k,0}F(\bx_{*})+\sum_{\ell=1}^{k}w_{k,\ell}\bx_{\ell}$,
one can finally prove $\sum_{k=1}^{K}p_{k}(F(\bx_{k+1})-F(\bz_{k}))\geq\Omega(F(\bx_{K+1})-F(\bx_{*}))$
for properly picked $w_{k,\ell}$ where $p_{k}$ is another carefully
chosen sequence and obtain the rate for the last iterate.

However, to make the above argument work, it indeed requires $F(\bx_{k+1})-F(\bz_{k})\leq\text{good terms}$,
where \textquotedbl good terms\textquotedbl{} means some quantities
we can finally bound (e.g., a telescoping sum). But in shuffling gradient
methods, a key difference appears: there is only (see Lemmas \ref{lem:core}
and \ref{lem:smooth-res})
\[
F(\bx_{k+1})-F(\bz_{k})\leq\text{good terms}+\O(B_{f}(\bz_{k},\bx_{*})+R_{k}),
\]
where $R_{k}$ is a \textquotedbl bad\textquotedbl{} residual due
to shuffling. Fortunately, $R_{k}$ can always be bounded (Lemma \ref{lem:var})
and thus can be included in \textquotedbl good terms\textquotedbl .
In contrast, bounding another new extra term $B_{f}(\bz_{k},\bx_{*})$
is more tricky, which is significantly distinct from the previous
works. Though the existence of $B_{f}(\bz_{k},\bx_{*})$, we still
sum from $k=1$ to $K$ to obtain
\begin{equation}
F(\bx_{K+1})-F(\bx_{*})\leq\text{good terms}+\O\left(\sum_{k=1}^{K}B_{f}(\bz_{k},\bx_{*})\right).\label{eq:1}
\end{equation}
The first key observation in the analysis is that $B_{f}(\bz_{k},\bx_{*})=\O(\sum_{\ell=1}^{k}B_{f}(\bx_{\ell},\bx_{*}))$
because $\bz_{k}$ is a convex combination of $\bx_{*},\bx_{1},\cdots,\bx_{k}$
and $B_{f}(\cdot,\cdot)$ is convex in the first argument. The second
important notice is that $F(\bx_{K+1})-F(\bx_{*})\geq B_{f}(\bx_{K+1},\bx_{*})$
always holds for whatever $\psi(\bx)$ is (see the proof of Theorem
\ref{thm:smooth-cvx-full} for details). Therefore, we have
\[
B_{f}(\bx_{K+1},\bx_{*})\leq\text{good terms}+\O\left(\sum_{k=1}^{K}B_{f}(\bx_{k},\bx_{*})\right).
\]
But this is still not enough to bound $F(\bx_{K+1})-F(\bx_{*})$.
To deal with this issue, departing from the previous works that only
prove one inequality for time $K$, we instead apply the above procedure
for every $k\in\left[K\right]$ to get
\[
B_{f}(\bx_{k+1},\bx_{*})\leq\text{good terms}+\O\left(\sum_{\ell=1}^{k}B_{f}(\bx_{\ell},\bx_{*})\right).
\]
We remark that this new way requires us to define the auxiliary sequence
$\left\{ \bz_{\ell},\forall\ell\in\left[k\right]\right\} $ carefully.
Specifically, for each $k\in\left[K\right]$, we let the auxiliary
sequence $\left\{ \bz_{\ell},\forall\ell\in\left[k\right]\right\} $
depend on the current time $k$. Hence, we indeed construct a total
of $K$ different auxiliary sequences. Next, we develop a new algebraic
inequality (Lemma \ref{lem:ineq}) to recursively bound all $B_{f}(\bx_{k},\bx_{*})$.
Equipped with the bound on $B_{f}(\bx_{k},\bx_{*})$, we finally invoke
(\ref{eq:1}) again to get the desired result. The reader could refer
to the appendix for detailed proof.

\section{Limitation\label{sec:conclusion}}

Here we list some limitations in our work and look forward to them
being addressed in the future. For smooth optimization, the extra
factor $1+\log K$ in the stepsizes seems necessary in our analysis.
Whether, and if so how, it can be removed is an important problem.
In addition, our results are limited to constant stepsizes depending
on the number of epochs $K$. Finding time-varying stepsizes that
still achieve the optimal dependence on both $n,K$ and other problem-dependent
parameters is another interesting task. For non-smooth optimization,
i.e., Lipschitz components, our theorems are stated for any kind of
permutation and can only match the sample complexity of GD but be
worse than SGD's. Hence, it is worth investigating whether shuffling
gradient methods can benefit from random permutations and be as good
as SGD under certain types of shuffling, for example, RR and SO.

\section*{Acknowledgements}

This work is supported by the NSF grant CCF-2106508. Zhengyuan Zhou
also gratefully acknowledges the 2024 NYU CGEB faculty grant.

\section*{Impact Statement}

This paper presents a theory work showing the last-iterate convergence
of shuffling gradient methods. There are no ethical impacts and expected
societal implications that we feel must be specifically highlighted
here.

\bibliographystyle{icml2024}
\bibliography{ref}

\appendix
\onecolumn

\section{Detailed Comparison\label{sec:detailed-table}}

We provide a more detailed comparison for smooth components in Table
\ref{tab:smooth-detailed}, where the case of strongly convex $\psi(\bx)$
and the previous fastest upper bounds for all cases are included.

\noindent\begin{minipage}[t]{1\columnwidth}%
\vspace{-0.25in}
\begin{table}[H]
\caption{\label{tab:smooth-detailed}Summary of our new upper bounds and the
previous best upper and lower bounds for $L$-smooth $f_{i}(\protect\bx)$
for large $K$. Here, $\protect\asigma\triangleq\frac{1}{n}\sum_{i=1}^{n}\left\Vert \nabla f_{i}(\protect\bx_{*})\right\Vert ^{2}$,
$\protect\rsigma\triangleq\protect\asigma+n\left\Vert \nabla f(\protect\bx_{*})\right\Vert ^{2}$
and $D\triangleq\left\Vert \protect\bx_{*}-\protect\bx_{1}\right\Vert $.
All rates use the function value gap as the convergence criterion.
In the column of \textquotedbl Type\textquotedbl , \textquotedbl Any\textquotedbl{}
means the rate holds for whatever permutation not limited to RR/SO/IG.
\textquotedbl Random\textquotedbl{} refers to the uniformly sampled
permutation but is not restricted to RR/SO (see Remark \ref{rem:sample}
for a detailed explanation). \textquotedbl Avg\textquotedbl , \textquotedbl Last\textquotedbl{}
and \textquotedbl Tail\textquotedbl{} in the \textquotedbl Output\textquotedbl{}
column stand for $\protect\bx_{K+1}^{\mathrm{avg}}\triangleq\frac{1}{K}\sum_{k=1}^{K}\protect\bx_{k+1}$,
$\protect\bx_{K+1}$ and $\protect\bx_{K+1}^{\mathrm{tail}}\triangleq\frac{1}{\left\lfloor \frac{K}{2}\right\rfloor +1}\sum_{k=\left\lceil \frac{K}{2}\right\rceil }^{K}\protect\bx_{k+1}$,
respectively. In the last column, \textquotedbl\ding{51}\textquotedbl{}
means $\psi(\protect\bx)$ can be taken arbitrarily and \textquotedbl\ding{55}\textquotedbl{}
implies $\psi(\protect\bx)=0$.}
\vspace{0.1in}

\centering{}%
\begin{tabular}{|>{\centering}m{0.135\textwidth}|c|c|c|c|c|}
\hline 
\multicolumn{6}{|c|}{$F(\bx)=f(\bx)+\psi(\bx)$ where $f(\bx)=\frac{1}{n}\sum_{i=1}^{n}f_{i}(\bx)$
and $f_{i}(\bx)$ and $\psi(\bx)$ are convex}\tabularnewline
\hline 
Settings & References & Rate & Type & Output & $\psi(\bx)$\tabularnewline
\hline 
 & \cite{mishchenko2020random} & $\O\left(\frac{L^{1/3}\sigma_{\mathrm{any}}^{2/3}D^{4/3}}{K^{2/3}}\right)$ & IG & Avg & \ding{55}\tabularnewline
 & \textbf{Ours} (Theorem \ref{thm:smooth-cvx-main}) & $\widetilde{\O}\left(\frac{L^{1/3}\sigma_{\mathrm{any}}^{2/3}D^{4/3}}{K^{2/3}}\right)$ & Any & Last & \ding{51}\tabularnewline
\cline{2-6} \cline{3-6} \cline{4-6} \cline{5-6} \cline{6-6} 
$L$-smooth $f_{i}(\bx)$ & \cite{mishchenko2020random,nguyen2021unified} & $\O\left(\frac{L^{1/3}\sigma_{\mathrm{any}}^{2/3}D^{4/3}}{n^{1/3}K^{2/3}}\right)$ & RR/SO & Avg & \ding{55}\tabularnewline
 & \textbf{Ours} (Theorem \ref{thm:smooth-cvx-main}) & $\widetilde{\O}\left(\frac{L^{1/3}\sigma_{\mathrm{rand}}^{2/3}D^{4/3}}{n^{1/3}K^{2/3}}\right)$\footnote{Note that when $\psi(\bx)=0$, there is $\rsigma=\asigma+n\left\Vert \nabla f(\bx_{*})\right\Vert ^{2}=\asigma$
due to $\nabla f(\bx_{*})=\mathbf{0}$.}\saveFN\sfnapp\ & Random & Last & \ding{51}\tabularnewline
 & \cite{pmlr-v202-cha23a} & $\Omega\left(\frac{L^{1/3}\sigma_{\mathrm{any}}^{2/3}D^{4/3}}{n^{1/3}K^{2/3}}\right)$ & RR & Last & \ding{55}\tabularnewline
\hline 
 & \cite{mishchenko2020random,nguyen2021unified} & $\widetilde{\O}\left(\frac{L^{2}\asigma}{\mu^{3}K^{2}}\right)$ & IG & Last & \ding{55}\tabularnewline
 & \textbf{Ours} (Theorem \ref{thm:smooth-str-main}) & $\widetilde{\O}\left(\frac{L\asigma}{\mu^{2}K^{2}}\right)$ & Any & Last & \ding{51}\tabularnewline
$L$-smooth $f_{i}(\bx)$, & \cite{safran2020good} & $\Omega\left(\frac{G^{2}}{\mu K^{2}}\right)$\footnote{This lower bound is established under $L=\mu$ and additionally requires
$\left\Vert \nabla f_{i}(\bx)\right\Vert \leq G$. Under the same
condition, our above upper bound $\widetilde{\O}\left(\frac{L\asigma}{\mu^{2}K^{2}}\right)$
will be $\widetilde{\O}\left(\frac{G^{2}}{\mu K^{2}}\right)$ and
hence almost matches this lower bound.} & IG & Last & \ding{55}\tabularnewline
\cline{2-6} \cline{3-6} \cline{4-6} \cline{5-6} \cline{6-6} 
$\mu$-strongly convex $f(\bx)$ & \cite{pmlr-v202-cha23a} & $\widetilde{\O}\left(\frac{L\asigma}{\mu^{2}nK^{2}}\right)$ & RR & Tail & \ding{55}\tabularnewline
 & \textbf{Ours} (Theorem \ref{thm:smooth-str-main}) & $\widetilde{\O}\left(\frac{L\rsigma}{\mu^{2}nK^{2}}\right)$\useFN\sfnapp\ & Random & Last & \ding{51}\tabularnewline
 & \cite{safran2021random,pmlr-v202-cha23a} & $\Omega\left(\frac{L\asigma}{\mu^{2}nK^{2}}\right)$ & RR/SO & Last & \ding{55}\tabularnewline
\hline 
 & \textbf{Ours} (Theorem \ref{thm:smooth-str-full}) & $\widetilde{\O}\left(\frac{L\asigma}{\mu^{2}K^{2}}\right)$ & Any & Last & \ding{51}\tabularnewline
\cline{2-6} \cline{3-6} \cline{4-6} \cline{5-6} \cline{6-6} 
$L$-smooth $f_{i}(\bx)$,\\
$\mu$-strongly convex $\psi(\bx)$ & \cite{mishchenko2022proximal} & $\widetilde{\O}\left(\frac{L^{2}\rsigma}{\mu^{3}nK^{2}}\right)$\footnote{The original rate $\widetilde{\O}\left(\frac{L\rsigma}{\mu^{3}nK^{2}}\right)$
is only proved for $\left\Vert \bx_{K+1}-\bx_{*}\right\Vert ^{2}$
and hence cannot measure the function value gap when a general $\psi(\bx)$
exists as explained in Section \ref{sec:intro}. However, for the
convenience of comparison, we still consider the traditional conversion
$F(\bx_{K+1})-F(\bx_{*})=\O(L\left\Vert \bx_{K+1}-\bx_{*}\right\Vert ^{2})$
here, though it does not hold in general.}\saveFN\sfnappa\  & RR/SO & Last\useFN\sfnappa\  & \ding{51}\tabularnewline
 & \textbf{Ours} (Theorem \ref{thm:smooth-str-full}) & $\widetilde{\O}\left(\frac{L\rsigma}{\mu^{2}nK^{2}}\right)$ & Random & Last & \ding{51}\tabularnewline
\hline 
\end{tabular}
\end{table}
\end{minipage}

\section{Important Notations}

We summarize important notations used in the analysis as follows for
the sake of readability.
\begin{itemize}
\item Objective function: $F(x)\triangleq f(\bx)+\psi(\bx)$ where $f(\bx)\triangleq\frac{1}{n}\sum_{i=1}^{n}f_{i}(\bx)$.
\item Number of epochs: $K\geq2$.
\item Initial distance: $D\triangleq\left\Vert \bx_{*}-\bx_{1}\right\Vert $.
\item Optimal function value: $F_{*}\triangleq F(\bx_{*})$.
\item The averaged smooth parameter: $\bL\triangleq\frac{1}{n}\sum_{i=1}^{n}L_{i}$.
\item The largest smooth parameter: $L^{*}\triangleq\max_{i\in\left[n\right]}L_{i}$.
\item Two quantities measuring the uncertainty: $\asigma\triangleq\frac{1}{n}\sum_{i=1}^{n}\left\Vert \nabla f_{i}(\bx_{*})\right\Vert ^{2}$
and $\rsigma\triangleq\sigma_{\mathrm{any}}^{2}+n\left\Vert \nabla f(\bx_{*})\right\Vert ^{2}$.
\item The averaged Lipschitz parameter: $\bG\triangleq\frac{1}{n}\sum_{i=1}^{n}G_{i}$.
\item The Bregman divergence induced by $f_{\sigma_{k}^{i}}(\bx)$: $B_{\sigma_{k}^{i}}(\cdot,\cdot)\triangleq B_{f_{\sigma_{k}^{i}}}(\cdot,\cdot)$.
\end{itemize}

\section{Full Theorems and Proofs\label{sec:full-thms}}

In this section, we provide full statements of all theorems with proofs.
Lemmas used in the proof can be found in Sections \ref{sec:lemmas}
and \ref{sec:tech}.

\subsection{Smooth Functions}

All results presented in this subsection are for smooth components
$f_{i}(\bx)$.

First, we consider the general convex case, i.e., $\mu_{f}=\mu_{\psi}=0$.
The key tools in the proof are Lemmas \ref{lem:smooth-core} and \ref{lem:ineq}.
\begin{thm}
\label{thm:smooth-cvx-full}(Full version of Theorem \ref{thm:smooth-cvx-main})
Under Assumptions \ref{assu:minimizer}, \ref{assu:basic} (with $\mu_{\psi}=0$)
and \ref{assu:smooth}, taking the stepsize $\eta_{k}=\eta\leq\frac{1}{4n\sqrt{2\bL L^{*}(1+\log K)}},\forall k\in\left[K\right]$:
\begin{itemize}
\item Regardless of how the permutation $\sigma_{k}$ is generated in every
epoch, Algorithm \ref{alg:shuffling-SGD} guarantees
\[
F(\bx_{K+1})-F_{*}\leq\O\left(\frac{D^{2}}{n\eta K}+(n\eta)^{2}\bL\asigma(1+\log K)\right).
\]
Setting $\eta=\frac{1}{4n\sqrt{2\bL L^{*}(1+\log K)}}\land\frac{D^{2/3}}{n\sqrt[3]{\bL\asigma K(1+\log K)}}$
to get
\[
F(\bx_{K+1})-F_{*}\leq\O\left(\frac{D^{2}\sqrt{\bL L^{*}(1+\log K)}}{K}+\frac{D^{4/3}\sqrt[3]{\bL\asigma(1+\log K)}}{K^{2/3}}\right).
\]
\item Suppose the permutation $\sigma_{k}$ is uniformly sampled without
replacement, Algorithm \ref{alg:shuffling-SGD} guarantees
\[
\E\left[F(\bx_{K+1})-F_{*}\right]\leq\O\left(\frac{D^{2}}{n\eta K}+n\eta^{2}\bL\rsigma(1+\log K)\right).
\]
Setting $\eta=\frac{1}{4n\sqrt{2\bL L^{*}(1+\log K)}}\land\frac{D^{2/3}}{\sqrt[3]{n^{2}\bL\rsigma K(1+\log K)}}$
to get
\[
\E\left[F(\bx_{K+1})-F_{*}\right]\leq\O\left(\frac{D^{2}\sqrt{\bL L^{*}(1+\log K)}}{K}+\frac{D^{4/3}\sqrt[3]{\bL\rsigma(1+\log K)}}{n^{1/3}K^{2/3}}\right).
\]
\end{itemize}
\end{thm}

\begin{proof}
Note that $\eta_{k}=\eta\leq\frac{1}{4n\sqrt{2\bL L^{*}(1+\log K)}}\leq\frac{1}{2n\sqrt{\bL L^{*}}}$
satisfies the requirement of Lemma \ref{lem:smooth-core}. Hence,
for any $k\in\left[K\right]$
\begin{align}
F(\bx_{k+1})-F(\bx_{*}) & \leq\frac{\left\Vert \bx_{*}-\bx_{1}\right\Vert ^{2}}{2n\sum_{\ell=1}^{k}\eta_{\ell}}+\sum_{\ell=1}^{k}\frac{4\eta_{\ell}^{3}R_{\ell}}{\sum_{s=\ell}^{k}\eta_{s}}+\sum_{\ell=2}^{k}\frac{8n^{2}\bL^{2}\eta_{\ell-1}\left(\sum_{s=\ell}^{k}\eta_{s}^{3}\right)}{\left(\sum_{s=\ell}^{k}\eta_{s}\right)\left(\sum_{s=\ell-1}^{k}\eta_{s}\right)}B_{f}(\bx_{\ell},\bx_{*})\nonumber \\
 & =\frac{\left\Vert \bx_{*}-\bx_{1}\right\Vert ^{2}}{2n\eta k}+\sum_{\ell=1}^{k}\frac{4\eta^{2}R_{\ell}}{k-\ell+1}+8(n\eta\bL)^{2}\sum_{\ell=2}^{k}\frac{B_{f}(\bx_{\ell},\bx_{*})}{k-\ell+2},\label{eq:smooth-cvx-full-1}
\end{align}
where $R_{\ell}=\sum_{i=2}^{n}\frac{L_{\sigma_{\ell}^{i}}}{n}\left\Vert \sum_{j=1}^{i-1}\nabla f_{\sigma_{\ell}^{j}}(\bx_{*})\right\Vert ^{2}$.
The definition of $\bx_{*}\in\argmin_{\bx\in\R^{d}}F(\bx)$ implies
$\exists\nabla\psi(\bx_{*})\in\partial\psi(\bx_{*})$ such that $\nabla f(\bx_{*})+\nabla\psi(\bx_{*})=\mathbf{0}$.
Thus, for any $k\in\left[K\right]$
\begin{align*}
F(\bx_{k+1})-F(\bx_{*}) & =F(\bx_{k+1})-F(\bx_{*})-\langle\nabla f(\bx_{*})+\nabla\psi(\bx_{*}),\bx_{k+1}-\bx_{*}\rangle\\
 & =B_{f}(\bx_{k+1},\bx_{*})+B_{\psi}(\bx_{k+1},\bx_{*})\geq B_{f}(\bx_{k+1},\bx_{*}),
\end{align*}
which implies
\begin{equation}
B_{f}(\bx_{k+1},\bx_{*})\leq\frac{\left\Vert \bx_{*}-\bx_{1}\right\Vert ^{2}}{2n\eta k}+\sum_{\ell=1}^{k}\frac{4\eta^{2}R_{\ell}}{k-\ell+1}+8(n\eta\bL)^{2}\sum_{\ell=2}^{k}\frac{B_{f}(\bx_{\ell},\bx_{*})}{k-\ell+2},\forall k\in\left[K\right].\label{eq:smooth-cvx-full-B}
\end{equation}
Besides, by taking $k=K$ in (\ref{eq:smooth-cvx-full-1}), we obtain
\begin{equation}
F(\bx_{K+1})-F(\bx_{*})\leq\frac{\left\Vert \bx_{*}-\bx_{1}\right\Vert ^{2}}{2n\eta K}+\sum_{\ell=1}^{K}\frac{4\eta^{2}R_{\ell}}{K-\ell+1}+8(n\eta\bL)^{2}\sum_{\ell=2}^{K}\frac{B_{f}(\bx_{\ell},\bx_{*})}{K-\ell+2}.\label{eq:smooth-cvx-full-F}
\end{equation}

First, by Lemma \ref{lem:var}, for any permutation $\sigma_{\ell}$
used in the $\ell$-th epoch we have
\begin{equation}
R_{\ell}\leq n^{2}\bL\asigma,\forall\ell\in\left[K\right].\label{eq:smooth-cvx-full-var-a}
\end{equation}
Hence, for any $k\in\left[K\right]$
\begin{align*}
B_{f}(\bx_{k+1},\bx_{*}) & \overset{\eqref{eq:smooth-cvx-full-B},\eqref{eq:smooth-cvx-full-var-a}}{\leq}\frac{\left\Vert \bx_{*}-\bx_{1}\right\Vert ^{2}}{2n\eta k}+\sum_{\ell=1}^{k}\frac{4\eta^{2}n^{2}\bL\asigma}{k-\ell+1}+8(n\eta\bL)^{2}\sum_{\ell=2}^{k}\frac{B_{f}(\bx_{\ell},\bx_{*})}{k-\ell+2}\\
 & \leq\frac{\left\Vert \bx_{*}-\bx_{1}\right\Vert ^{2}}{2n\eta k}+4(n\eta)^{2}\bL\asigma(1+\log k)+8(n\eta\bL)^{2}\sum_{\ell=2}^{k}\frac{B_{f}(\bx_{\ell},\bx_{*})}{k-\ell+2}.
\end{align*}
In addition,
\begin{align*}
F(\bx_{K+1})-F(\bx_{*}) & \overset{\eqref{eq:smooth-cvx-full-F},\eqref{eq:smooth-cvx-full-var-a}}{\leq}\frac{\left\Vert \bx_{*}-\bx_{1}\right\Vert ^{2}}{2n\eta K}+\sum_{\ell=1}^{K}\frac{4\eta^{2}n^{2}\bL\asigma}{K-\ell+1}+8(n\eta\bL)^{2}\sum_{\ell=2}^{K}\frac{B_{f}(\bx_{\ell},\bx_{*})}{K-\ell+2}\\
 & \leq\frac{\left\Vert \bx_{*}-\bx_{1}\right\Vert ^{2}}{2n\eta K}+4(n\eta)^{2}\bL\asigma(1+\log K)+8(n\eta\bL)^{2}\sum_{\ell=2}^{K}\frac{B_{f}(\bx_{\ell},\bx_{*})}{K-\ell+2}.
\end{align*}
We apply Lemma \ref{lem:ineq} with $d_{k+1}=\begin{cases}
B_{f}(\bx_{k+1},\bx_{*}) & k\in\left[K-1\right]\\
F(\bx_{K+1})-F(\bx_{*}) & k=K
\end{cases},a=\frac{\left\Vert \bx_{*}-\bx_{1}\right\Vert ^{2}}{2n\eta},b=4(n\eta)^{2}\bL\asigma,c=8(n\eta\bL)^{2}$ to obtain
\[
B_{f}(\bx_{k+1},\bx_{*})\leq\left(\frac{\left\Vert \bx_{*}-\bx_{1}\right\Vert ^{2}}{2n\eta k}+4(n\eta)^{2}\bL\asigma(1+\log k)\right)\sum_{i=0}^{k-1}\left(16(n\eta\bL)^{2}(1+\log k)\right)^{i},\forall k\in\left[K-1\right],
\]
and
\[
F(\bx_{K+1})-F(\bx_{*})\leq\left(\frac{\left\Vert \bx_{*}-\bx_{1}\right\Vert ^{2}}{2n\eta K}+4(n\eta)^{2}\bL\asigma(1+\log K)\right)\sum_{i=0}^{K-1}\left(16(n\eta\bL)^{2}(1+\log K)\right)^{i}.
\]
By $\eta_{k}=\eta\leq\frac{1}{4n\sqrt{2\bL L^{*}(1+\log K)}}$, we
have for any $k\in\left[K\right]$,
\[
\sum_{i=0}^{k-1}\left(16(n\eta\bL)^{2}(1+\log k)\right)^{i}\leq\sum_{i=0}^{k-1}\left(\frac{\bL(1+\log k)}{2L^{*}(1+\log K)}\right)^{i}\leq\sum_{i=0}^{\infty}\frac{1}{2^{i}}=2.
\]
So there is
\begin{align*}
B_{f}(\bx_{k+1},\bx_{*}) & \leq\frac{\left\Vert \bx_{*}-\bx_{1}\right\Vert ^{2}}{n\eta k}+8(n\eta)^{2}\bL\asigma(1+\log k),\forall k\in\left[K-1\right],
\end{align*}
and
\begin{align*}
F(\bx_{K+1})-F(\bx_{*}) & \leq\frac{\left\Vert \bx_{*}-\bx_{1}\right\Vert ^{2}}{n\eta K}+8(n\eta)^{2}\bL\asigma(1+\log K)\\
 & =\O\left(\frac{\left\Vert \bx_{*}-\bx_{1}\right\Vert ^{2}}{n\eta K}+(n\eta)^{2}\bL\asigma(1+\log K)\right).
\end{align*}
Finally, taking $\eta=\frac{1}{4n\sqrt{2\bL L^{*}(1+\log K)}}\land\frac{\left\Vert \bx_{*}-\bx_{1}\right\Vert ^{2/3}}{n\sqrt[3]{\bL\asigma K(1+\log K)}}$
to obtain
\[
F(\bx_{K+1})-F(\bx_{*})\leq\O\left(\frac{\left\Vert \bx_{*}-\bx_{1}\right\Vert ^{2}\sqrt{\bL L^{*}(1+\log K)}}{K}+\frac{\left\Vert \bx_{*}-\bx_{1}\right\Vert ^{4/3}\sqrt[3]{\bL\asigma(1+\log K)}}{K^{2/3}}\right).
\]

For the case of randomly generated permutations, by Lemma \ref{lem:var},
we have
\begin{equation}
\E\left[R_{\ell}\right]\leq\frac{2}{3}n\bL\rsigma,\forall\ell\in\left[K\right].\label{eq:smooth-cvx-full-var-r}
\end{equation}
Hence, for any $k\in\left[K\right]$,
\begin{align*}
\E\left[B_{f}(\bx_{k+1},\bx_{*})\right] & \overset{\eqref{eq:smooth-cvx-full-B},\eqref{eq:smooth-cvx-full-var-r}}{\leq}\frac{\left\Vert \bx_{*}-\bx_{1}\right\Vert ^{2}}{2n\eta k}+\sum_{\ell=1}^{k}\frac{\frac{8}{3}\eta^{2}n\bL\rsigma}{k-\ell+1}+8(n\eta\bL)^{2}\sum_{\ell=2}^{k}\frac{\E\left[B_{f}(\bx_{\ell},\bx_{*})\right]}{k-\ell+2}\\
 & \leq\frac{\left\Vert \bx_{*}-\bx_{1}\right\Vert ^{2}}{2n\eta k}+\frac{8}{3}n\eta^{2}\bL\rsigma(1+\log k)+8(n\eta\bL)^{2}\sum_{\ell=2}^{k}\frac{\E\left[B_{f}(\bx_{\ell},\bx_{*})\right]}{k-\ell+2}.
\end{align*}
Moreover,
\begin{align*}
\E\left[F(\bx_{K+1})-F(\bx_{*})\right] & \overset{\eqref{eq:smooth-cvx-full-F},\eqref{eq:smooth-cvx-full-var-r}}{\leq}\frac{\left\Vert \bx_{*}-\bx_{1}\right\Vert ^{2}}{2n\eta K}+\sum_{\ell=1}^{K}\frac{\frac{8}{3}\eta^{2}n\bL\rsigma}{K-\ell+1}+8(n\eta\bL)^{2}\sum_{\ell=2}^{K}\frac{\E\left[B_{f}(\bx_{\ell},\bx_{*})\right]}{K-\ell+2}\\
 & \leq\frac{\left\Vert \bx_{*}-\bx_{1}\right\Vert ^{2}}{2n\eta K}+\frac{8}{3}n\eta^{2}\bL\rsigma(1+\log K)+8(n\eta\bL)^{2}\sum_{\ell=2}^{K}\frac{B_{f}(\bx_{\ell},\bx_{*})}{K-\ell+2}.
\end{align*}
We apply Lemma \ref{lem:ineq} with $d_{k+1}=\begin{cases}
\E\left[B_{f}(\bx_{k+1},\bx_{*})\right] & k\in\left[K-1\right]\\
\E\left[F(\bx_{K+1})-F(\bx_{*})\right] & k=K
\end{cases},a=\frac{\left\Vert \bx_{*}-\bx_{1}\right\Vert ^{2}}{2n\eta},b=\frac{8}{3}n\eta^{2}\bL\rsigma,c=8(n\eta\bL)^{2}$ and then follow the similar steps used before to obtain
\begin{align*}
\E\left[F(\bx_{K+1})-F(\bx_{*})\right] & \leq\frac{\left\Vert \bx_{*}-\bx_{1}\right\Vert ^{2}}{n\eta K}+\frac{16}{3}n\eta^{2}\bL\rsigma(1+\log K)\\
 & =\O\left(\frac{\left\Vert \bx_{*}-\bx_{1}\right\Vert ^{2}}{n\eta K}+n\eta^{2}\bL\rsigma(1+\log K)\right).
\end{align*}
Finally, taking $\eta=\frac{1}{4n\sqrt{2\bL L^{*}(1+\log K)}}\land\frac{\left\Vert \bx_{*}-\bx_{1}\right\Vert ^{2/3}}{\sqrt[3]{n^{2}\bL\rsigma K(1+\log K)}}$
to obtain
\[
F(\bx_{K+1})-F(\bx_{*})\leq\O\left(\frac{\left\Vert \bx_{*}-\bx_{1}\right\Vert ^{2}\sqrt{\bL L^{*}(1+\log K)}}{K}+\frac{\left\Vert \bx_{*}-\bx_{1}\right\Vert ^{4/3}\sqrt[3]{\bL\rsigma(1+\log K)}}{n^{1/3}K^{2/3}}\right).
\]
\end{proof}

Next, we provide the full theorem when at least one of $f(\bx)$ and
$\psi(\bx)$ is strongly convex. The key step is to bound $F(\bx_{K+1})-F(\bx_{*})$
by $\left\Vert \bx_{*}-\bx_{\left\lceil \frac{K}{2}\right\rceil +1}\right\Vert ^{2}$
using Theorem \ref{thm:smooth-cvx-full} and bound $\left\Vert \bx_{*}-\bx_{\left\lceil \frac{K}{2}\right\rceil +1}\right\Vert ^{2}$
by $\left\Vert \bx_{*}-\bx_{1}\right\Vert ^{2}$ using Lemma \ref{lem:smooth-distance}.
\begin{thm}
\label{thm:smooth-str-full}(Full version of Theorem \ref{thm:smooth-str-main})
Under Assumptions \ref{assu:minimizer}, \ref{assu:basic}, \ref{assu:smooth}
and \ref{assu:str} with $\mu_{F}\triangleq\mu_{f}+2\mu_{\psi}>0$,
taking the stepsize $\eta_{k}=\eta\leq\frac{1}{4n\sqrt{2\bL L^{*}(1+\log K)}\lor\mu_{\psi}},\forall k\in\left[K\right]$:
\begin{itemize}
\item Regardless of how the permutation $\sigma_{k}$ is generated in every
epoch, Algorithm \ref{alg:shuffling-SGD} guarantees
\[
F(\bx_{K+1})-F_{*}\leq\mathcal{O}\left(\frac{D^{2}}{n\eta K}e^{-n\eta\mu_{F}K}+\frac{n\eta\bL\asigma}{\mu_{F}K}+(n\eta)^{2}\bL\asigma(1+\log K)\right).
\]
Setting $\eta=\frac{1}{4n\sqrt{2\bL L^{*}(1+\log K)}\lor\mu_{\psi}}\land\frac{u}{n\mu_{F}K}$
where $u=1\lor\log\frac{\mu_{F}^{3}D^{2}K^{2}}{\bL\asigma(1+\log K)}$
to get
\[
F(\bx_{K+1})-F_{*}\leq\mathcal{O}\left(\frac{D^{2}\left(\sqrt{\bL L^{*}(1+\log K)}\lor\mu_{\psi}\right)}{Ke^{\frac{\mu_{F}K}{\sqrt{\bL L^{*}(1+\log K)}\lor\mu_{\psi}}}}+\frac{(1+u^{2})\bL\asigma(1+\log K)}{\mu_{F}^{2}K^{2}}\right).
\]
\item Suppose the permutation $\sigma_{k}$ is uniformly sampled without
replacement, Algorithm \ref{alg:shuffling-SGD} guarantees
\[
\E\left[F(\bx_{K+1})-F_{*}\right]\leq\O\left(\frac{D^{2}}{n\eta K}e^{-n\eta\mu_{F}K}+\frac{\eta\bL\rsigma}{\mu_{F}K}+n\eta^{2}\bL\rsigma(1+\log K)\right).
\]
Setting $\eta=\frac{1}{4n\sqrt{2\bL L^{*}(1+\log K)}\lor\mu_{\psi}}\land\frac{u}{n\mu_{F}K}$
where $u=1\lor\log\frac{n\mu_{F}^{3}D^{2}K^{2}}{\bL\rsigma(1+\log K)}$
to get
\[
\E\left[F(\bx_{K+1})-F_{*}\right]\leq\mathcal{O}\left(\frac{D^{2}\left(\sqrt{\bL L^{*}(1+\log K)}\lor\mu_{\psi}\right)}{Ke^{\frac{\mu_{F}K}{\sqrt{\bL L^{*}(1+\log K)}\lor\mu_{\psi}}}}+\frac{(1+u^{2})\bL\rsigma(1+\log K)}{\mu_{F}^{2}nK^{2}}\right).
\]
\end{itemize}
\end{thm}

\begin{proof}
Suppose we run Algorithm \ref{alg:shuffling-SGD} by $K$ epochs with
the stepsize $\eta_{k}=\eta,\forall k\in\left[K\right]$. Then one
can recognize $\bx_{K+1}$ as obtained by running Algorithm \ref{alg:shuffling-SGD}
with the initial point $\bx_{\left\lceil \frac{K}{2}\right\rceil +1}$
by $K-\left\lceil \frac{K}{2}\right\rceil =\left\lfloor \frac{K}{2}\right\rfloor $
epochs. Note that $\eta\leq\frac{1}{4n\sqrt{2\bL L^{*}(1+\log K)}\lor\mu_{\psi}}\leq\frac{1}{4n\sqrt{2\bL L^{*}(1+\log\left\lfloor \frac{K}{2}\right\rfloor )}}$,
then we can apply Theorem \ref{thm:smooth-cvx-full} to obtain that
for any permutation $\sigma_{k}$
\begin{align}
F(\bx_{K+1})-F(\bx_{*}) & \leq\O\left(\frac{\left\Vert \bx_{*}-\bx_{\left\lceil \frac{K}{2}\right\rceil +1}\right\Vert ^{2}}{n\eta\left\lfloor \frac{K}{2}\right\rfloor }+(n\eta)^{2}\bL\asigma\left(1+\log\left\lfloor \frac{K}{2}\right\rfloor \right)\right)\nonumber \\
 & =\O\left(\frac{\left\Vert \bx_{*}-\bx_{\left\lceil \frac{K}{2}\right\rceil +1}\right\Vert ^{2}}{n\eta K}+(n\eta)^{2}\bL\asigma(1+\log K)\right).\label{eq:smooth-str-full-1}
\end{align}
If the permutation $\sigma_{k}$ is uniformly sampled without replacement,
there is
\begin{align}
\E\left[F(\bx_{K+1})-F(\bx_{*})\right] & \leq\O\left(\frac{\E\left[\left\Vert \bx_{*}-\bx_{\left\lceil \frac{K}{2}\right\rceil +1}\right\Vert ^{2}\right]}{n\eta\left\lfloor \frac{K}{2}\right\rfloor }+n\eta^{2}\bL\rsigma\left(1+\log\left\lfloor \frac{K}{2}\right\rfloor \right)\right)\nonumber \\
 & =\O\left(\frac{\E\left[\left\Vert \bx_{*}-\bx_{\left\lceil \frac{K}{2}\right\rceil +1}\right\Vert ^{2}\right]}{n\eta K}+n\eta^{2}\bL\rsigma(1+\log K)\right).\label{eq:smooth-str-full-2}
\end{align}

Now we use Lemma \ref{lem:smooth-distance} for $k=\left\lceil \frac{K}{2}\right\rceil $
to obtain
\begin{align}
\left\Vert \bx_{*}-\bx_{\left\lceil \frac{K}{2}\right\rceil +1}\right\Vert ^{2} & \leq\frac{\left\Vert \bx_{*}-\bx_{1}\right\Vert ^{2}}{\prod_{s=1}^{\left\lceil \frac{K}{2}\right\rceil }(1+n\eta_{s}(\mu_{f}+2\mu_{\psi}))}+\sum_{\ell=1}^{\left\lceil \frac{K}{2}\right\rceil }\frac{8n\eta_{\ell}^{3}R_{\ell}}{\prod_{s=\ell}^{\left\lceil \frac{K}{2}\right\rceil }(1+n\eta_{s}(\mu_{f}+2\mu_{\psi}))}\nonumber \\
 & =\frac{\left\Vert \bx_{*}-\bx_{1}\right\Vert ^{2}}{(1+n\eta\mu_{F})^{\left\lceil \frac{K}{2}\right\rceil }}+8n\eta^{3}\sum_{\ell=1}^{\left\lceil \frac{K}{2}\right\rceil }\frac{R_{\ell}}{(1+n\eta\mu_{F})^{\left\lceil \frac{K}{2}\right\rceil -\ell+1}}\nonumber \\
 & \overset{(a)}{\leq}\left\Vert \bx_{*}-\bx_{1}\right\Vert ^{2}e^{-\frac{n\eta\mu_{F}\left\lceil \frac{K}{2}\right\rceil }{1+n\eta\mu_{F}}}+8n\eta^{3}\sum_{\ell=1}^{\left\lceil \frac{K}{2}\right\rceil }\frac{R_{\ell}}{(1+n\eta\mu_{F})^{\left\lceil \frac{K}{2}\right\rceil -\ell+1}}\nonumber \\
 & \overset{(b)}{\leq}\left\Vert \bx_{*}-\bx_{1}\right\Vert ^{2}e^{-\frac{5n\eta\mu_{F}K}{32}}+8n\eta^{3}\sum_{\ell=1}^{\left\lceil \frac{K}{2}\right\rceil }\frac{R_{\ell}}{(1+n\eta\mu_{F})^{\left\lceil \frac{K}{2}\right\rceil -\ell+1}},\label{eq:smooth-str-full-distance}
\end{align}
where $(a)$ is by $\frac{1}{1+x}\leq\exp\left(-\frac{x}{1+x}\right)$
and $(b)$ is due to $\left\lceil \frac{K}{2}\right\rceil \geq\frac{K}{2}$
and $n\eta\mu_{F}\leq\frac{\mu_{f}+2\mu_{\psi}}{4\sqrt{2\bL L^{*}(1+\log K)}\lor\mu_{\psi}}\leq\frac{1}{4\sqrt{2}}+2\leq\frac{11}{5}$.

By Lemma \ref{lem:var}, there is $R_{\ell}\leq n^{2}\bL\asigma$
for any permutation $\sigma_{\ell}$, which implies
\begin{align}
\left\Vert \bx_{*}-\bx_{\left\lceil \frac{K}{2}\right\rceil +1}\right\Vert ^{2} & \overset{\eqref{eq:smooth-str-full-distance}}{\leq}\left\Vert \bx_{*}-\bx_{1}\right\Vert ^{2}e^{-\frac{5n\eta\mu_{F}K}{32}}+8(n\eta)^{3}\bL\asigma\sum_{\ell=1}^{\left\lceil \frac{K}{2}\right\rceil }\frac{1}{(1+n\eta\mu_{F})^{\left\lceil \frac{K}{2}\right\rceil -\ell+1}}\nonumber \\
 & \leq\left\Vert \bx_{*}-\bx_{1}\right\Vert ^{2}e^{-\frac{5n\eta\mu_{F}K}{32}}+\frac{8(n\eta)^{2}\bL\asigma}{\mu_{F}}\nonumber \\
 & =\mathcal{O}\left(\left\Vert \bx_{*}-\bx_{1}\right\Vert ^{2}e^{-n\eta\mu_{F}K}+\frac{(n\eta)^{2}\bL\asigma}{\mu_{F}}\right).\label{eq:smooth-str-full-3}
\end{align}
Additionally, if $\sigma_{\ell}$ is uniformly sampled without replacement,
then $\E\left[R_{\ell}\right]\leq\frac{2}{3}n\bL\rsigma$. Therefore
\begin{align}
\E\left[\left\Vert \bx_{*}-\bx_{\left\lceil \frac{K}{2}\right\rceil +1}\right\Vert ^{2}\right] & \overset{\eqref{eq:smooth-str-full-distance}}{\leq}\left\Vert \bx_{*}-\bx_{1}\right\Vert ^{2}e^{-\frac{5n\eta\mu_{F}K}{32}}+\frac{16}{3}n^{2}\eta^{3}\bL\rsigma\sum_{\ell=1}^{\left\lceil \frac{K}{2}\right\rceil }\frac{1}{(1+n\eta\mu_{F})^{\left\lceil \frac{K}{2}\right\rceil -\ell+1}}\nonumber \\
 & \leq\left\Vert \bx_{*}-\bx_{1}\right\Vert ^{2}e^{-\frac{5n\eta\mu_{F}K}{32}}+\frac{16n\eta^{2}\bL\rsigma}{3\mu_{F}}\nonumber \\
 & =\mathcal{O}\left(\left\Vert \bx_{*}-\bx_{1}\right\Vert ^{2}e^{-n\eta\mu_{F}K}+\frac{n\eta^{2}\bL\rsigma}{\mu_{F}}\right).\label{eq:smooth-str-full-4}
\end{align}

Combining (\ref{eq:smooth-str-full-1}) and (\ref{eq:smooth-str-full-3}),
we obtain for any permutation
\[
F(\bx_{K+1})-F(\bx_{*})\leq\mathcal{O}\left(\frac{\left\Vert \bx_{*}-\bx_{1}\right\Vert ^{2}}{n\eta K}e^{-n\eta\mu_{F}K}+\frac{n\eta\bL\asigma}{\mu_{F}K}+(n\eta)^{2}\bL\asigma(1+\log K)\right).
\]
Taking $\eta=\frac{1}{4n\sqrt{2\bL L^{*}(1+\log K)}\lor\mu_{\psi}}\land\frac{u}{n\mu_{F}K}$
where $u=1\lor\log\frac{\mu_{F}^{3}\left\Vert \bx_{*}-\bx_{1}\right\Vert ^{2}K^{2}}{\bL\asigma(1+\log K)}$
to obtain
\[
F(\bx_{K+1})-F(\bx_{*})\leq\mathcal{O}\left(\frac{\left\Vert \bx_{*}-\bx_{1}\right\Vert ^{2}\left(\sqrt{\bL L^{*}(1+\log K)}\lor\mu_{\psi}\right)}{Ke^{\frac{\mu_{F}K}{\sqrt{\bL L^{*}(1+\log K)}\lor\mu_{\psi}}}}+\frac{(1+u^{2})\bL\asigma(1+\log K)}{\mu_{F}^{2}K^{2}}\right).
\]

Combining (\ref{eq:smooth-str-full-2}) and (\ref{eq:smooth-str-full-4}),
we obtain for uniform sampling without replacement
\[
\E\left[F(\bx_{K+1})-F(\bx_{*})\right]\leq\O\left(\frac{\left\Vert \bx_{*}-\bx_{1}\right\Vert ^{2}}{n\eta K}e^{-n\eta\mu_{F}K}+\frac{\eta\bL\rsigma}{\mu_{F}K}+n\eta^{2}\bL\rsigma(1+\log K)\right).
\]
Taking $\eta=\frac{1}{4n\sqrt{2\bL L^{*}(1+\log K)}\lor\mu_{\psi}}\land\frac{u}{n\mu_{F}K}$
where $u=1\lor\log\frac{n\mu_{F}^{3}\left\Vert \bx_{*}-\bx_{1}\right\Vert ^{2}K^{2}}{\bL\rsigma(1+\log K)}$
to get
\[
\E\left[F(\bx_{K+1})-F(\bx_{*})\right]\leq\mathcal{O}\left(\frac{\left\Vert \bx_{*}-\bx_{1}\right\Vert ^{2}\left(\sqrt{\bL L^{*}(1+\log K)}\lor\mu_{\psi}\right)}{Ke^{\frac{\mu_{F}K}{\sqrt{\bL L^{*}(1+\log K)}\lor\mu_{\psi}}}}+\frac{(1+u^{2})\bL\rsigma(1+\log K)}{\mu_{F}^{2}nK^{2}}\right).
\]
\end{proof}

\subsection{Lipschitz Functions}

We focus on Lipschitz components $f_{i}(\bx)$ in this subsection.

First, we analyze the case of $\mu_{\psi}=0$ in Theorem \ref{thm:lip-cvx-full}
with three different stepsizes. It is worth noting that the last stepsize
schedule, which is inspired by \cite{zamani2023exact}, can remove
the extra $\O(\log K)$ term.
\begin{thm}
\label{thm:lip-cvx-full}(Full version of Theorem \ref{thm:lip-cvx-main})
Under Assumptions \ref{assu:minimizer}, \ref{assu:basic} (with $\mu_{\psi}=0$)
and \ref{assu:lip}, regardless of how the permutation $\sigma_{k}$
is generated in every epoch:
\begin{itemize}
\item Taking the stepsize $\eta_{k}=\frac{\eta}{\sqrt{k}},\forall k\in\left[K\right]$,
Algorithm \ref{alg:shuffling-SGD} guarantees
\[
F(\bx_{K+1})-F_{*}\leq\O\left(\left(\frac{D^{2}}{n\eta}+n\eta\bG^{2}(1+\log K)\right)\frac{1}{\sqrt{K}}\right).
\]
Setting $\eta=\frac{D^{2}}{n\bG}$ to get the best dependence on parameters.
\item Taking the stepsize $\eta_{k}=\frac{\eta}{\sqrt{K}},\forall k\in\left[K\right]$,
Algorithm \ref{alg:shuffling-SGD} guarantees
\[
F(\bx_{K+1})-F_{*}\leq\O\left(\left(\frac{D^{2}}{n\eta}+n\eta\bG^{2}(1+\log K)\right)\frac{1}{\sqrt{K}}\right).
\]
Setting $\eta=\frac{D}{n\bG\sqrt{1+\log K}}$ to get the best dependence
on parameters.
\item Taking the stepsize $\eta_{k}=\eta\frac{K-k+1}{K^{3/2}},\forall k\in\left[K\right]$,
Algorithm \ref{alg:shuffling-SGD} guarantees
\[
F(\bx_{K+1})-F_{*}\leq\O\left(\left(\frac{D^{2}}{n\eta}+n\eta\bG^{2}\right)\frac{1}{\sqrt{K}}\right).
\]
Setting $\eta=\frac{D}{n\bG}$ to get the best dependence on parameters.
\end{itemize}
\end{thm}

\begin{proof}
We invoke Lemma \ref{lem:lip-core} to get
\begin{align*}
F(\bx_{K+1})-F(\bx_{*}) & \leq\frac{\left\Vert \bx_{*}-\bx_{1}\right\Vert ^{2}-\left\Vert \bx_{*}-\bx_{K+1}\right\Vert ^{2}}{2n\sum_{k=1}^{K}\gamma_{k}}+3\bG^{2}n\sum_{k=1}^{K}\frac{\gamma_{k}\eta_{k}}{\sum_{\ell=k}^{K}\gamma_{\ell}}\\
 & \leq\frac{\left\Vert \bx_{*}-\bx_{1}\right\Vert ^{2}}{2n\sum_{k=1}^{K}\gamma_{k}}+3\bG^{2}n\sum_{k=1}^{K}\frac{\gamma_{k}\eta_{k}}{\sum_{\ell=k}^{K}\gamma_{\ell}},
\end{align*}
where $\gamma_{k}=\eta_{k}\prod_{\ell=2}^{k}(1+n\eta_{\ell-1}\mu_{\psi}),\forall k\in\left[K\right]$.
Note that $\mu_{\psi}=0\Rightarrow\gamma_{k}=\eta_{k},\forall k\in\left[K\right]$,
therefore
\begin{equation}
F(\bx_{K+1})-F(\bx_{*})\leq\frac{\left\Vert \bx_{*}-\bx_{1}\right\Vert ^{2}}{2n\sum_{k=1}^{K}\eta_{k}}+3\bG^{2}n\sum_{k=1}^{K}\frac{\eta_{k}^{2}}{\sum_{\ell=k}^{K}\eta_{\ell}}.\label{eq:lip-cvx-full-1}
\end{equation}

If $\eta_{k}=\frac{\eta}{\sqrt{k}},\forall k\in\left[K\right]$, then
\begin{align*}
F(\bx_{K+1})-F(\bx_{*}) & \overset{\eqref{eq:lip-cvx-full-1}}{\leq}\frac{\left\Vert \bx_{*}-\bx_{1}\right\Vert ^{2}}{2n\eta\sum_{k=1}^{K}1/\sqrt{k}}+3\bG^{2}n\eta\sum_{k=1}^{K}\frac{1}{k\sum_{\ell=k}^{K}1/\sqrt{\ell}}\\
 & \overset{(a)}{\leq}\frac{\left\Vert \bx_{*}-\bx_{1}\right\Vert ^{2}}{4n\eta(\sqrt{K+1}-1)}+3\bG^{2}n\eta\sum_{k=1}^{K}\frac{1}{2k(\sqrt{K+1}-\sqrt{k})}\\
 & \overset{(b)}{\leq}\frac{\left\Vert \bx_{*}-\bx_{1}\right\Vert ^{2}}{4n\eta(\sqrt{K+1}-1)}+\frac{3\bG^{2}n\eta}{\sqrt{K+1}}\sum_{k=1}^{K}\frac{1}{k}+\frac{1}{K+1-k}\\
 & =\O\left(\left(\frac{\left\Vert \bx_{*}-\bx_{1}\right\Vert ^{2}}{n\eta}+n\eta\bG^{2}(1+\log K)\right)\frac{1}{\sqrt{K}}\right),
\end{align*}
where $(a)$ is by $\sum_{\ell=k}^{K}1/\sqrt{\ell}\geq\int_{k}^{K+1}1/\sqrt{\ell}\mathrm{d}\ell=2\left(\sqrt{K+1}-\sqrt{k}\right)$
and $(b)$ is due to 
\[
\frac{1}{2k(\sqrt{K+1}-\sqrt{k})}=\frac{\sqrt{K+1}+\sqrt{k}}{2k(K+1-k)}\leq\frac{\sqrt{K+1}}{k(K+1-k)}=\frac{1}{\sqrt{K+1}}\left(\frac{1}{k}+\frac{1}{K+1-k}\right).
\]
Taking $\eta=\frac{\left\Vert \bx_{*}-\bx_{1}\right\Vert }{n\bG}$
to obtain
\[
F(\bx_{K+1})-F(\bx_{*})=\O\left(\frac{\bG\left\Vert \bx_{*}-\bx_{1}\right\Vert (1+\log K)}{\sqrt{K}}\right).
\]

If $\eta_{k}=\frac{\eta}{\sqrt{K}},\forall k\in\left[K\right]$, then
\begin{align*}
F(\bx_{K+1})-F(\bx_{*}) & \overset{\eqref{eq:lip-cvx-full-1}}{\leq}\frac{\left\Vert \bx_{*}-\bx_{1}\right\Vert ^{2}}{2n\eta\sum_{k=1}^{K}1/\sqrt{K}}+\frac{3\bG^{2}n\eta}{\sqrt{K}}\sum_{k=1}^{K}\frac{1}{K-k+1}\\
 & =\O\left(\left(\frac{\left\Vert \bx_{*}-\bx_{1}\right\Vert ^{2}}{n\eta}+n\eta\bG^{2}(1+\log K)\right)\frac{1}{\sqrt{K}}\right).
\end{align*}
Taking $\eta=\frac{\left\Vert \bx_{*}-\bx_{1}\right\Vert }{n\bG\sqrt{1+\log K}}$
to obtain
\[
F(\bx_{K+1})-F(\bx_{*})=\O\left(\frac{\bG\left\Vert \bx_{*}-\bx_{1}\right\Vert \sqrt{1+\log K}}{\sqrt{K}}\right).
\]

If $\eta_{k}=\eta\frac{K-k+1}{K^{3/2}},\forall k\in\left[K\right]$,
then
\begin{align*}
F(\bx_{K+1})-F(\bx_{*}) & \overset{\eqref{eq:lip-cvx-full-1}}{\leq}\frac{\left\Vert \bx_{*}-\bx_{1}\right\Vert ^{2}}{2n\eta\sum_{k=1}^{K}(K-k+1)/K^{3/2}}+\frac{3\bG^{2}n\eta}{K^{3/2}}\sum_{k=1}^{K}\frac{(K-k+1)^{2}}{\sum_{\ell=k}^{K}K-\ell+1}\\
 & =\frac{\sqrt{K}\left\Vert \bx_{*}-\bx_{1}\right\Vert ^{2}}{n\eta(K+1)}+\frac{6\bG^{2}n\eta}{K^{3/2}}\sum_{k=1}^{K}\frac{K-k+1}{K-k+2}\\
 & =\O\left(\left(\frac{\left\Vert \bx_{*}-\bx_{1}\right\Vert ^{2}}{n\eta}+\bG^{2}n\eta\right)\frac{1}{\sqrt{K}}\right),
\end{align*}
Taking $\eta=\frac{\left\Vert \bx_{*}-\bx_{1}\right\Vert }{n\bG}$
to obtain
\[
F(\bx_{K+1})-F(\bx_{*})=\O\left(\frac{\bG\left\Vert \bx_{*}-\bx_{1}\right\Vert }{\sqrt{K}}\right).
\]
\end{proof}

Next, Theorem \ref{thm:lip-ada-full} shows that, under the particularly
carefully designed stepsize, it is possible to obtain an asymptotic
rate with an optimal linear dependence on $D$ without knowing it.
The key is to use $r_{k}=r\lor\max_{\ell\in\left[k\right]}\left\Vert \bx_{\ell}-\bx_{1}\right\Vert $
to approximate $D$, which was orginally introduced by \cite{pmlr-v202-ivgi23a}.
\begin{thm}
\label{thm:lip-ada-full}(Full version of Theorem \ref{thm:lip-ada-main})
Under Assumptions \ref{assu:minimizer}, \ref{assu:basic} (with $\mu_{\psi}=0$)
and \ref{assu:lip}, regardless of how the permutation $\sigma_{k}$
is generated in every epoch, taking the stepsize $\eta_{k}=r_{k}\widetilde{\eta}_{k},\forall k\in\N$
where $r_{k}=r\lor\max_{\ell\in\left[k\right]}\left\Vert \bx_{\ell}-\bx_{1}\right\Vert $
for some $r>0$ and $\widetilde{\eta}_{k}=\frac{c}{n\bG\sqrt{6(1+\delta^{-1})k(1+\log k)^{1+\delta}}},\forall k\in\N$
for some $\delta>0,0<c<1$, Algorithm \ref{alg:shuffling-SGD} guarantees
\[
F(\bx_{K+1})-F_{*}\leq\left(\frac{r_{K+1}}{c\sum_{k=1}^{K}r_{k}/K}+\frac{3c}{1-c}\right)\bG\left(D\lor r\right)\sqrt{\frac{6(1+\delta^{-1})(1+\log K)^{1+\delta}}{K}}.
\]
Moreover, for large enough $K$
\[
F(\bx_{K+1})-F_{*}\leq\mathcal{O}\left(\frac{3c^{2}-c+1}{c(1-c)}\bG\left(D\lor r\right)\sqrt{\frac{(1+\delta^{-1})(1+\log K)^{1+\delta}}{K}}\right).
\]
\end{thm}

\begin{proof}
We first invoke Lemma \ref{lem:lip-core} to obtain
\begin{align*}
F(\bx_{K+1})-F(\bx_{*}) & \leq\frac{\left\Vert \bx_{*}-\bx_{1}\right\Vert ^{2}-\left\Vert \bx_{*}-\bx_{K+1}\right\Vert ^{2}}{2n\sum_{k=1}^{K}\gamma_{k}}+3\bG^{2}n\sum_{k=1}^{K}\frac{\gamma_{k}\eta_{k}}{\sum_{\ell=k}^{K}\gamma_{\ell}}\\
 & =\frac{\left\Vert \bx_{*}-\bx_{1}\right\Vert ^{2}-\left\Vert \bx_{*}-\bx_{K+1}\right\Vert ^{2}}{2n\sum_{k=1}^{K}\eta_{k}}+3\bG^{2}n\sum_{k=1}^{K}\frac{\eta_{k}^{2}}{\sum_{\ell=k}^{K}\eta_{\ell}},
\end{align*}
where the last equation holds due to $\gamma_{k}=\eta_{k}\prod_{\ell=2}^{k}(1+n\eta_{\ell-1}\mu_{\psi})=\eta_{k}$
when $\mu_{\psi}=0$. Now if $\left\Vert \bx_{*}-\bx_{1}\right\Vert \geq\left\Vert \bx_{*}-\bx_{K+1}\right\Vert $,
we have 
\begin{align*}
\left\Vert \bx_{*}-\bx_{1}\right\Vert ^{2}-\left\Vert \bx_{*}-\bx_{K+1}\right\Vert ^{2} & =\left(\left\Vert \bx_{*}-\bx_{1}\right\Vert -\left\Vert \bx_{*}-\bx_{K+1}\right\Vert \right)\left(\left\Vert \bx_{*}-\bx_{1}\right\Vert +\left\Vert \bx_{*}-\bx_{K+1}\right\Vert \right)\\
 & \leq2\left\Vert \bx_{1}-\bx_{K+1}\right\Vert \left\Vert \bx_{*}-\bx_{1}\right\Vert \leq2r_{K+1}\left\Vert \bx_{*}-\bx_{1}\right\Vert ,
\end{align*}
where we use $r_{K+1}=r\lor\max_{k\in\left[K+1\right]}\left\Vert \bx_{1}-\bx_{k}\right\Vert \geq\left\Vert \bx_{1}-\bx_{K+1}\right\Vert $
in the last step. If $\left\Vert \bx_{*}-\bx_{1}\right\Vert <\left\Vert \bx_{*}-\bx_{K+1}\right\Vert $,
there is still $\left\Vert \bx_{*}-\bx_{1}\right\Vert ^{2}-\left\Vert \bx_{*}-\bx_{K+1}\right\Vert ^{2}\leq0\leq2r_{K+1}\left\Vert \bx_{*}-\bx_{1}\right\Vert $.
Hence, we always have
\[
\left\Vert \bx_{*}-\bx_{1}\right\Vert ^{2}-\left\Vert \bx_{*}-\bx_{K+1}\right\Vert ^{2}\leq2r_{K+1}\left\Vert \bx_{*}-\bx_{1}\right\Vert .
\]
Now there is
\begin{align}
F(\bx_{K+1})-F(\bx_{*}) & \leq\frac{r_{K+1}\left\Vert \bx_{*}-\bx_{1}\right\Vert }{n\sum_{k=1}^{K}\eta_{k}}+3\bG^{2}n\sum_{k=1}^{K}\frac{\eta_{k}^{2}}{\sum_{\ell=k}^{K}\eta_{\ell}}\nonumber \\
 & =\frac{r_{K+1}\left\Vert \bx_{*}-\bx_{1}\right\Vert }{n\sum_{k=1}^{K}r_{k}\widetilde{\eta}_{k}}+3\bG^{2}n\sum_{k=1}^{K}\frac{r_{k}^{2}\widetilde{\eta}_{k}^{2}}{\sum_{\ell=k}^{K}r_{\ell}\widetilde{\eta}_{\ell}}\nonumber \\
 & \overset{(a)}{\leq}\frac{r_{K+1}\left\Vert \bx_{*}-\bx_{1}\right\Vert }{n\sum_{k=1}^{K}r_{k}\widetilde{\eta}_{k}}+3r_{K}\bG^{2}n\sum_{k=1}^{K}\frac{\widetilde{\eta}_{k}^{2}}{\sum_{\ell=k}^{K}\widetilde{\eta}_{\ell}}\nonumber \\
 & \overset{(b)}{\leq}\frac{r_{K+1}}{\sum_{k=1}^{K}r_{k}}\cdot\frac{\left\Vert \bx_{*}-\bx_{1}\right\Vert }{n\widetilde{\eta}_{K}}+3r_{K}\bG^{2}n\sum_{k=1}^{K}\frac{\widetilde{\eta}_{k}^{2}}{\sum_{\ell=k}^{K}\widetilde{\eta}_{\ell}},\label{eq:lip-ada-full-1}
\end{align}
where $(a)$ is due to $\frac{r_{k}^{2}\widetilde{\eta}_{k}^{2}}{\sum_{\ell=k}^{K}r_{\ell}\widetilde{\eta}_{\ell}}\leq\frac{r_{k}\widetilde{\eta}_{k}^{2}}{\sum_{\ell=k}^{K}\widetilde{\eta}_{\ell}}\leq\frac{r_{K}\widetilde{\eta}_{k}^{2}}{\sum_{\ell=k}^{K}\widetilde{\eta}_{\ell}},\forall k\in\left[K\right]$
because $r_{k}$ is non-decreasing and $(b)$ is by $\widetilde{\eta}_{K}\leq\widetilde{\eta}_{k},\forall k\in\left[K\right]$
because $\widetilde{\eta}_{k}$ is non-increasing.

From $\widetilde{\eta}_{k}=\frac{c}{n\bG\sqrt{6(1+\delta^{-1})k(1+\log k)^{1+\delta}}}$,
we have
\begin{align}
\sum_{k=1}^{K}\frac{\widetilde{\eta}_{k}^{2}}{\sum_{\ell=k}^{K}\widetilde{\eta}_{\ell}} & =\frac{c}{n\bG\sqrt{6(1+\delta^{-1})}}\sum_{k=1}^{K}\frac{1}{k(1+\log k)^{1+\delta}\sum_{\ell=k}^{K}1/\sqrt{\ell(1+\log\ell)^{1+\delta}}}\nonumber \\
 & \leq\frac{c}{n\bG}\sqrt{\frac{(1+\log K)^{1+\delta}}{6(1+\delta^{-1})}}\sum_{k=1}^{K}\frac{1}{k(1+\log k)^{1+\delta}\sum_{\ell=k}^{K}1/\sqrt{\ell}}\nonumber \\
 & \overset{(c)}{\leq}\frac{c}{n\bG}\sqrt{\frac{(1+\log K)^{1+\delta}}{6(1+\delta^{-1})}}\sum_{k=1}^{K}\frac{1}{2k(1+\log k)^{1+\delta}(\sqrt{K+1}-\sqrt{k})}\nonumber \\
 & \overset{(d)}{\leq}\frac{c}{n\bG}\sqrt{\frac{(1+\log K)^{1+\delta}}{6(1+\delta^{-1})(K+1)}}\sum_{k=1}^{K}\frac{1}{(1+\log k)^{1+\delta}}\left(\frac{1}{k}+\frac{1}{K+1-k}\right)\nonumber \\
 & \overset{(e)}{\leq}\frac{c}{n\bG}\sqrt{\frac{(1+\log K)^{1+\delta}}{6(1+\delta^{-1})(K+1)}}\sum_{k=1}^{K}\frac{2}{k(1+\log k)^{1+\delta}}\nonumber \\
 & \overset{(f)}{\leq}\frac{c}{n\bG}\sqrt{\frac{2(1+\delta^{-1})(1+\log K)^{1+\delta}}{3K}},\label{eq:lip-ada-full-2}
\end{align}
where $(c)$ is by $\sum_{\ell=k}^{K}1/\sqrt{\ell}\geq\int_{k}^{K+1}1/\sqrt{\ell}\mathrm{d}\ell=2\left(\sqrt{K+1}-\sqrt{k}\right)$
and $(d)$ is due to
\[
\frac{1}{2k(\sqrt{K+1}-\sqrt{k})}=\frac{\sqrt{K+1}+\sqrt{k}}{2k(K+1-k)}\leq\frac{\sqrt{K+1}}{k(K+1-k)}=\frac{1}{\sqrt{K+1}}\left(\frac{1}{k}+\frac{1}{K+1-k}\right).
\]
$(e)$ is by noticing that both $\frac{1}{k}$ and $\frac{1}{(1+\log k)^{1+\delta}}$
are decreasing sequences, then by rearrangement inequality
\[
\sum_{k=1}^{K}\frac{1}{(1+\log k)^{1+\delta}}\times\frac{1}{K+1-k}\leq\sum_{k=1}^{K}\frac{1}{(1+\log k)^{1+\delta}}\times\frac{1}{k}.
\]
$(f)$ is because
\begin{align}
\sum_{k=1}^{K}\frac{1}{k(1+\log k)^{1+\delta}} & \leq\sum_{k=1}^{\infty}\frac{1}{k(1+\log k)^{1+\delta}}\leq1+\int_{1}^{\infty}\frac{1}{k(1+\log k)^{1+\delta}}\mathrm{d}k\nonumber \\
 & =1+\left(-\delta^{-1}(1+\log k)^{-\delta}\right)\vert_{1}^{\infty}=1+\delta^{-1}.\label{eq:lip-ada-full-ub}
\end{align}

Next let us check the conditions in Lemma \ref{lem:lip-ada-distance}.
For Condition \ref{enu:condition-1}, we notice that $\widetilde{\eta}_{k}$
is positive defined on $\N$ from its definition. For Condition \ref{enu:condition-2},
we compute
\[
\sum_{k=1}^{\infty}6\bG^{2}n^{2}\widetilde{\eta}_{k}^{2}=\frac{c^{2}}{1+\delta^{-1}}\sum_{k=1}^{\infty}\frac{1}{k(1+\log k)^{1+\delta}}\overset{\eqref{eq:lip-ada-full-ub}}{\leq}c^{2}<1.
\]
Then we can invoke Lemma \ref{lem:lip-ada-distance} to get
\[
\left\Vert \bx_{k}-\bx_{1}\right\Vert \leq\frac{2}{1-c}\left\Vert \bx_{*}-\bx_{1}\right\Vert +\frac{c}{1-c}r,\forall k\in\N,
\]
which implies
\begin{align}
r_{K} & =r\lor\max_{k\in\left[K\right]}\left\Vert \bx_{k}-\bx_{1}\right\Vert \leq r+\frac{2}{1-c}\left\Vert \bx_{*}-\bx_{1}\right\Vert +\frac{c}{1-c}r\nonumber \\
 & =\frac{2\left\Vert \bx_{*}-\bx_{1}\right\Vert +r}{1-c}\leq\frac{3}{1-c}\left(\left\Vert \bx_{*}-\bx_{1}\right\Vert \lor r\right).\label{eq:lip-ada-full-3}
\end{align}

Plugging (\ref{eq:lip-ada-full-2}) and (\ref{eq:lip-ada-full-3})
into (\ref{eq:lip-ada-full-1}) to get
\begin{align}
F(\bx_{K+1})-F(\bx_{*})\leq & \frac{r_{K+1}}{\sum_{k=1}^{K}r_{k}}\cdot\frac{\left\Vert \bx_{*}-\bx_{1}\right\Vert }{n\widetilde{\eta}_{K}}+\frac{3c}{1-c}\bG\left(\left\Vert \bx_{*}-\bx_{1}\right\Vert \lor r\right)\sqrt{\frac{6(1+\delta^{-1})(1+\log K)^{1+\delta}}{K}}\nonumber \\
= & \frac{r_{K+1}}{\sum_{k=1}^{K}r_{k}}\cdot\frac{\left\Vert \bx_{*}-\bx_{1}\right\Vert \bG\sqrt{6(1+\delta^{-1})K(1+\log K)^{1+\delta}}}{c}\nonumber \\
 & +\frac{3c}{1-c}\bG\left(\left\Vert \bx_{*}-\bx_{1}\right\Vert \lor r\right)\sqrt{\frac{6(1+\delta^{-1})(1+\log K)^{1+\delta}}{K}}\nonumber \\
\leq & \left(\frac{r_{K+1}}{c\sum_{k=1}^{K}r_{k}/K}+\frac{3c}{1-c}\right)\bG\left(\left\Vert \bx_{*}-\bx_{1}\right\Vert \lor r\right)\sqrt{\frac{6(1+\delta^{-1})(1+\log K)^{1+\delta}}{K}},\label{eq:lip-ada-full-4}
\end{align}
where we use the definition of $\widetilde{\eta}_{k}=\frac{c}{n\bG\sqrt{6(1+\delta^{-1})k(1+\log k)^{1+\delta}}}$
in the second line.

Finally, note that $r_{k}$ is non-decreasing with a uniform upper
bound $\frac{3}{1-c}\left(\left\Vert \bx_{*}-\bx_{1}\right\Vert \lor r\right)$
by (\ref{eq:lip-ada-full-3}). Hence, there exists $\widetilde{r}=\lim_{k\to\infty}r_{k}<\infty$,
which implies
\[
\lim_{K\to\infty}\frac{r_{K+1}}{\sum_{k=1}^{K}r_{k}/K}=\frac{\widetilde{r}}{\widetilde{r}}=1.
\]
Thus, for large enough $K$, $\frac{r_{K+1}}{\sum_{k=1}^{K}r_{k}/K}=\mathcal{O}(1)$.
Combining (\ref{eq:lip-ada-full-4}), the following bound holds asymptotically
\[
F(\bx_{K+1})-F(\bx_{*})\leq\mathcal{O}\left(\frac{3c^{2}-c+1}{c(1-c)}\bG\left(\left\Vert \bx_{*}-\bx_{1}\right\Vert \lor r\right)\sqrt{\frac{(1+\delta^{-1})(1+\log K)^{1+\delta}}{K}}\right).
\]
\end{proof}

Finally, we turn the attention to the case of strongly convex $\psi(\bx)$.
Instead of $\eta_{k}=\frac{2}{n\mu_{\psi}k}$ presented in Theorem
\ref{thm:lip-str-main}, we show that $\eta_{k}=\frac{m}{n\mu_{\psi}k}$
for any $m\in\N$ guarantees the last-iterate convergence rate $O(m\log K/K)$.
\begin{thm}
\label{thm:lip-str-full}(Full version of Theorem \ref{thm:lip-str-main})
Under Assumptions \ref{assu:minimizer}, \ref{assu:basic} (with $\mu_{\psi}>0$)
and \ref{assu:lip}, regardless of how the permutation $\sigma_{k}$
is generated in every epoch, taking the stepsize $\eta_{k}=\frac{m}{n\mu_{\psi}k},\forall k\in\left[K\right]$
where $m\in\N$ can be chosen arbitrarily, Algorithm \ref{alg:shuffling-SGD}
guarantees
\[
F(\bx_{K+1})-F_{*}\leq\O\left(\frac{\mu_{\psi}D^{2}}{\binom{K+m}{m}}+\frac{\bG^{2}m\log K}{\mu_{\psi}K}\right).
\]
\end{thm}

\begin{proof}
We invoke Lemma \ref{lem:lip-core} to get
\begin{align*}
F(\bx_{K+1})-F(\bx_{*}) & \leq\frac{\left\Vert \bx_{*}-\bx_{1}\right\Vert ^{2}-\left\Vert \bx_{*}-\bx_{K+1}\right\Vert ^{2}}{2n\sum_{k=1}^{K}\gamma_{k}}+3\bG^{2}n\sum_{k=1}^{K}\frac{\gamma_{k}\eta_{k}}{\sum_{\ell=k}^{K}\gamma_{\ell}}\\
 & \leq\frac{\left\Vert \bx_{*}-\bx_{1}\right\Vert ^{2}}{2n\sum_{k=1}^{K}\gamma_{k}}+3\bG^{2}n\sum_{k=1}^{K}\frac{\gamma_{k}\eta_{k}}{\sum_{\ell=k}^{K}\gamma_{\ell}},
\end{align*}
where $\gamma_{k}=\eta_{k}\prod_{\ell=2}^{k}(1+n\eta_{\ell-1}\mu_{\psi}),\forall k\in\left[K\right]$.
When $\eta_{k}=\frac{m}{n\mu_{\psi}k},\forall k\in\left[K\right]$,
we have
\[
\gamma_{k}=\frac{m}{n\mu_{\psi}k}\prod_{\ell=2}^{k}\frac{\ell-1+m}{\ell-1}=\frac{1}{n\mu_{\psi}}\cdot\frac{(k+m-1)!}{(m-1)!k!}=\frac{\binom{k+m-1}{m-1}}{n\mu_{\psi}}.
\]
Thus
\begin{align*}
F(\bx_{K+1})-F(\bx_{*}) & \leq\frac{\mu_{\psi}\left\Vert \bx_{*}-\bx_{1}\right\Vert ^{2}}{2\sum_{k=1}^{K}\binom{k+m-1}{m-1}}+\frac{3\bG^{2}m}{\mu_{\psi}}\sum_{k=1}^{K}\frac{\binom{k+m-1}{m-1}}{k\sum_{\ell=k}^{K}\binom{\ell+m-1}{m-1}}\\
 & \overset{(a)}{=}\frac{\mu_{\psi}\left\Vert \bx_{*}-\bx_{1}\right\Vert ^{2}}{2\left(\binom{K+m}{m}-1\right)}+\frac{3\bG^{2}m}{\mu_{\psi}}\sum_{k=1}^{K}\frac{\binom{k+m-1}{m-1}}{k\sum_{\ell=k}^{K}\binom{\ell+m-1}{m-1}}\\
 & \overset{(b)}{\leq}\frac{\mu_{\psi}\left\Vert \bx_{*}-\bx_{1}\right\Vert ^{2}}{2\left(\binom{K+m}{m}-1\right)}+\frac{3\bG^{2}m}{\mu_{\psi}}\sum_{k=1}^{K}\frac{1}{k(K-k+1)}\\
 & =\frac{\mu_{\psi}\left\Vert \bx_{*}-\bx_{1}\right\Vert ^{2}}{2\left(\binom{K+m}{m}-1\right)}+\frac{3\bG^{2}m}{\mu_{\psi}(K+1)}\sum_{k=1}^{K}\frac{1}{k}+\frac{1}{K-k+1}\\
 & =\O\left(\frac{\mu_{\psi}\left\Vert \bx_{*}-\bx_{1}\right\Vert ^{2}}{\binom{K+m}{m}}+\frac{\bG^{2}m\log K}{\mu_{\psi}K}\right).
\end{align*}
where $(a)$ is due to
\[
\sum_{k=1}^{K}\binom{k+m-1}{m-1}=\sum_{k=1}^{K}\binom{k+m}{m}-\binom{k+m-1}{m}=\binom{K+m}{m}-1,
\]
and $(b)$ is by $\binom{\ell+m-1}{m-1}\geq\binom{k+m-1}{m-1},\forall\ell\geq k$.
\end{proof}

\section{Theoretical Analysis\label{sec:lemmas}}

In this section, we provide our analysis and several lemmas.

\subsection{Core Lemmas}

This subsection contains three core lemmas that play fundamental roles
in the analysis.

Inspired by \cite{mishchenko2020random}, we first introduce the following
descent inequality, which provides the progress on the objective value
in one epoch. We emphasize that Lemma \ref{lem:core} holds for any
point $\bz\in\R^{d}$ rather than only $\bx_{*}$. This important
fact ensures that we can use the key idea proposed by \cite{zamani2023exact}
and later developed by \cite{liu2024revisiting}. However, some barriers
would appear if we apply the technique used in \cite{zamani2023exact}
directly. Hence, our analysis departs from the previous paper and
needs new tools, e.g., Lemma \ref{lem:ineq}. The reader could refer
to Lemmas \ref{lem:smooth-core} and \ref{lem:lip-core} for how our
analysis proceeds.
\begin{lem}
\label{lem:core}Under Assumption \ref{assu:basic}, for any $k\in\left[K\right]$,
permutation $\sigma_{k}$ and $\bz\in\R^{d}$, Algorithm \ref{alg:shuffling-SGD}
guarantees 
\begin{align*}
F(\bx_{k+1})-F(\bz)\leq & \frac{\left\Vert \bz-\bx_{k}\right\Vert ^{2}}{2n\eta_{k}}-\left(\frac{1}{\eta_{k}}+n\mu_{\psi}\right)\frac{\left\Vert \bz-\bx_{k+1}\right\Vert ^{2}}{2n}-\frac{\left\Vert \bx_{k+1}-\bx_{k}\right\Vert ^{2}}{2n\eta_{k}}\\
 & +\frac{1}{n}\sum_{i=1}^{n}B_{\sigma_{k}^{i}}(\bx_{k+1},\bx_{k}^{i})-B_{\sigma_{k}^{i}}(\bz,\bx_{k}^{i}).
\end{align*}
\end{lem}

\begin{proof}
It is enough to only consider the case $\bz\in\mathrm{dom}\psi$.
Let $\bg_{k}\triangleq\sum_{i=1}^{n}\nabla f_{\sigma_{k}^{i}}(\bx_{k}^{i})$,
then $\bx_{k}^{n+1}=\bx_{k}^{1}-\eta_{k}\bg_{k}=\bx_{k}-\eta_{k}\bg_{k}$.
Observe that
\[
\bx_{k+1}=\argmin_{\bx\in\R^{d}}n\psi(\bx)+\frac{\left\Vert \bx-\bx_{k}^{n+1}\right\Vert ^{2}}{2\eta_{k}}=\argmin_{\bx\in\R^{d}}n\psi(\bx)+\langle\bg_{k},\bx-\bx_{k}\rangle+\frac{\left\Vert \bx-\bx_{k}\right\Vert ^{2}}{2\eta_{k}}.
\]
By the first-order optimality condition, there exists $\nabla\psi(\bx_{k+1})\in\partial\psi(\bx_{k+1})$
such that
\[
\mathbf{0}=n\nabla\psi(\bx_{k+1})+\bg_{k}+\frac{\bx_{k+1}-\bx_{k}}{\eta_{k}}.
\]
Therefore for any $\bz\in\mathrm{dom}\psi$
\begin{align}
\langle\bg_{k},\bx_{k+1}-\bz\rangle & =n\langle\nabla\psi(\bx_{k+1}),\bz-\bx_{k+1}\rangle+\frac{\langle\bx_{k}-\bx_{k+1},\bx_{k+1}-\bz\rangle}{\eta_{k}}\nonumber \\
 & \overset{(a)}{\leq}n\left(\psi(\bz)-\psi(\bx_{k+1})-\frac{\mu_{\psi}}{2}\left\Vert \bz-\bx_{k+1}\right\Vert ^{2}\right)+\frac{\langle\bx_{k}-\bx_{k+1},\bx_{k+1}-\bz\rangle}{\eta_{k}}\nonumber \\
 & =n\left(\psi(\bz)-\psi(\bx_{k+1})-\frac{\mu_{\psi}}{2}\left\Vert \bz-\bx_{k+1}\right\Vert ^{2}\right)+\frac{\left\Vert \bz-\bx_{k}\right\Vert ^{2}-\left\Vert \bz-\bx_{k+1}\right\Vert ^{2}-\left\Vert \bx_{k+1}-\bx_{k}\right\Vert ^{2}}{2\eta_{k}}\nonumber \\
 & =n\left(\psi(\bz)-\psi(\bx_{k+1})\right)+\frac{\left\Vert \bz-\bx_{k}\right\Vert ^{2}}{2\eta_{k}}-\left(\frac{1}{\eta_{k}}+n\mu_{\psi}\right)\frac{\left\Vert \bz-\bx_{k+1}\right\Vert ^{2}}{2}-\frac{\left\Vert \bx_{k+1}-\bx_{k}\right\Vert ^{2}}{2\eta_{k}},\label{eq:core-1}
\end{align}
where $(a)$ is by the strong convexity of $\psi$. Note that (recall
$B_{\sigma_{k}^{i}}\triangleq B_{f_{\sigma_{k}^{i}}}$)
\[
\langle\nabla f_{\sigma_{k}^{i}}(\bx_{k}^{i}),\bx_{k+1}-\bz\rangle=f_{\sigma_{k}^{i}}(\bx_{k+1})-f_{\sigma_{k}^{i}}(\bz)-B_{\sigma_{k}^{i}}(\bx_{k+1},\bx_{k}^{i})+B_{\sigma_{k}^{i}}(\bz,\bx_{k}^{i}),
\]
which implies 
\begin{align}
\langle\bg_{k},\bx_{k+1}-\bz\rangle & =\sum_{i=1}^{n}\langle\nabla f_{\sigma_{k}^{i}}(\bx_{k}^{i}),\bx_{k+1}-\bz\rangle=\sum_{i=1}^{n}f_{\sigma_{k}^{i}}(\bx_{k+1})-f_{\sigma_{k}^{i}}(\bz)-B_{\sigma_{k}^{i}}(\bx_{k+1},\bx_{k}^{i})+B_{\sigma_{k}^{i}}(\bz,\bx_{k}^{i})\nonumber \\
 & =n\left(f(\bx_{k+1})-f(\bz)\right)-\sum_{i=1}^{n}B_{\sigma_{k}^{i}}(\bx_{k+1},\bx_{k}^{i})-B_{\sigma_{k}^{i}}(\bz,\bx_{k}^{i}).\label{eq:core-2}
\end{align}
Plugging (\ref{eq:core-2}) into (\ref{eq:core-1}) and rearranging
the terms to obtain
\begin{align*}
F(\bx_{k+1})-F(\bz)\leq & \frac{\left\Vert \bz-\bx_{k}\right\Vert ^{2}}{2n\eta_{k}}-\left(\frac{1}{\eta_{k}}+n\mu_{\psi}\right)\frac{\left\Vert \bz-\bx_{k+1}\right\Vert ^{2}}{2n}-\frac{\left\Vert \bx_{k+1}-\bx_{k}\right\Vert ^{2}}{2n\eta_{k}}\\
 & +\frac{1}{n}\sum_{i=1}^{n}B_{\sigma_{k}^{i}}(\bx_{k+1},\bx_{k}^{i})-B_{\sigma_{k}^{i}}(\bz,\bx_{k}^{i}).
\end{align*}
\end{proof}

Note that Lemma \ref{lem:core} includes a term $\frac{1}{n}\sum_{i=1}^{n}B_{\sigma_{k}^{i}}(\bx_{k+1},\bx_{k}^{i})-B_{\sigma_{k}^{i}}(\bz,\bx_{k}^{i})$,
which can be bounded easily when $n=1$. However, for a general $n$,
we need to make extra efforts to find a proper bound on it.

In Lemma \ref{lem:smooth-res}, we first provide the bound on the
extra term for smooth functions. There are several points we want
to discuss here. First, we need the stepsize to satisfy $\eta_{k}\leq\frac{1}{2n\sqrt{\bL L^{*}}}$.
Similar requirements also existed in previous works (e.g., \cite{mishchenko2020random,pmlr-v202-cha23a}).
Besides the normal term $\bL\left\Vert \bx_{k+1}-\bx_{k}\right\Vert ^{2}$,
there are two unusual quantities appearing in the R.H.S. of the inequality.
One is the term $B_{f}(\bz,\bx_{*})$, which leads to the main difficulty
in the analysis (see Lemma \ref{lem:smooth-core} for details). The
other is the residual $R_{k}$. Fortunately, it can bounded as shown
in Lemma \ref{lem:var} in Section \ref{sec:tech}.
\begin{lem}
\label{lem:smooth-res}Under Assumptions \ref{assu:minimizer}, \ref{assu:basic}
and \ref{assu:smooth}, suppose $\eta_{k}\leq\frac{1}{2n\sqrt{\bL L^{*}}},\forall k\in\left[K\right]$,
then for any $k\in\left[K\right]$, permutation $\sigma_{k}$ and
$\bz\in\R^{d}$, Algorithm \ref{alg:shuffling-SGD} guarantees 
\[
\frac{1}{n}\sum_{i=1}^{n}B_{\sigma_{k}^{i}}(\bx_{k+1},\bx_{k}^{i})-B_{\sigma_{k}^{i}}(\bz,\bx_{k}^{i})\leq\bL\left\Vert \bx_{k+1}-\bx_{k}\right\Vert ^{2}+8\eta_{k}^{2}n^{2}\bL^{2}B_{f}(\bz,\bx_{*})+4\eta_{k}^{2}R_{k},
\]
where $R_{k}\triangleq\sum_{i=2}^{n}\frac{L_{\sigma_{k}^{i}}}{n}\left\Vert \sum_{j=1}^{i-1}\nabla f_{\sigma_{k}^{j}}(\bx_{*})\right\Vert ^{2},\forall k\in\left[K\right]$.
\end{lem}

\begin{proof}
By Lemma \ref{lem:cocore}, there are
\begin{align*}
B_{\sigma_{k}^{i}}(\bx_{k+1},\bx_{k}^{i}) & \leq\frac{L_{\sigma_{k}^{i}}}{2}\left\Vert \bx_{k+1}-\bx_{k}^{i}\right\Vert ^{2}\leq L_{\sigma_{k}^{i}}\left(\left\Vert \bx_{k+1}-\bx_{k}\right\Vert ^{2}+\left\Vert \bx_{k}^{i}-\bx_{k}\right\Vert ^{2}\right),\\
B_{\sigma_{k}^{i}}(\bz,\bx_{k}^{i}) & \geq\frac{\left\Vert \nabla f_{\sigma_{k}^{i}}(\bx_{k}^{i})-\nabla f_{\sigma_{k}^{i}}(\bz)\right\Vert ^{2}}{2L_{\sigma_{k}^{i}}}.
\end{align*}
Therefore
\begin{align}
\frac{1}{n}\sum_{i=1}^{n}B_{\sigma_{k}^{i}}(\bx_{k+1},\bx_{k}^{i})-B_{\sigma_{k}^{i}}(\bz,\bx_{k}^{i}) & \leq\frac{1}{n}\sum_{i=1}^{n}L_{\sigma_{k}^{i}}\left(\left\Vert \bx_{k+1}-\bx_{k}\right\Vert ^{2}+\left\Vert \bx_{k}^{i}-\bx_{k}\right\Vert ^{2}\right)-\frac{\left\Vert \nabla f_{\sigma_{k}^{i}}(\bx_{k}^{i})-\nabla f_{\sigma_{k}^{i}}(\bz)\right\Vert ^{2}}{2L_{\sigma_{k}^{i}}}\nonumber \\
 & =\bL\left\Vert \bx_{k+1}-\bx_{k}\right\Vert ^{2}+\frac{1}{n}\sum_{i=1}^{n}L_{\sigma_{k}^{i}}\left\Vert \bx_{k}^{i}-\bx_{k}\right\Vert ^{2}-\frac{\left\Vert \nabla f_{\sigma_{k}^{i}}(\bx_{k}^{i})-\nabla f_{\sigma_{k}^{i}}(\bz)\right\Vert ^{2}}{2L_{\sigma_{k}^{i}}}.\label{eq:smooth-res-1}
\end{align}
Note that $\bx_{k}=\bx_{k}^{1}$, hence
\begin{align}
 & \sum_{i=1}^{n}L_{\sigma_{k}^{i}}\left\Vert \bx_{k}^{i}-\bx_{k}\right\Vert ^{2}=\sum_{i=2}^{n}L_{\sigma_{k}^{i}}\left\Vert \bx_{k}^{i}-\bx_{k}^{1}\right\Vert ^{2}=\eta_{k}^{2}\sum_{i=2}^{n}L_{\sigma_{k}^{i}}\left\Vert \sum_{j=1}^{i-1}\nabla f_{\sigma_{k}^{j}}(\bx_{k}^{j})\right\Vert ^{2}\nonumber \\
= & \eta_{k}^{2}\sum_{i=2}^{n}L_{\sigma_{k}^{i}}\left\Vert \sum_{j=1}^{i-1}\nabla f_{\sigma_{k}^{j}}(\bx_{k}^{j})-\nabla f_{\sigma_{k}^{j}}(\bz)+\nabla f_{\sigma_{k}^{j}}(\bz)-\nabla f_{\sigma_{k}^{j}}(\bx_{*})+\nabla f_{\sigma_{k}^{j}}(\bx_{*})\right\Vert ^{2}\nonumber \\
\leq & \eta_{k}^{2}\sum_{i=2}^{n}L_{\sigma_{k}^{i}}\left(2\left\Vert \sum_{j=1}^{i-1}\nabla f_{\sigma_{k}^{j}}(\bx_{k}^{j})-\nabla f_{\sigma_{k}^{j}}(\bz)\right\Vert ^{2}+4\left\Vert \sum_{j=1}^{i-1}\nabla f_{\sigma_{k}^{j}}(\bz)-\nabla f_{\sigma_{k}^{j}}(\bx_{*})\right\Vert ^{2}+4\left\Vert \sum_{j=1}^{i-1}\nabla f_{\sigma_{k}^{j}}(\bx_{*})\right\Vert ^{2}\right)\nonumber \\
= & 2\eta_{k}^{2}\sum_{i=2}^{n}L_{\sigma_{k}^{i}}\left\Vert \sum_{j=1}^{i-1}\nabla f_{\sigma_{k}^{j}}(\bx_{k}^{j})-\nabla f_{\sigma_{k}^{j}}(\bz)\right\Vert ^{2}+4\eta_{k}^{2}\sum_{i=2}^{n}L_{\sigma_{k}^{i}}\left\Vert \sum_{j=1}^{i-1}\nabla f_{\sigma_{k}^{j}}(\bz)-\nabla f_{\sigma_{k}^{j}}(\bx_{*})\right\Vert ^{2}\nonumber \\
 & +4\eta_{k}^{2}\sum_{i=2}^{n}L_{\sigma_{k}^{i}}\left\Vert \sum_{j=1}^{i-1}\nabla f_{\sigma_{k}^{j}}(\bx_{*})\right\Vert ^{2}.\label{eq:smooth-res-2}
\end{align}
We bound
\begin{align}
 & \sum_{i=2}^{n}L_{\sigma_{k}^{i}}\left\Vert \sum_{j=1}^{i-1}\nabla f_{\sigma_{k}^{j}}(\bx_{k}^{j})-\nabla f_{\sigma_{k}^{j}}(\bz)\right\Vert ^{2}\nonumber \\
\leq & \sum_{i=2}^{n}\sum_{j=1}^{i-1}L_{\sigma_{k}^{i}}(i-1)\left\Vert \nabla f_{\sigma_{k}^{j}}(\bx_{k}^{j})-\nabla f_{\sigma_{k}^{j}}(\bz)\right\Vert ^{2}=\sum_{j=1}^{n-1}\left(\sum_{i=j+1}^{n}L_{\sigma_{k}^{i}}(i-1)\right)\left\Vert \nabla f_{\sigma_{k}^{j}}(\bx_{k}^{j})-\nabla f_{\sigma_{k}^{j}}(\bz)\right\Vert ^{2}\nonumber \\
\leq & \sum_{j=1}^{n-1}n^{2}\bL\left\Vert \nabla f_{\sigma_{k}^{j}}(\bx_{k}^{j})-\nabla f_{\sigma_{k}^{j}}(\bz)\right\Vert ^{2}\leq\sum_{i=1}^{n}n^{2}\bL\left\Vert \nabla f_{\sigma_{k}^{i}}(\bx_{k}^{i})-\nabla f_{\sigma_{k}^{i}}(\bz)\right\Vert ^{2},\label{eq:smooth-res-3}
\end{align}
and
\begin{align}
 & \sum_{i=2}^{n}L_{\sigma_{k}^{i}}\left\Vert \sum_{j=1}^{i-1}\nabla f_{\sigma_{k}^{j}}(\bz)-\nabla f_{\sigma_{k}^{j}}(\bx_{*})\right\Vert ^{2}\nonumber \\
\overset{(a)}{\leq} & \sum_{i=2}^{n}L_{\sigma_{k}^{i}}\times2\left(\sum_{j=1}^{i-1}L_{\sigma_{k}^{j}}\right)\left(\sum_{\ell=1}^{i-1}B_{\sigma_{k}^{\ell}}(\bz,\bx_{*})\right)=\sum_{i=2}^{n}\sum_{j=1}^{i-1}\sum_{\ell=1}^{i-1}2L_{\sigma_{k}^{i}}L_{\sigma_{k}^{j}}B_{\sigma_{k}^{\ell}}(\bz,\bx_{*})\nonumber \\
= & \sum_{j=1}^{n-1}\sum_{\ell=1}^{n-1}\left(\sum_{i=j\lor\ell+1}^{n}2L_{\sigma_{k}^{i}}\right)L_{\sigma_{k}^{j}}B_{\sigma_{k}^{\ell}}(\bz,\bx_{*})\leq2n\bL\sum_{j=1}^{n-1}\sum_{\ell=1}^{n-1}L_{\sigma_{k}^{j}}B_{\sigma_{k}^{\ell}}(\bz,\bx_{*})\nonumber \\
\leq & 2n^{2}\bL^{2}\sum_{\ell=1}^{n-1}B_{\sigma_{k}^{\ell}}(\bz,\bx_{*})\overset{(b)}{\leq}2n^{2}\bL^{2}\sum_{\ell=1}^{n}B_{\sigma_{k}^{\ell}}(\bz,\bx_{*})=2n^{3}\bL^{2}B_{f}(\bz,\bx_{*}),\label{eq:smooth-res-4}
\end{align}
where $(a)$ is due to Lemma \ref{lem:cocore} and $(b)$ is by $B_{\sigma_{k}^{\ell}}\geq0$.

Plugging (\ref{eq:smooth-res-3}) and (\ref{eq:smooth-res-4}) into
(\ref{eq:smooth-res-2}) to get
\[
\sum_{i=1}^{n}L_{\sigma_{k}^{i}}\left\Vert \bx_{k}^{i}-\bx_{k}\right\Vert ^{2}\leq\sum_{i=1}^{n}2\eta_{k}^{2}n^{2}\bL\left\Vert \nabla f_{\sigma_{k}^{i}}(\bx_{k}^{i})-\nabla f_{\sigma_{k}^{i}}(\bz)\right\Vert ^{2}+8\eta_{k}^{2}n^{3}\bL^{2}B_{f}(\bz,\bx_{*})+4\eta_{k}^{2}\sum_{i=2}^{n}L_{\sigma_{k}^{i}}\left\Vert \sum_{j=1}^{i-1}\nabla f_{\sigma_{k}^{j}}(\bx_{*})\right\Vert ^{2}.
\]
Combing (\ref{eq:smooth-res-1}), we obtain
\begin{align*}
\frac{1}{n}\sum_{i=1}^{n}B_{\sigma_{k}^{i}}(\bx_{k+1},\bx_{k}^{i})-B_{\sigma_{k}^{i}}(\bz,\bx_{k}^{i})\leq & \bL\left\Vert \bx_{k+1}-\bx_{k}\right\Vert ^{2}+8\eta_{k}^{2}n^{2}\bL^{2}B_{f}(\bz,\bx_{*})+4\eta_{k}^{2}\sum_{i=2}^{n}\frac{L_{\sigma_{k}^{i}}}{n}\left\Vert \sum_{j=1}^{i-1}\nabla f_{\sigma_{k}^{j}}(\bx_{*})\right\Vert ^{2}\\
 & +\frac{1}{n}\sum_{i=1}^{n}\left(2\eta_{k}^{2}n^{2}\bL-\frac{1}{2L_{\sigma_{k}^{i}}}\right)\left\Vert \nabla f_{\sigma_{k}^{i}}(\bx_{k}^{i})-\nabla f_{\sigma_{k}^{i}}(\bz)\right\Vert ^{2}\\
\leq & \bL\left\Vert \bx_{k+1}-\bx_{k}\right\Vert ^{2}+8\eta_{k}^{2}n^{2}\bL^{2}B_{f}(\bz,\bx_{*})+4\eta_{k}^{2}\sum_{i=2}^{n}\frac{L_{\sigma_{k}^{i}}}{n}\left\Vert \sum_{j=1}^{i-1}\nabla f_{\sigma_{k}^{j}}(\bx_{*})\right\Vert ^{2},
\end{align*}
where the last inequality is due to $2\eta_{k}^{2}n^{2}\bL\leq\frac{1}{2L^{*}}\leq\frac{1}{2L_{\sigma_{k}^{i}}},\forall i\in\left[n\right]$
from the condition $\eta_{k}\leq\frac{1}{2n\sqrt{\bL L^{*}}},\forall k\in\left[K\right]$.
\end{proof}

Next, let us consider the Lipschitz case, i.e., $\left\Vert \nabla f_{i}(\bx)\right\Vert \leq G_{i},\forall\bx\in\R^{d},i\in\left[n\right]$.
Unlike Lemma \ref{lem:smooth-res}, the following lemma shows that
the extra term $\frac{1}{n}\sum_{i=1}^{n}B_{\sigma_{k}^{i}}(\bx_{k+1},\bx_{k}^{i})-B_{\sigma_{k}^{i}}(\bz,\bx_{k}^{i})$
can be always bounded by $\bG,n,\eta_{k}$ and $\left\Vert \bx_{k+1}-\bx_{k}\right\Vert $
now regardless of what $\bz$ is.
\begin{lem}
\label{lem:lip-res}Under Assumptions \ref{assu:basic} and \ref{assu:lip},
for any $k\in\left[K\right]$, permutation $\sigma_{k}$ and $\bz\in\R^{d}$,
Algorithm \ref{alg:shuffling-SGD} guarantees 
\[
\frac{1}{n}\sum_{i=1}^{n}B_{\sigma_{k}^{i}}(\bx_{k+1},\bx_{k}^{i})-B_{\sigma_{k}^{i}}(\bz,\bx_{k}^{i})\leq2\bG\left\Vert \bx_{k+1}-\bx_{k}\right\Vert +\bG^{2}n\eta_{k}.
\]
\end{lem}

\begin{proof}
By noticing $B_{\sigma_{k}^{i}}\geq0$ because of the covnexity of
$f_{\sigma_{k}^{i}}(\bx)$ from Assumption \ref{assu:basic}, we have
\begin{equation}
\frac{1}{n}\sum_{i=1}^{n}B_{\sigma_{k}^{i}}(\bx_{k+1},\bx_{k}^{i})-B_{\sigma_{k}^{i}}(\bz,\bx_{k}^{i})\leq\frac{1}{n}\sum_{i=1}^{n}B_{\sigma_{k}^{i}}(\bx_{k+1},\bx_{k}^{i}).\label{eq:lip-res-1}
\end{equation}
Under Assumption \ref{assu:lip}
\begin{align*}
B_{\sigma_{k}^{i}}(\bx_{k+1},\bx_{k}^{i}) & =f_{\sigma_{k}^{i}}(\bx_{k+1})-f_{\sigma_{k}^{i}}(\bx_{k}^{i})-\langle\nabla f_{\sigma_{k}^{i}}(\bx_{k}^{i}),\bx_{k+1}-\bx_{k}^{i}\rangle\\
 & \overset{(a)}{\leq}\langle\nabla f_{\sigma_{k}^{i}}(\bx_{k+1}),\bx_{k+1}-\bx_{k}^{i}\rangle-\langle\nabla f_{\sigma_{k}^{i}}(\bx_{k}^{i}),\bx_{k+1}-\bx_{k}^{i}\rangle\\
 & \leq\left(\left\Vert \nabla f_{\sigma_{k}^{i}}(\bx_{k+1})\right\Vert +\left\Vert \nabla f_{\sigma_{k}^{i}}(\bx_{k}^{i})\right\Vert \right)\left\Vert \bx_{k+1}-\bx_{k}^{i}\right\Vert \\
 & \leq2G_{\sigma_{k}^{i}}\left\Vert \bx_{k+1}-\bx_{k}^{i}\right\Vert \leq2G_{\sigma_{k}^{i}}\left(\left\Vert \bx_{k+1}-\bx_{k}\right\Vert +\left\Vert \bx_{k}^{i}-\bx_{k}\right\Vert \right),
\end{align*}
where $(a)$ is by the convexity of $f_{\sigma_{k}^{i}}(\bx)$ from
Assumption \ref{assu:basic}. Therefore
\begin{equation}
\frac{1}{n}\sum_{i=1}^{n}B_{\sigma_{k}^{i}}(\bx_{k+1},\bx_{k}^{i})\leq\frac{1}{n}\sum_{i=1}^{n}2G_{\sigma_{k}^{i}}\left(\left\Vert \bx_{k+1}-\bx_{k}\right\Vert +\left\Vert \bx_{k}^{i}-\bx_{k}\right\Vert \right)=2\bG\left\Vert \bx_{k+1}-\bx_{k}\right\Vert +\frac{2}{n}\sum_{i=1}^{n}G_{\sigma_{k}^{i}}\left\Vert \bx_{k}^{i}-\bx_{k}\right\Vert .\label{eq:lip-res-2}
\end{equation}
Note that $\bx_{k}=\bx_{k}^{1}$, hence
\begin{align}
\frac{2}{n}\sum_{i=1}^{n}G_{\sigma_{k}^{i}}\left\Vert \bx_{k}^{i}-\bx_{k}\right\Vert  & =\frac{2}{n}\sum_{i=2}^{n}G_{\sigma_{k}^{i}}\left\Vert \bx_{k}^{i}-\bx_{k}^{1}\right\Vert \leq\frac{2}{n}\sum_{i=2}^{n}\sum_{j=1}^{i-1}G_{\sigma_{k}^{i}}\left\Vert \bx_{k}^{j+1}-\bx_{k}^{j}\right\Vert =\frac{2\eta_{k}}{n}\sum_{i=2}^{n}\sum_{j=1}^{i-1}G_{\sigma_{k}^{i}}\left\Vert \nabla f_{\sigma_{k}^{j}}(\bx_{k}^{j})\right\Vert \nonumber \\
 & \leq\frac{2\eta_{k}}{n}\sum_{i=2}^{n}\sum_{j=1}^{i-1}G_{\sigma_{k}^{i}}G_{\sigma_{k}^{j}}=\frac{\eta_{k}}{n}\left[\left(\sum_{i=1}^{n}G_{\sigma_{k}^{i}}\right)^{2}-\sum_{i=1}^{n}G_{\sigma_{k}^{i}}^{2}\right]\leq\bG^{2}(n-1)\eta_{k}\leq\bG^{2}n\eta_{k}.\label{eq:lip-res-3}
\end{align}
Combining (\ref{eq:lip-res-1}), (\ref{eq:lip-res-2}) and (\ref{eq:lip-res-3}),
we obtain
\[
\frac{1}{n}\sum_{i=1}^{n}B_{\sigma_{k}^{i}}(\bx_{k+1},\bx_{k}^{i})-B_{\sigma_{k}^{i}}(\bz,\bx_{k}^{i})\leq2\bG\left\Vert \bx_{k+1}-\bx_{k}\right\Vert +\bG^{2}n\eta_{k}.
\]
\end{proof}

\subsection{Smooth Functions}

In this subsection, we deal with smooth functions and will prove two
important results.

First, we present the most important result for the smooth case, Lemma
\ref{lem:smooth-core}. The main difficulty is to deal with the extra
term $B_{f}(\bz,\bx_{*})$ after using Lemmas \ref{lem:core} and
\ref{lem:smooth-res}. Suppose we follow the same way used in the
previous works, i.e., setting $\bz=\bz_{k}$ for a carefully designed
sequence $\left\{ \bz_{k},\forall k\in\left[K\right]\right\} $. We
can only bound $F(\bx_{K+1})-F(\bx_{*})$ by $\O(\sum_{k=1}^{K}B_{f}(\bz_{k},\bx_{*}))$.

A key observation is that $B_{f}(\bz_{k},\bx_{*})=\O(\sum_{\ell=1}^{k}B_{f}(\bx_{\ell},\bx_{*}))$
because $\bz_{k}$ will be taken as a convex combination of $\bx_{*},\bx_{1},\cdots,\bx_{k}$.
Thus, we can bound $F(\bx_{K+1})-F(\bx_{*})$ by $\O(\sum_{k=1}^{K}B_{f}(\bx_{k},\bx_{*}))$.
However, this is not enough since we still need the bound on $B_{f}(\bx_{k},\bx_{*})$
for every $k\in\left[K\right]$. Temporarily assume $\psi(\bx)=0$,
we can find there is $B_{f}(\bx_{k},\bx_{*})=F(\bx_{k})-F(\bx_{*})$.
So if $F(\bx_{k})-F(\bx_{*})$ are small enough for every $k\in\left[K\right]$,
we can hope that $F(\bx_{K+1})-F(\bx_{*})$ is also small. This thought
inspires us to bound $F(\bx_{k+1})-F(\bx_{*})$ for every $k\in\left[K\right]$
instead of only bounding $F(\bx_{K+1})-F(\bx_{*})$. Hence, departing
from the existing works that only bound the function value gap once
for time $K$, we prove the following anytime inequality, which can
finally help us prove the last-iterate rate (see the proof of Theorem
\ref{thm:smooth-cvx-full}). In addition, we would like to emphasize
that the sequence $\left\{ \bz_{\ell},\forall\ell\in\left[k\right]\right\} $
now is defined differently for every $k\in\left[K\right]$ as mentioned
in Section \ref{sec:idea}.
\begin{lem}
\label{lem:smooth-core}Under Assumptions \ref{assu:minimizer}, \ref{assu:basic}
and \ref{assu:smooth}, suppose $\eta_{k}\leq\frac{1}{2n\sqrt{\bL L^{*}}},\forall k\in\left[K\right]$,
then for any permutation $\sigma_{k},\forall k\in\left[K\right]$,
Algorithm \ref{alg:shuffling-SGD} guarantees
\[
F(\bx_{k+1})-F(\bx_{*})\leq\frac{\left\Vert \bx_{*}-\bx_{1}\right\Vert ^{2}}{2n\sum_{\ell=1}^{k}\eta_{\ell}}+\sum_{\ell=1}^{k}\frac{4\eta_{\ell}^{3}R_{\ell}}{\sum_{s=\ell}^{k}\eta_{s}}+\sum_{\ell=2}^{k}\frac{8n^{2}\bL^{2}\eta_{\ell-1}\left(\sum_{s=\ell}^{k}\eta_{s}^{3}\right)}{\left(\sum_{s=\ell}^{k}\eta_{s}\right)\left(\sum_{s=\ell-1}^{k}\eta_{s}\right)}B_{f}(\bx_{\ell},\bx_{*}),\forall k\in\left[K\right],
\]
where $R_{\ell}$ is defined in Lemma \ref{lem:smooth-res}.
\end{lem}

\begin{proof}
Fix $k\in\left[K\right]$, we define the non-decreasing sequence
\begin{align}
v_{s} & \triangleq\frac{\eta_{k}}{\sum_{\ell=s}^{k}\eta_{\ell}},\forall s\in\left[k\right],\label{eq:smooth-core-v}\\
v_{0} & \triangleq v_{1},\label{eq:smooth-core-v0}
\end{align}
and the auxiliary points $\bz_{0}\triangleq\bx_{*}$ and 
\begin{equation}
\bz_{s}\triangleq\left(1-\frac{v_{s-1}}{v_{s}}\right)\bx_{s}+\frac{v_{s-1}}{v_{s}}\bz_{s-1},\forall s\in\left[k\right].\label{eq:smooth-core-z}
\end{equation}
Equivalently, we can write $\bz_{s}$ as
\begin{equation}
\bz_{s}\triangleq\frac{v_{0}}{v_{s}}\bx_{*}+\sum_{\ell=1}^{s}\frac{v_{\ell}-v_{\ell-1}}{v_{s}}\bx_{\ell},\forall s\in\left\{ 0\right\} \cup\left[k\right].\label{eq:smooth-core-z-equivalent}
\end{equation}
Note that $\bz_{s}\in\mathrm{dom}\psi,\forall s\in\left\{ 0\right\} \cup\left[k\right]$
since it is a convex combination of $\bx_{*},\bx_{1},\cdots\bx_{s}$
due to $v_{s}\geq v_{s-1}$ and all of which fall into $\mathrm{dom}\psi$.

We invoke Lemma \ref{lem:core} with $k=s$ (this $k$ is for $k$
in Lemma \ref{lem:core} and is not our current fixed $k$) and $\bz=\bz_{s}$
to obtain
\begin{align*}
F(\bx_{s+1})-F(\bz_{s}) & \leq\frac{\left\Vert \bz_{s}-\bx_{s}\right\Vert ^{2}}{2n\eta_{s}}-\left(\frac{1}{\eta_{s}}+n\mu_{\psi}\right)\frac{\left\Vert \bz_{s}-\bx_{s+1}\right\Vert ^{2}}{2n}-\frac{\left\Vert \bx_{s+1}-\bx_{s}\right\Vert ^{2}}{2n\eta_{s}}+\frac{1}{n}\sum_{i=1}^{n}B_{\sigma_{s}^{i}}(\bx_{s+1},\bx_{s}^{i})-B_{\sigma_{s}^{i}}(\bz_{s},\bx_{s}^{i})\\
 & \overset{(a)}{\leq}\frac{\left\Vert \bz_{s}-\bx_{s}\right\Vert ^{2}}{2n\eta_{s}}-\frac{\left\Vert \bz_{s}-\bx_{s+1}\right\Vert ^{2}}{2n\eta_{s}}-\frac{\left\Vert \bx_{s+1}-\bx_{s}\right\Vert ^{2}}{2n\eta_{s}}+\frac{1}{n}\sum_{i=1}^{n}B_{\sigma_{s}^{i}}(\bx_{s+1},\bx_{s}^{i})-B_{\sigma_{s}^{i}}(\bz_{s},\bx_{s}^{i})\\
 & \overset{(b)}{\leq}\frac{\left\Vert \bz_{s}-\bx_{s}\right\Vert ^{2}}{2n\eta_{s}}-\frac{\left\Vert \bz_{s}-\bx_{s+1}\right\Vert ^{2}}{2n\eta_{s}}-\frac{\left\Vert \bx_{s+1}-\bx_{s}\right\Vert ^{2}}{2n\eta_{s}}+\bL\left\Vert \bx_{s+1}-\bx_{s}\right\Vert ^{2}+8\eta_{s}^{2}n^{2}\bL^{2}B_{f}(\bz_{s},\bx_{*})+4\eta_{s}^{2}R_{s}\\
 & \overset{(c)}{\leq}\frac{\left\Vert \bz_{s}-\bx_{s}\right\Vert ^{2}}{2n\eta_{s}}-\frac{\left\Vert \bz_{s}-\bx_{s+1}\right\Vert ^{2}}{2n\eta_{s}}+8\eta_{s}^{2}n^{2}\bL^{2}B_{f}(\bz_{s},\bx_{*})+4\eta_{s}^{2}R_{s},
\end{align*}
where $(a)$ is due to $\mu_{\psi}\geq0$, $(b)$ is by Lemma \ref{lem:smooth-res}
and $(c)$ holds because of $\bL\leq\frac{1}{2n\eta_{s}}$ from the
requirment of $\eta_{k}\leq\frac{1}{2n\sqrt{\bL L^{*}}},\forall k\in\left[K\right]$.
Note that
\[
\left\Vert \bz_{s}-\bx_{s}\right\Vert ^{2}\overset{\eqref{eq:smooth-core-z}}{=}\frac{v_{s-1}^{2}}{v_{s}^{2}}\left\Vert \bz_{s-1}-\bx_{s}\right\Vert ^{2}\leq\frac{v_{s-1}}{v_{s}}\left\Vert \bz_{s-1}-\bx_{s}\right\Vert ^{2},
\]
where the last inequality is due to $v_{s-1}\leq v_{s}$. Hence
\begin{align}
F(\bx_{s+1})-F(\bz_{s}) & \leq\frac{\frac{v_{s-1}}{v_{s}}\left\Vert \bz_{s-1}-\bx_{s}\right\Vert ^{2}}{2n\eta_{s}}-\frac{\left\Vert \bz_{s}-\bx_{s+1}\right\Vert ^{2}}{2n\eta_{s}}+8\eta_{s}^{2}n^{2}\bL^{2}B_{f}(\bz_{s},\bx_{*})+4\eta_{s}^{2}R_{s}\nonumber \\
\Rightarrow\eta_{s}v_{s}\left(F(\bx_{s+1})-F(\bz_{s})\right) & \leq\frac{v_{s-1}\left\Vert \bz_{s-1}-\bx_{s}\right\Vert ^{2}}{2n}-\frac{v_{s}\left\Vert \bz_{s}-\bx_{s+1}\right\Vert ^{2}}{2n}+\underbrace{8\eta_{s}^{3}v_{s}n^{2}\bL^{2}B_{f}(\bz_{s},\bx_{*})+4\eta_{s}^{3}v_{s}R_{s}}_{\triangleq Q_{s}},\label{eq:smooth-core-1}
\end{align}
Summing (\ref{eq:smooth-core-1}) from $s=1$ to $k$ to obtain
\begin{equation}
\sum_{s=1}^{k}\eta_{s}v_{s}\left(F(\bx_{s+1})-F(\bz_{s})\right)\leq\frac{v_{0}\left\Vert \bz_{0}-\bx_{1}\right\Vert ^{2}}{2n}-\frac{v_{k}\left\Vert \bz_{k}-\bx_{k+1}\right\Vert ^{2}}{2n}+\sum_{s=1}^{k}Q_{s}\leq\frac{v_{0}\left\Vert \bx_{*}-\bx_{1}\right\Vert ^{2}}{2n}+\sum_{s=1}^{k}Q_{s},\label{eq:smooth-core-2}
\end{equation}
where the last line is by $\bz_{0}=\bx_{*}$ and $\left\Vert \bz_{k}-\bx_{k+1}\right\Vert ^{2}\geq0$. 

By the convexity of $F(\bx)$ and (\ref{eq:smooth-core-z-equivalent}),
we can bound
\[
F(\bz_{s})\leq\frac{v_{0}}{v_{s}}F(\bx_{*})+\sum_{\ell=1}^{s}\frac{v_{\ell}-v_{\ell-1}}{v_{s}}F(\bx_{\ell})=F(\bx_{*})+\sum_{\ell=1}^{s}\frac{v_{\ell}-v_{\ell-1}}{v_{s}}\left(F(\bx_{\ell})-F(\bx_{*})\right),
\]
which implies
\begin{align*}
\sum_{s=1}^{k}\eta_{s}v_{s}\left(F(\bx_{s+1})-F(\bz_{s})\right)\geq & \sum_{s=1}^{k}\eta_{s}\left(v_{s}\left(F(\bx_{s+1})-F(\bx_{*})\right)-\sum_{\ell=1}^{s}(v_{\ell}-v_{\ell-1})\left(F(\bx_{\ell})-F(\bx_{*})\right)\right)\\
= & \sum_{s=1}^{k}\eta_{s}v_{s}\left(F(\bx_{s+1})-F(\bx_{*})\right)-\sum_{s=1}^{k}\sum_{\ell=1}^{s}\eta_{s}(v_{\ell}-v_{\ell-1})\left(F(\bx_{\ell})-F(\bx_{*})\right)\\
= & \sum_{s=1}^{k}\eta_{s}v_{s}\left(F(\bx_{s+1})-F(\bx_{*})\right)-\sum_{s=1}^{k}\left(\sum_{\ell=s}^{k}\eta_{\ell}\right)(v_{s}-v_{s-1})\left(F(\bx_{s})-F(\bx_{*})\right)\\
= & \eta_{k}v_{k}\left(F(\bx_{k+1})-F(\bx_{*})\right)-\left(\sum_{\ell=1}^{k}\eta_{\ell}\right)\left(v_{1}-v_{0}\right)\left(F(\bx_{1})-F(\bx_{*})\right)\\
 & +\sum_{s=2}^{k}\left[\eta_{s-1}v_{s-1}-\left(\sum_{\ell=s}^{k}\eta_{\ell}\right)(v_{s}-v_{s-1})\right]\left(F(\bx_{s})-F(\bx_{*})\right).
\end{align*}
Note that $v_{1}\overset{\eqref{eq:smooth-core-v0}}{=}v_{0}$ and
for $2\leq s\leq k$
\[
\eta_{s-1}v_{s-1}-\left(\sum_{\ell=s}^{k}\eta_{\ell}\right)(v_{s}-v_{s-1})=\left(\sum_{\ell=s-1}^{k}\eta_{\ell}\right)v_{s-1}-\left(\sum_{\ell=s}^{k}\eta_{\ell}\right)v_{s}\overset{\eqref{eq:smooth-core-v}}{=}\eta_{k}-\eta_{k}=0.
\]
Hence, we know
\begin{equation}
\sum_{s=1}^{k}\eta_{s}v_{s}\left(F(\bx_{s+1})-F(\bz_{s})\right)\geq\eta_{k}v_{k}\left(F(\bx_{k+1})-F(\bx_{*})\right).\label{eq:smooth-core-3}
\end{equation}

Plugging (\ref{eq:smooth-core-3}) into (\ref{eq:smooth-core-2}),
we obtain
\begin{align}
\eta_{k}v_{k}\left(F(\bx_{k+1})-F(\bx_{*})\right) & \leq\frac{v_{0}\left\Vert \bx_{*}-\bx_{1}\right\Vert ^{2}}{2n}+\sum_{s=1}^{k}Q_{s}\nonumber \\
\Rightarrow F(\bx_{k+1})-F(\bx_{*}) & \leq\frac{v_{0}}{\eta_{k}v_{k}}\cdot\frac{\left\Vert \bx_{*}-\bx_{1}\right\Vert ^{2}}{2n}+\frac{1}{\eta_{k}v_{k}}\sum_{s=1}^{k}Q_{s}\overset{\eqref{eq:smooth-core-v},\eqref{eq:smooth-core-v0}}{=}\frac{\left\Vert \bx-\bx_{1}\right\Vert ^{2}}{2n\sum_{\ell=1}^{k}\eta_{\ell}}+\frac{1}{\eta_{k}}\sum_{s=1}^{k}Q_{s}\nonumber \\
 & =\frac{\left\Vert \bx_{*}-\bx_{1}\right\Vert ^{2}}{2n\sum_{\ell=1}^{k}\eta_{\ell}}+\underbrace{\frac{1}{\eta_{k}}\sum_{s=1}^{k}4\eta_{s}^{3}v_{s}R_{s}}_{(i)}+\underbrace{\frac{1}{\eta_{k}}\sum_{s=1}^{k}8\eta_{s}^{3}v_{s}n^{2}\bL^{2}B_{f}(\bz_{s},\bx_{*})}_{(ii)}.\label{eq:smooth-core-4}
\end{align}
For term $(i)$, we have
\begin{equation}
(i)\overset{\eqref{eq:smooth-core-v}}{=}\sum_{s=1}^{k}\frac{4\eta_{s}^{3}R_{s}}{\sum_{\ell=s}^{k}\eta_{\ell}}=\sum_{\ell=1}^{k}\frac{4\eta_{\ell}^{3}R_{\ell}}{\sum_{s=\ell}^{k}\eta_{s}}.\label{eq:smooth-core-i}
\end{equation}
For term $(ii)$, by the convexity of the first argument of $B_{f}(\cdot,\cdot)$
(which is implied by the convexity of $f(\bx)$) and (\ref{eq:smooth-core-z-equivalent}),
we first bound
\[
B_{f}(\bz_{s},\bx_{*})\leq\frac{v_{0}}{v_{s}}B_{f}(\bx_{*},\bx_{*})+\sum_{\ell=1}^{s}\frac{v_{\ell}-v_{\ell-1}}{v_{s}}B_{f}(\bx_{\ell},\bx_{*})=\sum_{\ell=1}^{s}\frac{v_{\ell}-v_{\ell-1}}{v_{s}}B_{f}(\bx_{\ell},\bx_{*}),
\]
which implies
\begin{align}
(ii) & \leq\frac{1}{\eta_{k}}\sum_{s=1}^{k}\sum_{\ell=1}^{s}8\eta_{s}^{3}n^{2}\bL^{2}(v_{\ell}-v_{\ell-1})B_{f}(\bx_{\ell},\bx_{*})\nonumber \\
 & \overset{\eqref{eq:smooth-core-v0}}{=}\frac{1}{\eta_{k}}\sum_{s=2}^{k}\sum_{\ell=2}^{s}8\eta_{s}^{3}n^{2}\bL^{2}(v_{\ell}-v_{\ell-1})B_{f}(\bx_{\ell},\bx_{*})\nonumber \\
 & =\frac{1}{\eta_{k}}\sum_{\ell=2}^{k}8n^{2}\bL^{2}\left(\sum_{s=\ell}^{k}\eta_{s}^{3}\right)(v_{\ell}-v_{\ell-1})B_{f}(\bx_{\ell},\bx_{*})\nonumber \\
 & \overset{\eqref{eq:smooth-core-v}}{=}\frac{1}{\eta_{k}}\sum_{\ell=2}^{k}8n^{2}\bL^{2}\left(\sum_{s=\ell}^{k}\eta_{s}^{3}\right)\left(\frac{\eta_{k}}{\sum_{s=\ell}^{k}\eta_{s}}-\frac{\eta_{k}}{\sum_{s=\ell-1}^{k}\eta_{s}}\right)B_{f}(\bx_{\ell},\bx_{*})\nonumber \\
 & =\sum_{\ell=2}^{k}\frac{8n^{2}\bL^{2}\eta_{\ell-1}\left(\sum_{s=\ell}^{k}\eta_{s}^{3}\right)}{\left(\sum_{s=\ell}^{k}\eta_{s}\right)\left(\sum_{s=\ell-1}^{k}\eta_{s}\right)}B_{f}(\bx_{\ell},\bx_{*}).\label{eq:smooth-core-ii}
\end{align}
Plugging (\ref{eq:smooth-core-i}) and (\ref{eq:smooth-core-ii})
into (\ref{eq:smooth-core-4}), we finally obtain
\[
F(\bx_{k+1})-F(\bx_{*})\leq\frac{\left\Vert \bx_{*}-\bx_{1}\right\Vert ^{2}}{2n\sum_{\ell=1}^{k}\eta_{\ell}}+\sum_{\ell=1}^{k}\frac{4\eta_{\ell}^{3}R_{\ell}}{\sum_{s=\ell}^{k}\eta_{s}}+\sum_{\ell=2}^{k}\frac{8n^{2}\bL^{2}\eta_{\ell-1}\left(\sum_{s=\ell}^{k}\eta_{s}^{3}\right)}{\left(\sum_{s=\ell}^{k}\eta_{s}\right)\left(\sum_{s=\ell-1}^{k}\eta_{s}\right)}B_{f}(\bx_{\ell},\bx_{*}).
\]
\end{proof}

The second result, Lemma \ref{lem:smooth-distance}, is particularly
useful for the strongly convex case, i.e., $\mu_{\psi}>0$ or $\mu_{f}>0$.
Though the distance from the optimum is not the desired bound we want,
it can finally help us to bound the function value. The key idea here
is using $\left\Vert \bx_{*}-\bx_{\left\lceil \frac{K}{2}\right\rceil +1}\right\Vert ^{2}$
to bound $F(\bx_{K+1})-F(\bx_{*})$ and applying Lemma \ref{lem:smooth-distance}
to bound $\left\Vert \bx_{*}-\bx_{\left\lceil \frac{K}{2}\right\rceil +1}\right\Vert ^{2}$.
A similar argument showed up in \cite{pmlr-v202-cha23a} but was used
to bound $F(\bx_{K+1}^{\mathrm{tail}})-F(\bx_{*})$ (recall $\bx_{K+1}^{\mathrm{tail}}\triangleq\frac{1}{\left\lfloor \frac{K}{2}\right\rfloor +1}\sum_{k=\left\lceil \frac{K}{2}\right\rceil }^{K}\bx_{k+1}$).
In contrast, our work directly bounds the last iterate instead of
the tail average iterate.
\begin{lem}
\label{lem:smooth-distance}Under Assumptions \ref{assu:minimizer},
\ref{assu:basic}, \ref{assu:smooth} and \ref{assu:str}, suppose
$\eta_{k}\leq\frac{1}{2n\sqrt{\bL L^{*}}},\forall k\in\left[K\right]$,
then for any permutation $\sigma_{k},\forall k\in\left[K\right]$,
Algorithm \ref{alg:shuffling-SGD} guarantees
\[
\left\Vert \bx_{k+1}-\bx_{*}\right\Vert ^{2}\leq\frac{\left\Vert \bx_{*}-\bx_{1}\right\Vert ^{2}}{\prod_{s=1}^{k}(1+n\eta_{s}(\mu_{f}+2\mu_{\psi}))}+\sum_{\ell=1}^{k}\frac{8n\eta_{\ell}^{3}R_{\ell}}{\prod_{s=\ell}^{k}(1+n\eta_{s}(\mu_{f}+2\mu_{\psi}))},\forall k\in\left[K\right],
\]
where $R_{\ell}$ is defined in Lemma \ref{lem:smooth-res}.
\end{lem}

\begin{proof}
We invoke Lemma \ref{lem:core} with $\bz=\bx_{*}$ to obtain
\begin{align*}
F(\bx_{k+1})-F(\bx_{*})\leq & \frac{\left\Vert \bx_{*}-\bx_{k}\right\Vert ^{2}}{2n\eta_{k}}-\left(\frac{1}{\eta_{k}}+n\mu_{\psi}\right)\frac{\left\Vert \bx_{*}-\bx_{k+1}\right\Vert ^{2}}{2n}-\frac{\left\Vert \bx_{k+1}-\bx_{k}\right\Vert ^{2}}{2n\eta_{k}}\\
 & +\frac{1}{n}\sum_{i=1}^{n}B_{\sigma_{k}^{i}}(\bx_{k+1},\bx_{k}^{i})-B_{\sigma_{k}^{i}}(\bx_{*},\bx_{k}^{i})\\
\overset{(a)}{\leq} & \frac{\left\Vert \bx_{*}-\bx_{k}\right\Vert ^{2}}{2n\eta_{k}}-\left(\frac{1}{\eta_{k}}+n\mu_{\psi}\right)\frac{\left\Vert \bx_{*}-\bx_{k+1}\right\Vert ^{2}}{2n}-\frac{\left\Vert \bx_{k+1}-\bx_{k}\right\Vert ^{2}}{2n\eta_{k}}\\
 & +\bL\left\Vert \bx_{k+1}-\bx_{k}\right\Vert ^{2}+8\eta_{k}^{2}n^{2}\bL^{2}B_{f}(\bx_{*},\bx_{*})+4\eta_{k}^{2}R_{k}\\
\overset{(b)}{\leq} & \frac{\left\Vert \bx_{*}-\bx_{k}\right\Vert ^{2}}{2n\eta_{k}}-\left(\frac{1}{\eta_{k}}+n\mu_{\psi}\right)\frac{\left\Vert \bx_{*}-\bx_{k+1}\right\Vert ^{2}}{2n}+4\eta_{k}^{2}R_{k},
\end{align*}
where $(a)$ is by Lemma \ref{lem:smooth-res} and $(b)$ holds due
to $\bL\leq\frac{1}{2n\eta_{k}}$ from the requirment of $\eta_{k}\leq\frac{1}{2n\sqrt{\bL L^{*}}},\forall k\in\left[K\right]$
and $B_{f}(\bx_{*},\bx_{*})=0$. Note that $\bx_{*}\in\argmin_{\bx\in\R^{d}}F(\bx)$
implies $\exists\nabla\psi(\bx_{*})\in\partial\psi(\bx_{*})$ such
that $\nabla f(\bx_{*})+\nabla\psi(\bx_{*})=\mathbf{0}$, hence
\begin{align*}
F(\bx_{k+1})-F(\bx_{*}) & =F(\bx_{k+1})-F(\bx_{*})-\langle\nabla f(\bx_{*})+\nabla\psi(\bx_{*}),\bx_{k+1}-\bx_{*}\rangle\\
 & =B_{f}(\bx_{k+1},\bx_{*})+B_{\psi}(\bx_{k+1},\bx_{*})\geq\frac{\mu_{f}+\mu_{\psi}}{2}\left\Vert \bx_{*}-\bx_{k+1}\right\Vert ^{2},
\end{align*}
where the last inequality is due to the strong convexity of $f$ and
$\psi$. So we know
\begin{align*}
\frac{\mu_{f}+\mu_{\psi}}{2}\left\Vert \bx_{*}-\bx_{k+1}\right\Vert ^{2} & \leq\frac{\left\Vert \bx_{*}-\bx_{k}\right\Vert ^{2}}{2n\eta_{k}}-\left(\frac{1}{\eta_{k}}+n\mu_{\psi}\right)\frac{\left\Vert \bx_{*}-\bx_{k+1}\right\Vert ^{2}}{2n}+4\eta_{k}^{2}R_{k}\\
\Rightarrow(1+n\eta_{k}(\mu_{f}+2\mu_{\psi}))\left\Vert \bx_{*}-\bx_{k+1}\right\Vert ^{2} & \leq\left\Vert \bx_{*}-\bx_{k}\right\Vert ^{2}+8n\eta_{k}^{3}R_{k}\\
\Rightarrow\left\Vert \bx_{*}-\bx_{k+1}\right\Vert ^{2} & \leq\frac{\left\Vert \bx_{*}-\bx_{1}\right\Vert ^{2}}{\prod_{s=1}^{k}(1+n\eta_{s}(\mu_{f}+2\mu_{\psi}))}+\frac{\sum_{\ell=1}^{k}8n\eta_{\ell}^{3}R_{\ell}\prod_{s=1}^{\ell-1}(1+n\eta_{s}(\mu_{f}+2\mu_{\psi}))}{\prod_{s=1}^{k}(1+n\eta_{s}(\mu_{f}+2\mu_{\psi}))}\\
 & =\frac{\left\Vert \bx_{*}-\bx_{1}\right\Vert ^{2}}{\prod_{s=1}^{k}(1+n\eta_{s}(\mu_{f}+2\mu_{\psi}))}+\sum_{\ell=1}^{k}\frac{8n\eta_{\ell}^{3}R_{\ell}}{\prod_{s=\ell}^{k}(1+n\eta_{s}(\mu_{f}+2\mu_{\psi}))}.
\end{align*}
\end{proof}

\subsection{Lipschitz Functions}

We focus on the Lipschitz case and prove two useful results in this
subsection.

We first introduce Lemma \ref{lem:lip-core}, which is a consequence
of combining Lemmas \ref{lem:core} and \ref{lem:lip-res}. However,
the proof of Lemma \ref{lem:lip-core} is different from both the
prior works \cite{zamani2023exact,liu2024revisiting} and Lemma \ref{lem:smooth-core}.
For any fixed $k\in\left[K\right]$ here, instead of setting $\bz=\bz_{k}$
directly, we will invoke Lemma \ref{lem:core} $k+1$ times with $\bz=\bx_{\ell}$
for $\ell\in\left[k\right]$ and $\bz=\bx_{*}$ and then sum them
up with some weights. The key reason is that we need a refined bound
for proving Theorem \ref{thm:lip-ada-full}. To be more precise, one
can see there is an extra negative term $-\left\Vert \bx_{*}-\bx_{K+1}\right\Vert ^{2}$
showing up in Lemma \ref{lem:lip-core}. It plays an important role
in the proof of Theorem \ref{thm:lip-ada-full}. However, if we totally
follow \cite{zamani2023exact,liu2024revisiting}, the term $-\left\Vert \bx_{*}-\bx_{K+1}\right\Vert ^{2}$
will be replaced by a larger quantity, which cannot help us prove
Theorem \ref{thm:lip-ada-full}.
\begin{lem}
\label{lem:lip-core}Under Assumptions \ref{assu:minimizer}, \ref{assu:basic}
and \ref{assu:lip}, let $\gamma_{k}\triangleq\eta_{k}\prod_{\ell=2}^{k}(1+n\eta_{\ell-1}\mu_{\psi}),\forall k\in\left[K\right]$.
For any permutation $\sigma_{k},\forall k\in\left[K\right]$, Algorithm
\ref{alg:shuffling-SGD} guarantees
\[
F(\bx_{K+1})-F(\bx_{*})\leq\frac{\left\Vert \bx_{*}-\bx_{1}\right\Vert ^{2}-\left\Vert \bx_{*}-\bx_{K+1}\right\Vert ^{2}}{2n\sum_{k=1}^{K}\gamma_{k}}+3\bG^{2}n\sum_{k=1}^{K}\frac{\gamma_{k}\eta_{k}}{\sum_{\ell=k}^{K}\gamma_{\ell}}.
\]
\end{lem}

\begin{proof}
We define the non-decreasing sequence 
\begin{align}
v_{k} & \triangleq\frac{\gamma_{K}}{\sum_{\ell=k}^{K}\gamma_{\ell}},\forall k\in\left[K\right],\label{eq:lip-core-v}\\
v_{0} & \triangleq v_{1}.\label{eq:lip-core-v0}
\end{align}
Fix $k\in\left[K\right]$, we invoke Lemma \ref{lem:core} to obtain
for any $\bz\in\R^{d}$
\begin{align}
F(\bx_{k+1})-F(\bz)\leq & \frac{\left\Vert \bz-\bx_{k}\right\Vert ^{2}}{2n\eta_{k}}-\left(\frac{1}{\eta_{k}}+n\mu_{\psi}\right)\frac{\left\Vert \bz-\bx_{k+1}\right\Vert ^{2}}{2n}-\frac{\left\Vert \bx_{k+1}-\bx_{k}\right\Vert ^{2}}{2n\eta_{k}}\nonumber \\
 & +\frac{1}{n}\sum_{i=1}^{n}B_{\sigma_{k}^{i}}(\bx_{k+1},\bx_{k}^{i})-B_{\sigma_{k}^{i}}(\bz,\bx_{k}^{i})\nonumber \\
\overset{(a)}{\leq} & \frac{\left\Vert \bz-\bx_{k}\right\Vert ^{2}}{2n\eta_{k}}-\left(\frac{1}{\eta_{k}}+n\mu_{\psi}\right)\frac{\left\Vert \bz-\bx_{k+1}\right\Vert ^{2}}{2n}-\frac{\left\Vert \bx_{k+1}-\bx_{k}\right\Vert ^{2}}{2n\eta_{k}}\nonumber \\
 & +2\bG\left\Vert \bx_{k+1}-\bx_{k}\right\Vert +\bG^{2}n\eta_{k}\nonumber \\
\overset{(b)}{\leq} & \frac{\left\Vert \bz-\bx_{k}\right\Vert ^{2}}{2n\eta_{k}}-\left(\frac{1}{\eta_{k}}+n\mu_{\psi}\right)\frac{\left\Vert \bz-\bx_{k+1}\right\Vert ^{2}}{2n}+3\bG^{2}n\eta_{k}\nonumber \\
\Rightarrow\gamma_{k}\left(F(\bx_{k+1})-F(\bz)\right)\leq & \frac{\gamma_{k}\eta_{k}^{-1}\left\Vert \bz-\bx_{k}\right\Vert ^{2}}{2n}-\gamma_{k}\left(\frac{1}{\eta_{k}}+n\mu_{\psi}\right)\frac{\left\Vert \bz-\bx_{k+1}\right\Vert ^{2}}{2n}+3\bG^{2}n\gamma_{k}\eta_{k}\nonumber \\
\overset{(c)}{=} & \frac{\gamma_{k}\eta_{k}^{-1}\left\Vert \bz-\bx_{k}\right\Vert ^{2}}{2n}-\frac{\gamma_{k+1}\eta_{k+1}^{-1}\left\Vert \bz-\bx_{k+1}\right\Vert ^{2}}{2n}+3\bG^{2}n\gamma_{k}\eta_{k},\label{eq:lip-core-1}
\end{align}
where $(a)$ is due to Lemma \ref{lem:lip-res}, $(b)$ is by AM-GM
inequality and $(c)$ is because of 
\[
\gamma_{k}\left(\frac{1}{\eta_{k}}+n\mu_{\psi}\right)=\left(\frac{1}{\eta_{k}}+n\mu_{\psi}\right)\times\eta_{k}\prod_{\ell=2}^{k}\left(1+n\eta_{\ell-1}\mu_{\psi}\right)=\prod_{\ell=2}^{k+1}\left(1+n\eta_{\ell-1}\mu_{\psi}\right)=\gamma_{k+1}\eta_{k+1}^{-1}.
\]

Taking $\bz=\bx_{\ell}$ for $\ell\in\left[k\right]$ in (\ref{eq:lip-core-1})
and multiplying both sides by $v_{\ell}-v_{\ell-1}$ (it is non-negative
due to $v_{\ell}\geq v_{\ell-1}$) to obtain
\[
\gamma_{k}(v_{\ell}-v_{\ell-1})\left(F(\bx_{k+1})-F(\bx_{\ell})\right)\leq(v_{\ell}-v_{\ell-1})\left(\frac{\gamma_{k}\eta_{k}^{-1}\left\Vert \bx_{\ell}-\bx_{k}\right\Vert ^{2}}{2n}-\frac{\gamma_{k+1}\eta_{k+1}^{-1}\left\Vert \bx_{\ell}-\bx_{k+1}\right\Vert ^{2}}{2n}+3\bG^{2}n\gamma_{k}\eta_{k}\right).
\]
Summing it up from $\ell=1$ to $k$
\begin{align}
 & \sum_{\ell=1}^{k}\gamma_{k}(v_{\ell}-v_{\ell-1})\left(F(\bx_{k+1})-F(\bx_{\ell})\right)\nonumber \\
\leq & \frac{\gamma_{k}\eta_{k}^{-1}\sum_{\ell=1}^{k-1}(v_{\ell}-v_{\ell-1})\left\Vert \bx_{\ell}-\bx_{k}\right\Vert ^{2}}{2n}-\frac{\gamma_{k+1}\eta_{k+1}^{-1}\sum_{\ell=1}^{k}(v_{\ell}-v_{\ell-1})\left\Vert \bx_{\ell}-\bx_{k+1}\right\Vert ^{2}}{2n}+3\bG^{2}n\gamma_{k}\eta_{k}(v_{k}-v_{0}).\label{eq:lip-core-2}
\end{align}
Next, taking $\bz=\bx_{*}$ in (\ref{eq:lip-core-1}) and multiplying
both sides by $v_{0}$ to obtain
\begin{equation}
\gamma_{k}v_{0}\left(F(\bx_{k+1})-F(\bx_{*})\right)\leq\frac{\gamma_{k}\eta_{k}^{-1}v_{0}\left\Vert \bx_{*}-\bx_{k}\right\Vert ^{2}}{2n}-\frac{\gamma_{k+1}\eta_{k+1}^{-1}v_{0}\left\Vert \bx_{*}-\bx_{k+1}\right\Vert ^{2}}{2n}+3\bG^{2}n\gamma_{k}\eta_{k}v_{0}.\label{eq:lip-core-3}
\end{equation}
Adding (\ref{eq:lip-core-3}) to (\ref{eq:lip-core-2}) to get
\begin{align}
 & \sum_{\ell=1}^{k}\gamma_{k}(v_{\ell}-v_{\ell-1})\left(F(\bx_{k+1})-F(\bx_{\ell})\right)+\gamma_{k}v_{0}\left(F(\bx_{k+1})-F(\bx_{*})\right)\nonumber \\
\leq & \frac{\gamma_{k}\eta_{k}^{-1}v_{0}\left\Vert \bx_{*}-\bx_{k}\right\Vert ^{2}}{2n}-\frac{\gamma_{k+1}\eta_{k+1}^{-1}v_{0}\left\Vert \bx_{*}-\bx_{k+1}\right\Vert ^{2}}{2n}+3\bG^{2}n\gamma_{k}\eta_{k}v_{k}\nonumber \\
 & +\frac{\gamma_{k}\eta_{k}^{-1}\sum_{\ell=1}^{k-1}(v_{\ell}-v_{\ell-1})\left\Vert \bx_{\ell}-\bx_{k}\right\Vert ^{2}}{2n}-\frac{\gamma_{k+1}\eta_{k+1}^{-1}\sum_{\ell=1}^{k}(v_{\ell}-v_{\ell-1})\left\Vert \bx_{\ell}-\bx_{k+1}\right\Vert ^{2}}{2n}.\label{eq:lip-core-4}
\end{align}
Summing up (\ref{eq:lip-core-4}) from $k=1$ to $K$ to get
\begin{align}
 & \sum_{k=1}^{K}\sum_{\ell=1}^{k}\gamma_{k}(v_{\ell}-v_{\ell-1})\left(F(\bx_{k+1})-F(\bx_{\ell})\right)+\sum_{k=1}^{K}\gamma_{k}v_{0}\left(F(\bx_{k+1})-F(\bx_{*})\right)\nonumber \\
\leq & \sum_{k=1}^{K}\frac{\gamma_{k}\eta_{k}^{-1}v_{0}\left\Vert \bx_{*}-\bx_{k}\right\Vert ^{2}}{2n}-\frac{\gamma_{k+1}\eta_{k+1}^{-1}v_{0}\left\Vert \bx_{*}-\bx_{k+1}\right\Vert ^{2}}{2n}+\sum_{k=1}^{K}3\bG^{2}n\gamma_{k}\eta_{k}v_{k}\nonumber \\
 & +\sum_{k=1}^{K}\frac{\gamma_{k}\eta_{k}^{-1}\sum_{\ell=1}^{k-1}(v_{\ell}-v_{\ell-1})\left\Vert \bx_{\ell}-\bx_{k}\right\Vert ^{2}}{2n}-\frac{\gamma_{k+1}\eta_{k+1}^{-1}\sum_{\ell=1}^{k}(v_{\ell}-v_{\ell-1})\left\Vert \bx_{\ell}-\bx_{k+1}\right\Vert ^{2}}{2n}\nonumber \\
= & \frac{\gamma_{1}\eta_{1}^{-1}v_{0}\left\Vert \bx_{*}-\bx_{1}\right\Vert ^{2}}{2n}-\frac{\gamma_{K+1}\eta_{K+1}^{-1}v_{0}\left\Vert \bx_{*}-\bx_{K+1}\right\Vert ^{2}}{2n}+\sum_{k=1}^{K}3\bG^{2}n\gamma_{k}\eta_{k}v_{k}\nonumber \\
 & -\frac{\gamma_{K+1}\eta_{K+1}^{-1}\sum_{\ell=1}^{K}(v_{\ell}-v_{\ell-1})\left\Vert \bx_{\ell}-\bx_{K+1}\right\Vert ^{2}}{2n}\nonumber \\
\leq & \frac{v_{0}\left\Vert \bx_{*}-\bx_{1}\right\Vert ^{2}}{2n}-\frac{v_{0}\left\Vert \bx_{*}-\bx_{K+1}\right\Vert ^{2}}{2n}+\sum_{k=1}^{K}3\bG^{2}n\gamma_{k}\eta_{k}v_{k},\label{eq:lip-core-5}
\end{align}
where the last line is by $\gamma_{1}\eta_{1}^{-1}=1$, $\gamma_{K+1}\eta_{K+1}^{-1}=\prod_{\ell=2}^{K+1}(1+n\eta_{\ell-1}\mu_{\psi})\geq1$
and $(v_{\ell}-v_{\ell-1})\left\Vert \bx_{\ell}-\bx_{K+1}\right\Vert ^{2}\geq0$.

For the L.H.S. of (\ref{eq:lip-core-5}), we observe that
\begin{align*}
 & \sum_{k=1}^{K}\sum_{\ell=1}^{k}\gamma_{k}(v_{\ell}-v_{\ell-1})\left(F(\bx_{k+1})-F(\bx_{\ell})\right)\\
= & \sum_{k=1}^{K}\sum_{\ell=1}^{k}\gamma_{k}(v_{\ell}-v_{\ell-1})\left(F(\bx_{k+1})-F(\bx_{*})\right)-\sum_{k=1}^{K}\sum_{\ell=1}^{k}\gamma_{k}(v_{\ell}-v_{\ell-1})\left(F(\bx_{\ell})-F(\bx_{*})\right)\\
= & \sum_{k=1}^{K}\gamma_{k}\left(v_{k}-v_{0}\right)\left(F(\bx_{k+1})-F(\bx_{*})\right)-\sum_{\ell=1}^{K}\left(\sum_{k=\ell}^{K}\gamma_{k}\right)(v_{\ell}-v_{\ell-1})\left(F(\bx_{\ell})-F(\bx_{*})\right).
\end{align*}
Hence
\begin{align*}
 & \sum_{k=1}^{K}\sum_{\ell=1}^{k}\gamma_{k}(v_{\ell}-v_{\ell-1})\left(F(\bx_{k+1})-F(\bx_{\ell})\right)+\sum_{k=1}^{K}\gamma_{k}v_{0}\left(F(\bx_{k+1})-F(\bx_{*})\right)\\
= & \sum_{k=1}^{K}\gamma_{k}v_{k}\left(F(\bx_{k+1})-F(\bx_{*})\right)-\sum_{\ell=1}^{K}\left(\sum_{k=\ell}^{K}\gamma_{k}\right)(v_{\ell}-v_{\ell-1})\left(F(\bx_{\ell})-F(\bx_{*})\right)\\
= & \gamma_{K}v_{K}\left(F(\bx_{K+1})-F(\bx_{*})\right)-\left(\sum_{k=1}^{K}\gamma_{k}\right)\left(v_{1}-v_{0}\right)\left(F(\bx_{1})-F(\bx_{*})\right)\\
 & +\sum_{k=2}^{K}\left[\gamma_{k-1}v_{k-1}-\left(\sum_{\ell=k}^{K}\gamma_{\ell}\right)(v_{k}-v_{k-1})\right]\left(F(\bx_{k})-F(\bx_{*})\right).
\end{align*}
Note that $v_{1}\overset{\eqref{eq:lip-core-v0}}{=}v_{0}$ and for
$2\leq k\leq K$
\[
\gamma_{k-1}v_{k-1}-\left(\sum_{\ell=k}^{K}\gamma_{\ell}\right)(v_{k}-v_{k-1})=\left(\sum_{\ell=k-1}^{K}\gamma_{\ell}\right)v_{k-1}-\left(\sum_{\ell=k}^{K}\gamma_{\ell}\right)v_{k}\overset{\eqref{eq:lip-core-v}}{=}\gamma_{K}-\gamma_{K}=0.
\]
Thus, we know
\begin{equation}
\sum_{k=1}^{K}\sum_{\ell=1}^{k}\gamma_{k}(v_{\ell}-v_{\ell-1})\left(F(\bx_{k+1})-F(\bx_{\ell})\right)+\sum_{k=1}^{K}\gamma_{k}v_{0}\left(F(\bx_{k+1})-F(\bx_{*})\right)=\gamma_{K}v_{K}\left(F(\bx_{K+1})-F(\bx_{*})\right).\label{eq:lip-core-6}
\end{equation}

Plugging (\ref{eq:lip-core-6}) into (\ref{eq:lip-core-5}), we finally
obtain
\begin{align*}
\gamma_{K}v_{K}\left(F(\bx_{K+1})-F(\bx_{*})\right) & \leq\frac{v_{0}\left\Vert \bx_{*}-\bx_{1}\right\Vert ^{2}}{2n}-\frac{v_{0}\left\Vert \bx_{*}-\bx_{K+1}\right\Vert ^{2}}{2n}+\sum_{k=1}^{K}3\bG^{2}n\gamma_{k}\eta_{k}v_{k}\\
\Rightarrow F(\bx_{K+1})-F(\bx_{*}) & \leq\frac{v_{0}}{\gamma_{K}v_{K}}\cdot\frac{\left\Vert \bx_{*}-\bx_{1}\right\Vert ^{2}-\left\Vert \bx_{*}-\bx_{K+1}\right\Vert ^{2}}{2n}+3\bG^{2}n\sum_{k=1}^{K}\frac{\gamma_{k}\eta_{k}v_{k}}{\gamma_{K}v_{K}}\\
 & \overset{\eqref{eq:lip-core-v},\eqref{eq:lip-core-v0}}{=}\frac{\left\Vert \bx_{*}-\bx_{1}\right\Vert ^{2}-\left\Vert \bx_{*}-\bx_{K+1}\right\Vert ^{2}}{2n\sum_{k=1}^{K}\gamma_{k}}+3\bG^{2}n\sum_{k=1}^{K}\frac{\gamma_{k}\eta_{k}}{\sum_{\ell=k}^{K}\gamma_{\ell}}.
\end{align*}
\end{proof}

Next, for a special class of stepsizes (including the stepsize used
in Theorem \ref{thm:lip-ada-full} as a special case), the following
Lemma \ref{lem:lip-ada-distance} gives a uniform bound on the movement
of $\bx_{k}$ from the initial point $\bx_{1}$. A simple but useful
fact implied by this result is that $r_{k}$ is also uniformly upper
bounded. This important corollary will finally lead us to an asymptotic
rate having a linear dependence on $D$ (see the proof of Theorem
\ref{thm:lip-ada-full} for details).
\begin{lem}
\label{lem:lip-ada-distance}Under Assumptions \ref{assu:minimizer},
\ref{assu:basic} (with $\mu_{\psi}=0$) and \ref{assu:lip}, suppose
the following two conditions hold:
\begin{enumerate}
\item \label{enu:condition-1}$\eta_{k}=r_{k}\widetilde{\eta}_{k},\forall k\in\N$
where $r_{k}=r\lor\max_{\ell\in\left[k\right]}\left\Vert \bx_{\ell}-\bx_{1}\right\Vert $
for some $r>0$ and $\widetilde{\eta}_{k}$ is a positive sequence
defined on $\N$.
\item \label{enu:condition-2}$\sum_{k=1}^{\infty}6\bG^{2}n^{2}\widetilde{\eta}_{k}^{2}\leq c^{2}<1$
for some constant $c>0$.
\end{enumerate}
Then for any permutation $\sigma_{k},\forall k\in\N$, Algorithm \ref{alg:shuffling-SGD}
guarantees
\[
\left\Vert \bx_{k}-\bx_{1}\right\Vert \leq\frac{2}{1-c}\left\Vert \bx_{*}-\bx_{1}\right\Vert +\frac{c}{1-c}r,\forall k\in\N.
\]
\end{lem}

\begin{proof}
We invoke Lemma \ref{lem:core} with $\bz=\bx_{*}$ and $\mu_{\psi}=0$
to obtain
\begin{align*}
F(\bx_{k+1})-F(\bx_{*}) & \leq\frac{\left\Vert \bx_{*}-\bx_{k}\right\Vert ^{2}}{2n\eta_{k}}-\frac{\left\Vert \bx_{*}-\bx_{k+1}\right\Vert ^{2}}{2n\eta_{k}}-\frac{\left\Vert \bx_{k+1}-\bx_{k}\right\Vert ^{2}}{2n\eta_{k}}+\frac{1}{n}\sum_{i=1}^{n}B_{\sigma_{k}^{i}}(\bx_{k+1},\bx_{k}^{i})-B_{\sigma_{k}^{i}}(\bx_{*},\bx_{k}^{i})\\
 & \overset{(a)}{\leq}\frac{\left\Vert \bx_{*}-\bx_{k}\right\Vert ^{2}}{2n\eta_{k}}-\frac{\left\Vert \bx_{*}-\bx_{k+1}\right\Vert ^{2}}{2n\eta_{k}}-\frac{\left\Vert \bx_{k+1}-\bx_{k}\right\Vert ^{2}}{2n\eta_{k}}+2\bG\left\Vert \bx_{k+1}-\bx_{k}\right\Vert +\bG^{2}n\eta_{k}\\
 & \overset{(b)}{\leq}\frac{\left\Vert \bx_{*}-\bx_{k}\right\Vert ^{2}}{2n\eta_{k}}-\frac{\left\Vert \bx_{*}-\bx_{k+1}\right\Vert ^{2}}{2n\eta_{k}}+3\bG^{2}n\eta_{k},
\end{align*}
where the $(a)$ is due to Lemma \ref{lem:lip-res} and $(b)$ is
by AM-GM inequality. Note that $F(\bx_{k+1})-F(\bx_{*})\geq0$, hence
\begin{align}
\left\Vert \bx_{*}-\bx_{k+1}\right\Vert ^{2} & \leq\left\Vert \bx_{*}-\bx_{k}\right\Vert ^{2}+6\bG^{2}n^{2}\eta_{k}^{2}\nonumber \\
\Rightarrow\left\Vert \bx_{*}-\bx_{k+1}\right\Vert ^{2} & \leq\left\Vert \bx_{*}-\bx_{1}\right\Vert ^{2}+\sum_{\ell=1}^{k}6\bG^{2}n^{2}\eta_{\ell}^{2}\nonumber \\
\Rightarrow\left\Vert \bx_{k+1}-\bx_{1}\right\Vert  & \leq\left\Vert \bx_{*}-\bx_{1}\right\Vert +\left\Vert \bx_{*}-\bx_{k+1}\right\Vert \leq\left\Vert \bx_{*}-\bx_{1}\right\Vert +\sqrt{\left\Vert \bx_{*}-\bx_{1}\right\Vert ^{2}+\sum_{\ell=1}^{k}6\bG^{2}n^{2}\eta_{\ell}^{2}}\nonumber \\
 & \leq2\left\Vert \bx_{*}-\bx_{1}\right\Vert +\sqrt{\sum_{\ell=1}^{k}6\bG^{2}n^{2}\eta_{\ell}^{2}}\overset{(c)}{\leq}2\left\Vert \bx_{*}-\bx_{1}\right\Vert +r_{k}\sqrt{\sum_{\ell=1}^{k}6\bG^{2}n^{2}\widetilde{\eta}_{\ell}^{2}},\nonumber \\
 & \overset{(d)}{\leq}2\left\Vert \bx_{*}-\bx_{1}\right\Vert +cr_{k}\leq2\left\Vert \bx_{*}-\bx_{1}\right\Vert +cr+c\max_{\ell\in\left[k\right]}\left\Vert \bx_{\ell}-\bx_{1}\right\Vert ,\label{eq:lip-distance-1}
\end{align}
where $(c)$ is by $\eta_{\ell}^{2}=r_{\ell}^{2}\widetilde{\eta}_{\ell}^{2}\leq r_{k}^{2}\widetilde{\eta}_{\ell}^{2},\forall\ell\in\left[k\right]$
and $(d)$ is from the requirement $\sum_{k=1}^{\infty}6\bG^{2}n^{2}\widetilde{\eta}_{k}^{2}\leq c^{2}.$

Now we use induction to prove (\ref{eq:lip-distance-1}) implies the
following fact when $c<1$,
\begin{equation}
\left\Vert \bx_{k}-\bx_{1}\right\Vert \leq\frac{2}{1-c}\left\Vert \bx_{*}-\bx_{1}\right\Vert +\frac{c}{1-c}r,\forall k\in\N.\label{eq:lip-distance-hypo}
\end{equation}
First if $k=1$, (\ref{eq:lip-distance-hypo}) reduces to $0\leq\frac{2}{1-c}\left\Vert \bx_{*}-\bx_{1}\right\Vert +\frac{c}{1-c}r$,
which is true automatically. Suppose (\ref{eq:lip-distance-hypo})
holds for $1$ to $k$. Then for $k+1$, by (\ref{eq:lip-distance-1})
\begin{align*}
\left\Vert \bx_{k+1}-\bx_{1}\right\Vert  & \leq2\left\Vert \bx_{*}-\bx_{1}\right\Vert +cr+c\max_{\ell\in\left[k\right]}\left\Vert \bx_{\ell}-\bx_{1}\right\Vert \\
 & \leq2\left\Vert \bx_{*}-\bx_{1}\right\Vert +cr+c\left(\frac{2}{1-c}\left\Vert \bx_{*}-\bx_{1}\right\Vert +\frac{c}{1-c}r\right)\\
 & =\frac{2}{1-c}\left\Vert \bx_{*}-\bx_{1}\right\Vert +\frac{c}{1-c}r.
\end{align*}
Therefore, the induction is completed.
\end{proof}

\section{Technical Lemmas\label{sec:tech}}

In this section, we introduce two helpful lemmas used in the analysis.

First, we provide an upper bound for the residual term $R_{k}\triangleq\sum_{i=2}^{n}\frac{L_{\sigma_{k}^{i}}}{n}\left\Vert \sum_{j=1}^{i-1}\nabla f_{\sigma_{k}^{j}}(\bx_{*})\right\Vert ^{2},\forall k\in\left[K\right]$
defined in Lemma \ref{lem:smooth-res}. Specifically, we will prove
two bounds for this term: one in (\ref{eq:var-any}) holds for any
permutation, and the other in (\ref{eq:var-rand}) holds for the permutation
uniformly sampled without replacement. It is worth noting that if
$L_{i}=L$ for any $i\in\left[n\right]$, then one can apply Lemma
1 in \cite{mishchenko2020random} to get the desired bound. However,
$L_{i}$ can be different in our setting, which requires a more careful
analysis.
\begin{lem}
\label{lem:var}Under Assumptions \ref{assu:minimizer} and \ref{assu:smooth},
for any permutation $\sigma$ of $\left[n\right]$, there is
\begin{equation}
\sum_{i=2}^{n}\frac{L_{\sigma^{i}}}{n}\left\Vert \sum_{j=1}^{i-1}\nabla f_{\sigma^{j}}(\bx_{*})\right\Vert ^{2}\leq n^{2}\bL\asigma.\label{eq:var-any}
\end{equation}
If the permutation $\sigma$ is uniformly sampled without replacement,
there is
\begin{equation}
\E\left[\sum_{i=2}^{n}\frac{L_{\sigma^{i}}}{n}\left\Vert \sum_{j=1}^{i-1}\nabla f_{\sigma^{j}}(\bx_{*})\right\Vert ^{2}\right]\leq\frac{2}{3}n\bL\rsigma,\label{eq:var-rand}
\end{equation}
where the expectation is taken over the permutation $\sigma$.
\end{lem}

\begin{proof}
First note that
\begin{align*}
\sum_{i=2}^{n}\frac{L_{\sigma^{i}}}{n}\left\Vert \sum_{j=1}^{i-1}\nabla f_{\sigma^{j}}(\bx_{*})\right\Vert ^{2} & \leq\sum_{i=2}^{n}\sum_{j=1}^{i-1}\frac{L_{\sigma^{i}}}{n}(i-1)\left\Vert \nabla f_{\sigma^{j}}(\bx_{*})\right\Vert ^{2}=\sum_{j=1}^{n-1}\sum_{i=j+1}^{n}\frac{L_{\sigma^{i}}}{n}(i-1)\left\Vert \nabla f_{\sigma^{j}}(\bx_{*})\right\Vert ^{2}\\
 & \leq n\bL\sum_{j=1}^{n-1}\left\Vert \nabla f_{\sigma^{j}}(\bx_{*})\right\Vert ^{2}\leq n\bL\sum_{j=1}^{n}\left\Vert \nabla f_{\sigma^{j}}(\bx_{*})\right\Vert ^{2}=n^{2}\bL\asigma.
\end{align*}
Next, suppose $\sigma$ is uniformly sampled without replacement. 

When $n=2$
\begin{align*}
\E\left[\sum_{i=2}^{n}\frac{L_{\sigma^{i}}}{n}\left\Vert \sum_{j=1}^{i-1}\nabla f_{\sigma^{j}}(\bx_{*})\right\Vert ^{2}\right] & =\E\left[\frac{L_{\sigma^{2}}}{2}\left\Vert \nabla f_{\sigma^{1}}(\bx_{*})\right\Vert ^{2}\right]=\frac{L_{1}\left\Vert \nabla f_{2}(\bx_{*})\right\Vert ^{2}+L_{2}\left\Vert \nabla f_{1}(\bx_{*})\right\Vert ^{2}}{4}\\
 & =\underbrace{\frac{L_{1}+L_{2}}{2}\cdot\frac{\left\Vert \nabla f_{1}(\bx_{*})\right\Vert ^{2}+\left\Vert \nabla f_{2}(\bx_{*})\right\Vert ^{2}}{2}}_{=\bL\asigma}-\frac{L_{1}\left\Vert \nabla f_{1}(\bx_{*})\right\Vert ^{2}+L_{2}\left\Vert \nabla f_{2}(\bx_{*})\right\Vert ^{2}}{4}\\
 & \leq\bL\asigma\leq\frac{2}{3}n\bL\rsigma,
\end{align*}
where the last inequality holds due to $\asigma\leq\asigma+2\left\Vert \nabla f(\bx_{*})\right\Vert ^{2}=\rsigma$
and $\frac{2}{3}n=\frac{4}{3}>1$.

When $n\geq3$
\begin{align}
 & \E\left[\sum_{i=2}^{n}\frac{L_{\sigma^{i}}}{n}\left\Vert \sum_{j=1}^{i-1}\nabla f_{\sigma^{j}}(\bx_{*})\right\Vert ^{2}\right]=\frac{1}{n}\sum_{i=2}^{n}\E\left[L_{\sigma^{i}}\left\Vert \sum_{j=1}^{i-1}\nabla f_{\sigma^{j}}(\bx_{*})\right\Vert ^{2}\right]\nonumber \\
= & \frac{1}{n}\sum_{i=2}^{n}\E\left[\sum_{j=1}^{i-1}L_{\sigma^{i}}\left\Vert \nabla f_{\sigma^{j}}(\bx_{*})\right\Vert ^{2}+\sum_{1\leq p\neq q\leq i-1}L_{\sigma^{i}}\langle\nabla f_{\sigma^{p}}(\bx_{*}),\nabla f_{\sigma^{q}}(\bx_{*})\rangle\right]\nonumber \\
= & \frac{1}{n}\sum_{i=2}^{n}\sum_{j=1}^{i-1}\E\left[L_{\sigma^{i}}\left\Vert \nabla f_{\sigma^{j}}(\bx_{*})\right\Vert ^{2}\right]+\frac{1}{n}\sum_{i=2}^{n}\sum_{1\leq p\neq q\leq i-1}\E\left[L_{\sigma^{i}}\langle\nabla f_{\sigma^{p}}(\bx_{*}),\nabla f_{\sigma^{q}}(\bx_{*})\rangle\right].\label{eq:var-1}
\end{align}

For $j<i$, we have
\begin{align}
\E\left[L_{\sigma^{i}}\left\Vert \nabla f_{\sigma^{j}}(\bx_{*})\right\Vert ^{2}\right] & =\E\left[\E\left[L_{\sigma^{i}}\mid\sigma^{j}\right]\left\Vert \nabla f_{\sigma^{j}}(\bx_{*})\right\Vert ^{2}\right]=\E\left[\frac{n\bL-L_{\sigma^{j}}}{n-1}\left\Vert \nabla f_{\sigma^{j}}(\bx_{*})\right\Vert ^{2}\right]\nonumber \\
 & =\frac{n}{n-1}\bL\asigma-\frac{\sum_{\ell=1}^{n}L_{\ell}\left\Vert \nabla f_{\ell}(\bx_{*})\right\Vert ^{2}}{n(n-1)}\nonumber \\
\Rightarrow\frac{1}{n}\sum_{i=2}^{n}\sum_{j=1}^{i-1}\E\left[L_{\sigma^{i}}\left\Vert \nabla f_{\sigma^{j}}(\bx_{*})\right\Vert ^{2}\right] & =\frac{1}{n}\sum_{i=2}^{n}\sum_{j=1}^{i-1}\left(\frac{n}{n-1}\bL\asigma-\frac{\sum_{\ell=1}^{n}L_{\ell}\left\Vert \nabla f_{\ell}(\bx_{*})\right\Vert ^{2}}{n(n-1)}\right)\nonumber \\
 & =\frac{n\bL\asigma}{2}-\frac{\sum_{\ell=1}^{n}L_{\ell}\left\Vert \nabla f_{\ell}(\bx_{*})\right\Vert ^{2}}{2n}.\label{eq:var-2}
\end{align}

For $1\leq p\neq q\leq i-1$, we have
\begin{align}
 & \E\left[L_{\sigma^{i}}\langle\nabla f_{\sigma^{p}}(\bx_{*}),\nabla f_{\sigma^{q}}(\bx_{*})\rangle\right]=\E\left[\E\left[L_{\sigma^{i}}\mid\sigma^{p},\sigma^{q}\right]\langle\nabla f_{\sigma^{p}}(\bx_{*}),\nabla f_{\sigma^{q}}(\bx_{*})\rangle\right]\nonumber \\
= & \E\left[\frac{n\bL-L_{\sigma^{p}}-L_{\sigma^{q}}}{n-2}\langle\nabla f_{\sigma^{p}}(\bx_{*}),\nabla f_{\sigma^{q}}(\bx_{*})\rangle\right]\nonumber \\
= & \frac{n\bL}{n-2}\E\left[\langle\nabla f_{\sigma^{p}}(\bx_{*}),\nabla f_{\sigma^{q}}(\bx_{*})\rangle\right]-\frac{1}{n-2}\E\left[\left(L_{\sigma^{p}}+L_{\sigma^{q}}\right)\langle\nabla f_{\sigma^{p}}(\bx_{*}),\nabla f_{\sigma^{q}}(\bx_{*})\rangle\right].\label{eq:var-3}
\end{align}
Note that
\begin{equation}
\E\left[\langle\nabla f_{\sigma^{p}}(\bx_{*}),\nabla f_{\sigma^{q}}(\bx_{*})\rangle\right]=\E\left[\langle\frac{n\nabla f(\bx_{*})-\nabla f_{\sigma^{q}}(\bx_{*})}{n-1},\nabla f_{\sigma^{q}}(\bx_{*})\rangle\right]=\frac{n\left\Vert \nabla f(\bx_{*})\right\Vert ^{2}-\asigma}{n-1},\label{eq:var-4}
\end{equation}
and
\begin{align}
 & \E\left[\left(L_{\sigma^{p}}+L_{\sigma^{q}}\right)\langle\nabla f_{\sigma^{p}}(\bx_{*}),\nabla f_{\sigma^{q}}(\bx_{*})\rangle\right]\nonumber \\
= & 2\E\left[L_{\sigma^{p}}\langle\nabla f_{\sigma^{p}}(\bx_{*}),\nabla f_{\sigma^{q}}(\bx_{*})\rangle\right]\nonumber \\
= & 2\E\left[L_{\sigma^{p}}\langle\nabla f_{\sigma^{p}}(\bx_{*}),\frac{n\nabla f(\bx_{*})-\nabla f_{\sigma^{p}}(\bx_{*})}{n-1}\rangle\right]\nonumber \\
= & \frac{2\langle\sum_{\ell=1}^{n}L_{\ell}\nabla f_{\ell}(\bx_{*}),\nabla f(\bx_{*})\rangle}{n-1}-\frac{2\sum_{\ell=1}^{n}L_{\ell}\left\Vert \nabla f_{\ell}(\bx_{*})\right\Vert ^{2}}{(n-1)n}.\label{eq:var-5}
\end{align}
Plugging (\ref{eq:var-4}) and (\ref{eq:var-5}) into (\ref{eq:var-3})
to obtain
\begin{align*}
 & \E\left[L_{\sigma^{i}}\langle\nabla f_{\sigma^{p}}(\bx_{*}),\nabla f_{\sigma^{q}}(\bx_{*})\rangle\right]\\
= & \frac{n\bL}{n-2}\cdot\frac{n\left\Vert \nabla f(\bx_{*})\right\Vert ^{2}-\asigma}{n-1}-\frac{1}{n-2}\left(\frac{2\langle\sum_{\ell=1}^{n}L_{\ell}\nabla f_{\ell}(\bx_{*}),\nabla f(\bx_{*})\rangle}{n-1}-\frac{2\sum_{\ell=1}^{n}L_{\ell}\left\Vert \nabla f_{\ell}(\bx_{*})\right\Vert ^{2}}{(n-1)n}\right)\\
= & \frac{n^{2}\bL\left\Vert \nabla f(\bx_{*})\right\Vert ^{2}-n\bL\asigma-2\langle\sum_{\ell=1}^{n}L_{\ell}\nabla f_{\ell}(\bx_{*}),\nabla f(\bx_{*})\rangle}{(n-2)(n-1)}+\frac{2\sum_{\ell=1}^{n}L_{\ell}\left\Vert \nabla f_{\ell}(\bx_{*})\right\Vert ^{2}}{(n-2)(n-1)n},
\end{align*}
which implies
\begin{align}
 & \frac{1}{n}\sum_{i=2}^{n}\sum_{1\leq p\neq q\leq i-1}\E\left[L_{\sigma^{i}}\langle\nabla f_{\sigma^{p}}(\bx_{*}),\nabla f_{\sigma^{q}}(\bx_{*})\rangle\right]\nonumber \\
= & \frac{1}{n}\sum_{i=2}^{n}\sum_{1\leq p\neq q\leq i-1}\frac{n^{2}\bL\left\Vert \nabla f(\bx_{*})\right\Vert ^{2}-n\bL\asigma-2\langle\sum_{\ell=1}^{n}L_{\ell}\nabla f_{\ell}(\bx_{*}),\nabla f(\bx_{*})\rangle}{(n-2)(n-1)}+\frac{2\sum_{\ell=1}^{n}L_{\ell}\left\Vert \nabla f_{\ell}(\bx_{*})\right\Vert ^{2}}{(n-2)(n-1)n}\nonumber \\
= & \frac{n^{2}\bL\left\Vert \nabla f(\bx_{*})\right\Vert ^{2}-n\bL\asigma-2\langle\sum_{\ell=1}^{n}L_{\ell}\nabla f_{\ell}(\bx_{*}),\nabla f(\bx_{*})\rangle}{3}+\frac{2\sum_{\ell=1}^{n}L_{\ell}\left\Vert \nabla f_{\ell}(\bx_{*})\right\Vert ^{2}}{3n}.\label{eq:var-6}
\end{align}

Finally, plugging (\ref{eq:var-2}) and (\ref{eq:var-6}) into (\ref{eq:var-1}),
we know
\begin{align*}
 & \E\left[\sum_{i=2}^{n}\frac{L_{\sigma^{i}}}{n}\left\Vert \sum_{j=1}^{i-1}\nabla f_{\sigma^{j}}(\bx_{*})\right\Vert ^{2}\right]\\
= & \frac{n\bL\asigma}{2}-\frac{\sum_{\ell=1}^{n}L_{\ell}\left\Vert \nabla f_{\ell}(\bx_{*})\right\Vert ^{2}}{2n}+\frac{2\sum_{\ell=1}^{n}L_{\ell}\left\Vert \nabla f_{\ell}(\bx_{*})\right\Vert ^{2}}{3n}\\
 & +\frac{n^{2}\bL\left\Vert \nabla f(\bx_{*})\right\Vert ^{2}-n\bL\asigma-2\langle\sum_{\ell=1}^{n}L_{\ell}\nabla f_{\ell}(\bx_{*}),\nabla f(\bx_{*})\rangle}{3}\\
= & \frac{n\bL\asigma}{6}+\frac{\sum_{\ell=1}^{n}L_{\ell}\left\Vert \nabla f_{\ell}(\bx_{*})\right\Vert ^{2}}{6n}+\frac{n^{2}\bL\left\Vert \nabla f(\bx_{*})\right\Vert ^{2}}{3}-\frac{2\langle\sum_{\ell=1}^{n}L_{\ell}\nabla f_{\ell}(\bx_{*}),\nabla f(\bx_{*})\rangle}{3}.
\end{align*}
Note that
\begin{align*}
\frac{\sum_{\ell=1}^{n}L_{\ell}\left\Vert \nabla f_{\ell}(\bx_{*})\right\Vert ^{2}}{6n} & \leq\frac{n\bL\asigma}{6},\\
-\frac{2\langle\sum_{\ell=1}^{n}L_{\ell}\nabla f_{\ell}(\bx_{*}),\nabla f(\bx_{*})\rangle}{3} & \leq\frac{n^{2}\bL\left\Vert \nabla f(\bx_{*})\right\Vert ^{2}}{3}+\frac{\left\Vert \sum_{\ell=1}^{n}L_{\ell}\nabla f_{\ell}(\bx_{*})\right\Vert ^{2}}{3n^{2}\bL}\\
 & \leq\frac{n^{2}\bL\left\Vert \nabla f(\bx_{*})\right\Vert ^{2}}{3}+\frac{\sum_{\ell=1}^{n}L_{\ell}\left\Vert \nabla f_{\ell}(\bx_{*})\right\Vert ^{2}}{3n}\\
 & \leq\frac{n^{2}\bL\left\Vert \nabla f(\bx_{*})\right\Vert ^{2}}{3}+\frac{n\bL\asigma}{3}.
\end{align*}
Hence, there is
\[
\E\left[\sum_{i=2}^{n}\frac{L_{\sigma^{i}}}{n}\left\Vert \sum_{j=1}^{i-1}\nabla f_{\sigma^{j}}(\bx_{*})\right\Vert ^{2}\right]\leq\frac{2}{3}n\bL\left(\asigma+n\left\Vert \nabla f(\bx_{*})\right\Vert ^{2}\right)=\frac{2}{3}n\bL\rsigma.
\]
\end{proof}

Next, we introduce the following algebraic inequality, which is a
useful tool when proving Theorem \ref{thm:smooth-cvx-full}.
\begin{lem}
\label{lem:ineq}Given a sequence $d_{2},\cdots,d_{K},d_{K+1}$, suppose
there exist positive constants $a,b,c$ satisfying 
\begin{equation}
d_{k+1}\leq\frac{a}{k}+b(1+\log k)+c\sum_{\ell=2}^{k}\frac{d_{\ell}}{k-\ell+2},\forall k\in\left[K\right],\label{eq:ineq-condi}
\end{equation}
then the following inequality holds
\begin{equation}
d_{k+1}\leq\left(\frac{a}{k}+b(1+\log k)\right)\sum_{i=0}^{k-1}\left(2c(1+\log k)\right)^{i},\forall k\in\left[K\right].\label{eq:ineq-result}
\end{equation}
\end{lem}

\begin{proof}
We use induction to prove (\ref{eq:ineq-result}) holds for every
$k\in\left[K\right]$. First, for $k=1$, we need to show
\[
d_{2}\leq a+b,
\]
which is true by taking $k=1$ in (\ref{eq:ineq-condi}). Suppose
(\ref{eq:ineq-result}) holds for $1$ to $k-1$ (where $2\leq k\leq K$),
i.e.,
\[
d_{\ell}\leq\left(\frac{a}{\ell-1}+b(1+\log(\ell-1))\right)\sum_{i=0}^{\ell-2}\left(2c(1+\log(\ell-1))\right)^{i},\forall2\leq\ell\leq k.
\]
which implies
\begin{equation}
d_{\ell}\leq\left(\frac{a}{\ell-1}+b(1+\log k)\right)\sum_{i=0}^{\ell-2}\left(2c(1+\log k)\right)^{i},\forall2\leq\ell\leq k.\label{eq:ineq-hypo}
\end{equation}
Now for $d_{k+1}$, from (\ref{eq:ineq-condi})
\begin{align}
d_{k+1}\leq & \frac{a}{k}+b(1+\log k)+c\sum_{\ell=2}^{k}\frac{d_{\ell}}{k-\ell+2}\nonumber \\
\overset{\eqref{eq:ineq-hypo}}{\leq} & \frac{a}{k}+ac\sum_{\ell=2}^{k}\sum_{i=0}^{\ell-2}\frac{\left(2c(1+\log k)\right)^{i}}{(k-\ell+2)(\ell-1)}\nonumber \\
 & +b(1+\log k)\left(1+c\sum_{\ell=2}^{k}\sum_{i=0}^{\ell-2}\frac{\left(2c(1+\log k)\right)^{i}}{k-\ell+2}\right).\label{eq:ineq-1}
\end{align}
Note that
\begin{align}
c\sum_{\ell=2}^{k}\sum_{i=0}^{\ell-2}\frac{\left(2c(1+\log k)\right)^{i}}{(k-\ell+2)(\ell-1)} & =c\sum_{i=0}^{k-2}\left(2c(1+\log k)\right)^{i}\left(\sum_{\ell=2+i}^{k}\frac{1}{(k-\ell+2)(\ell-1)}\right)\nonumber \\
 & =\frac{c}{k+1}\sum_{i=0}^{k-2}\left(2c(1+\log k)\right)^{i}\left(\sum_{\ell=2+i}^{k}\frac{1}{k-\ell+2}+\frac{1}{\ell-1}\right)\nonumber \\
 & \leq\frac{c}{k+1}\sum_{i=0}^{k-2}\left(2c(1+\log k)\right)^{i}\sum_{\ell=1}^{k}\frac{2}{\ell}\leq\frac{\sum_{i=0}^{k-2}\left(2c(1+\log k)\right)^{i+1}}{k+1}\nonumber \\
 & =\frac{\sum_{i=1}^{k-1}\left(2c(1+\log k)\right)^{i}}{k+1}\leq\frac{\sum_{i=1}^{k-1}\left(2c(1+\log k)\right)^{i}}{k},\label{eq:ineq-2}
\end{align}
and
\begin{align}
c\sum_{\ell=2}^{k}\sum_{i=0}^{\ell-2}\frac{\left(2c(1+\log k)\right)^{i}}{k-\ell+2} & =c\sum_{i=0}^{k-2}\left(2c(1+\log k)\right)^{i}\sum_{\ell=2+i}^{k}\frac{1}{k-\ell+2}\leq c\sum_{i=0}^{k-2}\left(2c(1+\log k)\right)^{i}\sum_{\ell=1}^{k}\frac{1}{\ell}\nonumber \\
 & \leq c(1+\log k)\sum_{i=0}^{k-2}\left(2c(1+\log k)\right)^{i}\leq\sum_{i=0}^{k-2}\left(2c(1+\log k)\right)^{i+1}\nonumber \\
 & =\sum_{i=1}^{k-1}\left(2c(1+\log k)\right)^{i}.\label{eq:ineq-3}
\end{align}
Combining (\ref{eq:ineq-1}), (\ref{eq:ineq-2}) and (\ref{eq:ineq-3}),
we obtain the following inequality and finish the induction
\begin{align*}
d_{k+1} & \leq\frac{a}{k}+\frac{a}{k}\sum_{i=1}^{k-1}\left(2c(1+\log k)\right)^{i}+b(1+\log k)\left(1+\sum_{i=1}^{k-1}\left(2c(1+\log k)\right)^{i}\right)\\
 & =\left(\frac{a}{k}+b(1+\log k)\right)\sum_{i=0}^{k-1}\left(2c(1+\log k)\right)^{i}.
\end{align*}
\end{proof}

\end{document}